\documentclass[sn-mathphys-ay]{sn-jnl}%

\usepackage{graphicx}%
\usepackage{multirow}%
\usepackage{amsmath,amssymb,amsfonts}%
\usepackage{amsthm}%
\usepackage{mathtools}
\usepackage{mathrsfs}%
\usepackage[title]{appendix}%
\usepackage{xcolor}%
\usepackage{textcomp}%
\usepackage{manyfoot}%
\usepackage{booktabs}%
\usepackage{algorithm}%
\usepackage{algorithmicx}%
\usepackage{algpseudocode}%
\usepackage{listings}%
\usepackage{longtable}
\usepackage{bbold}
\usepackage{lmodern}
\usepackage{bbm}
\usepackage{url}

\newcommand{\abs}[1]{\lvert #1 \rvert}

\newcommand{\AAbs}[1]{\left\lvert #1 \right\rvert}
\newcommand{\norm}[1]{\lVert #1 \rVert}
\newcommand{\Norm}[1]{\left\lVert #1 \right\rVert}
\newcommand{\brac}[1]{\langle #1 \rangle}

\newcommand{\cov}{\mathop{\rm Cov}}
\newcommand{\var}{\mathop{\rm Var}}

\newcommand{\ie}{{\em i.e.,~}}
\newcommand{\R}{\mathbb{R}}
\newcommand{\Obj}{\mathcal{O}}

\newcommand{\E}{\mathbb{E}}
\renewcommand{\P}{\mathbb{P}}
\newcommand{\N}{\mathbb{N}}

\newcommand{\Dmu}{\mathbf{\Delta_\mu}}
\newcommand{\Dmuu}{\mathbf{\Delta}^2_\mu}
\newcommand{\Dgam}{\mathbf{\Delta_\gamma}}
\newcommand{\Dgamm}{\mathbf{\Delta}^2_\gamma}
\newcommand{\Dlamb}{\mathbf{\Delta_\lambda}}
\newcommand{\thr}{\theta}
\newcommand{\hatj}{\hat{j}}
\newcommand{\W}{W}
\newcommand{\WW}{W}
\newcommand{\taupi}{\tau^{\Pi^j}_n}
\newcommand{\taun}{\tau^{N^j}_n}
\newcommand{\tauw}{\tau^{W^j}_n}
\newcommand{\taupibis}{\Tilde{\tau}^{\pi^j}_n}
\newcommand{\tauwbis}{\Tilde{\tau}^{W^j}_n}
\newcommand{\taunbis}{\Tilde{\tau}^{N^j}_n}

\usepackage{xcolor}         %

\theoremstyle{plain}
\newtheorem{theorem}{Theorem}[section]
\newtheorem{proposition}[theorem]{Proposition}
\newtheorem{lemma}[theorem]{Lemma}
\newtheorem{corollary}[theorem]{Corollary}
\theoremstyle{definition}
\newtheorem{definition}[theorem]{Definition}
\newtheorem{assumption}[theorem]{Assumption}
\theoremstyle{remark}
\newtheorem{remark}[theorem]{Remark}

\begin{document}

\title[Spiking Neural Models for Decision-Making Tasks with Learning]{Spiking Neural Models for Decision-Making Tasks with Learning}

\author*[1]{\fnm{Sophie} \sur{Jaffard}}\email{sophie.jaffard@univ-cotedazur.fr}

\author[2]{\fnm{Giulia} \sur{Mezzadri}}\email{gm3026@columbia.edu}

\author[3]{\fnm{Patricia} \sur{Reynaud-Bouret}}\email{patricia.reynaud-bouret@univ-cotedazur.fr}
\author[4]{\fnm{Etienne} \sur{Tanré}}\email{Etienne.Tanre@inria.fr}

\affil[1]{\orgdiv{Laboratoire J. A. Dieudonné}, \orgname{Université Côte d'Azur}, \city{Nice}, \country{France}}

\affil[2]{\orgdiv{Cognition and Decision Lab}, \orgname{Columbia University}, \city{New York}, \country{US}}

\affil[3]{\orgdiv{Laboratoire J. A. Dieudonné}, \orgname{CNRS}, \orgname{Université Côte d'Azur}, \city{Nice}, \country{France}}

\affil[4]{\orgdiv{Universit\'e C\^ote D'Azur}, \orgname{Inria, CNRS, LJAD}, \city{Nice}, \country{France}}

\abstract{
    In cognition, response times and choices in decision-making tasks are commonly modeled using Drift Diffusion Models (DDMs), which describe the accumulation of evidence for a decision as a stochastic process, specifically a Brownian motion, with the drift rate reflecting the strength of the evidence. In the same vein, the Poisson counter model describes the accumulation of evidence as discrete events whose counts over time are modeled as Poisson processes. This model has a spiking neurons interpretation as these processes are used to model neuronal activities. However, these models lack a learning mechanism and are limited to tasks where participants have prior knowledge of the categories. To bridge the gap between cognitive and biological models, we propose a biologically plausible Spiking Neural Network (SNN) model for decision-making that incorporates a learning mechanism and whose neurons activities are modeled by a multivariate Hawkes process. First, we show a coupling result between the DDM and the Poisson counter model, establishing that these two models provide similar categorizations and reaction times and that the DDM can be approximated by spiking Poisson neurons. To go further, we show that a particular DDM with correlated noise can be derived from a Hawkes network of spiking neurons governed by a local learning rule. In addition, we designed an online categorization task to evaluate the model predictions. This work provides a significant step toward integrating biologically relevant neural mechanisms into cognitive models, fostering a deeper understanding of the relationship between neural activity and behavior.
}

\keywords{Spiking Neural Network, Decision-making, Hawkes process, Local learning rule, Drift diffusion model}

\pacs[MSC Classification]{60G55, 68T05, 60F15, 92-10}

\maketitle
\subsection*{Contribution}
S. Jaffard did the writing, the proofs, collected the data and analyzed them. G. Mezzadri designed the experiment and participated to the writing. P. Reynaud-Bouret participated to the writing and the proofs, provided some of the initial ideas about the coupling and supervised the work. E. Tanré provided some of the initial ideas about the coupling and a thorough checking of all the mathematical details.

\section{Introduction}

One of the main challenges of the $21$st century is to bridge the gap between brain and behavior, and computational models are a fundamental tool 
researchers use to address this challenge. In this line,
developing biologically realistic models to understand the decision-making processes of humans and animals is a key focus in neuroscience and cognition. 

A prominent model in cognition is the Drift Diffusion Model (DDM), introduced by \cite{ratcliff1978theory,ratcliff1981theory,ratcliff1988continuous} 
who demonstrated how the DDM could accurately predict outcomes such as accuracy, reaction times, and error latencies across various decision-making tasks. 
The model achieves this by integrating evidence over time, with a drift rate reflecting evidence strength and decision boundaries representing the thresholds for making a decision. Initially designed for two-choice experiments, it extends to multiple choices by modeling evidence accumulation as a multidimensional drift-diffusion process, with each component competing against the others \citep{roxin2019drift}.
The reasons for the growing attention to the DDM are multiple. It has been shown to provide accurate descriptions of accuracy 
and reaction time data across a wide range of psychological tasks, including perceptual tasks 
\citep{ratcliff1998modeling, ditterich2006stochastic, brunton2013rats} and value-based choices 
\citep{ratcliff2008diffusion, philiastides2013influence, hutcherson2015neurocomputational}.

It is also supported by neurobiological evidence: in the 1990s and 2000s, neurobiologists like Shadlen and Newsome provided crucial neurobiological evidence supporting the DDM by showing that neural activity in brain areas such as the lateral intraparietal area (LIP) 
reflects the accumulation of sensory evidence, consistent with DDM predictions 
\citep{shadlen1996motion, gold2001neural, ratcliff2007dual, shadlen2013decision}. 
However, DDMs do not directly include the modeling of neuronal activity.

Biologically relevant decision-making models \citep{brody2003basic, machens2005flexible, deco2007deterministic} rely on perceptual information to guide the choice between behaviors. A straightforward approach involves two interacting groups of neurons, each defined by its average firing rate. Decision-making in this context is modeled as a transition from a spontaneous state (where both groups have similar firing rates) to a decision state (characterized by a high or low activity ratio between the groups), and this process is often represented as a stochastic dynamical system. Such models effectively approximate reaction times and can be used to compute biologically meaningful information. For example, \cite{carrillo2011decision} offer a mathematical analysis of these systems in nonlinear scenarios and computes the probability of reaching specific decisions and the average time required to do so. Other approaches than evidence accumulation can be implemented in such models: for instance, \cite{insabato2010confidence} propose a decision confidence mechanism. However none of these models are linked to the DDM nor show such accurate descriptions of real accuracy or reaction times.

Simpler than these models and more similar to the Drift Diffusion Model, the Poisson counter model \citep{laberge1994quantitative, ratcliff2004comparison} represents evidence accumulation as discrete events, with their counts over time modeled as homogeneous Poisson processes. This approach aligns with how neuroscientists often use Poisson processes to model neuronal activity based on firing rates, offering a spiking neuron interpretation: evidence accumulation can be seen as the spike counts of neurons coding for potential responses. \cite{smith2000time} extend the counter model by incorporating inhomogeneous Poisson processes. However, since the Poisson process is not well-suited for modeling the interactions between neurons, this approach is limited to representing individual neurons coding for categories and cannot capture the dynamics of more complex networks of spiking neurons.

A common tool to model spike trains of interacting neurons is the multivariate Hawkes process \citep{hawkes1971spectra}, which is a self-exciting and mutually exciting point process. Its field of application is not limited to neuroscience: it is well-adapted to model earthquake data \citep{turkyilmaz2013comparing}, financial transactions \citep{bacry2015hawkes}, health data \citep{pmlr-v68-bao17a}, social networks \citep{zhou2013learning}, or more generally, any sequences of events such that the occurrence of an event influences the probability of further events to occur.

Poisson and Hawkes processes seem to be promising microscopic processes to explain the power of the DDM prediction. Indeed diffusion approximations of Poisson and Hawkes processes have been investigated in several contexts. For instance, \cite{bacry2012scaling} provide a functional central limit theorem for multivariate Hawkes processes for finance applications, and \cite{ethier2009markov} as well as \cite{bretagnolle1989hungarian} provide strong approximations of the Poisson process by a Brownian process. %
Strong approximations by diffusive limits have also been derived to study mean field limits of networks of interacting neurons \citep{erny2023strong}, and \cite{besanccon2024diffusive} provide coupling results between Poisson and Hawkes processes in one dimension.

However, neither the DDM, the Poisson counter model, nor the other biologically relevant decision-making models discussed here incorporate a learning mechanism. These models are limited to tasks where participants already possess prior knowledge of the possible responses. In the DDM, for instance, this prior knowledge is represented by the value of the drift rate, while in stochastic dynamical systems, it is captured by the (fixed) strength of synaptic connections. For example, \cite{deco2007deterministic} assumes that synaptic connections between neurons are the result of a Hebbian learning mechanism that has taken place before the model they propose.

On the other hand, learning mechanisms are at the heart of artificial neural networks (ANNs), the first of which is the perceptron \citep{rosenblatt1957perceptron}, inspired by the structure and function of biological neurons and their networks. This biological inspiration has continued to influence the development of increasingly advanced machine learning models, such as convolutional neural networks \citep{lecun1989handwritten}, which draw their inspiration from biological processes observed in the visual cortex by \cite{hubel1962receptive}. Conversely, machine learning algorithms play a central role in implementing spiking neural networks (SNNs) \citep{Tavanaei_2019}. These networks, more complex than ANNs, model neuronal activity through sequences of spikes, mimicking the electrical pulses of biological neurons. Particularly relevant for studying the brain's neural code, SNNs also offer insights into the design of energy-efficient algorithms \citep{stone2018principles}. To be realistic, SNNs are typically trained using local learning rules such as Hebbian learning \citep{hebb2005organization} or spike-timing-dependent plasticity \citep{caporale2008spike}. However, spiking neural networks (SNNs) are typically not designed to model decision-making tasks. They do not inherently account for reaction times and are primarily optimized for achieving strong empirical performance by mimicking the functioning of the brain rather than making decisions that closely resemble human behavior.
\newline

The focus of the present work is to bridge the gap between decision-making models used in cognition and brain-inspired models of spiking neurons, and to establish that prior knowledge assumed in decision-making models may be the result of a biologically relevant learning mechanism. First, we introduce both a drift diffusion model and a Poisson counter model, outlining conditions under which each produces accurate decisions (Theorems~\ref{th ddm} and~\ref{th poisson correct categ}). As an initial step toward deriving the DDM from spiking neuron-based models, we establish a coupling between the two models, establishing that they provide similar categorizations and reaction times (Theorem~\ref{th coupling ddm poisson}).  

To go further, we introduce a biologically inspired Spiking Neural Network (SNN) model for decision-making. This model integrates a learning mechanism and is provably close to the drift diffusion model. Neuronal activity within the network is represented by a multivariate Hawkes process, which accounts for interactions between neurons. The output neurons encode possible decisions, with their spike counts representing evidence accumulation for respective responses, which makes of our model a Hawkes counter model. Synaptic connections are adjusted via a local learning rule provided by the expert aggregation framework \citep{cesa2006prediction}.  In \cite{jaffard2024provable}, we introduced a simpler version of the network. The focus of this work was to prove theoretically that a biologically plausible neural network could learn using local mechanisms only. Building on this foundation, \cite{jaffard2024chani} extended the network by incorporating hidden layers specifically designed to detect neuronal synchronizations. This extension not only advanced the theoretical results from \cite{jaffard2024provable} but also demonstrated that our algorithm automatically produces neuronal assemblies in the sense that the network can encode several concepts and that a same neuron in the intermediate layers might be activated by more than one concept. Additionally, we provided an in-depth analysis of the types of concepts the network can encode and numerical results on handwritten digits dataset. However, this initial version was not suited to model decision-making tasks as it was unable to model reaction times.

The model presented in the current work is a refined version of this network, which introduces two key differences from these earlier versions. First, neuronal activities are modeled in continuous time rather than discrete time, which is more realistic and closer to the drift diffusion model and Poisson counter model. Second, we have incorporated dynamic decision durations, allowing the model to account for reaction times alongside decision-making processes, unlike the fixed decision durations used previously. These improvements enable the model to be compared with other models such as the DDM or the Poisson counter model.

We analyze the network's asymptotic behavior (Proposition~\ref{prop conv weights}) and provide conditions under which the model delivers accurate decisions (Theorem~\ref{th correct classif hawkes}). Additionally, we establish a coupling between our model and the DDM (Theorem~\ref{th coupling hawkes ddm}) thanks to strong approximations similar to \cite{ethier2009markov}, demonstrating that a DDM with correlated noise can be approximated by a Hawkes network of spiking neurons capable of learning via local rules.

To evaluate our model, we designed an online categorization task in which participants were asked to classify stimuli -- rockets in this case -- defined by four characteristics, with one randomly selected feature serving as the basis for learning. Such feature-based tasks have long been central to cognitive science \citep{Nosofsky1986,Love2004}, providing valuable insights into how individuals learn to group stimuli based on shared characteristics. This well-established paradigm made it a natural starting point for testing our model. By analyzing participants' reaction times 
and choices, we demonstrate the utility of the Hawkes counter model in studying learning and 
decision-making processes.  We describe the experiment we designed and detail its implementation within our model. We prove that the network can learn to perform this task successfully (Proposition~\ref{prop cas particulier}). Furthermore, we estimate model parameters for individual participants, enabling a detailed analysis of their learning behaviors.  
\newline

Section~\ref{sec ddm and poisson} introduces the decision-making task we aim to model, along with the drift diffusion model and the Poisson counter model, including their mathematical analyses and coupling. Section~\ref{sec hawkes counter model} outlines the Hawkes counter model for decision-making tasks with learning and its corresponding mathematical analysis and coupling with the DDM. Section~\ref{sec expe} details our experiment and its analysis based on our model. All notations defined in the main text are listed in Appendix~\ref{sec notations}, and any proofs not included in the main text are provided in Appendix~\ref{sec proofs}.

\section{Drift Diffusion and Poisson counter models} \label{sec ddm and poisson}

\subsection{Framework}

We aim to model first the following decision-making task: a participant sequentially categorizes objects, 
each having a nature from a set $\Obj$\label{def:obj}, into one of several categories belonging to a set $J$. Each category 
$j\in J$ is a subset of $\Obj$, and together the categories form a partition of $\Obj$. Participants may 
encounter the same object nature multiple times; for example, the first object presented might be a blue 
circle, as might the tenth object. See Figure~\ref{fig objects} for an example.

In the present section, we present two models: the drift diffusion model and the Poisson counter model. 
They provide predictions for both categorizations and reaction times. These models assume that participants already possess prior knowledge about object natures and categories, and do not describe scenarios where participants must learn the categories of the presented objects.

For both models, we give explicit conditions under which they predict correct categorizations with high probability. Then, we prove a coupling result between the models indicating that these two models provide very similar reaction times and categorizations.

In the sequel, to simplify formula, we use the following notation: for parameters $p_1,\dots, p_k$, we denote by $\Box_{p_1,\dots,p_k}$ any positive constant depending only on parameters $p_1,\dots, p_k$. The value of $\Box_{p_1,\dots,p_k}$ can change from line to line and it is not necessarily equal on both sides of an inequality.

For $a,b\in \R$, we denote $a\vee b \coloneqq \max(a,b)$ and $a\wedge b \coloneqq \min(a,b)$. Also when quantities have several indices, the dropping of one index corresponds to the sequence on this index. For instance, $\gamma = (\gamma^j_o)_{j\in J, o\in \Obj}$ whereas $\gamma_o = (\gamma^j_o)_{j\in J}$ and $\gamma^j = (\gamma^j_o)_{o\in \Obj}$. If two sets of indices are available, typically $I$ and $J$, to avoid confusion, we use the set instead of the index. For instance, $\gamma^I=(\gamma^i_o)_{i\in I, o\in \Obj}$. Also $ \| \cdot\|_{\infty}$ is the infinite norm.

\begin{figure}
    \centering
    \includegraphics[width=0.8\linewidth]{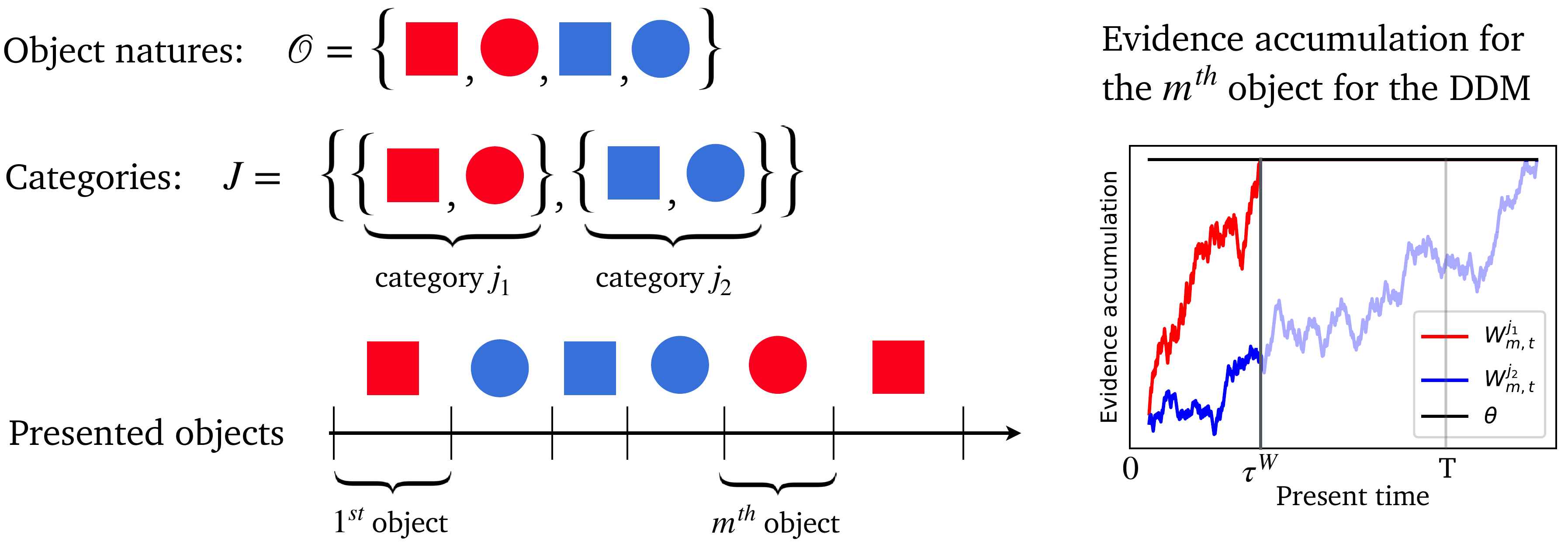}
    \caption{Illustrative example of objects natures, categories, presented objects and evidence accumulation 
    for the $m^{th}$ presented object for the DDM. 
    Here, there are $4$ object natures, $2$ categories and $6$ 
    presented objects. The chosen response $\hatj$ is the category coded by the first process to reach threshold $\thr$, \ie $\hatj = j_1$ and the corresponding reaction time is $\tau^W$.
    }
    \label{fig objects}
\end{figure}

\subsection{Drift Diffusion model} \label{sec ddm}

\subsubsection{The model}

During the presentation of the $m^{th}$ object\label{def:m}, the accumulation of evidence of the categories $j\in J$ is modeled by a 
$\abs{J}-$dimensional drifted Brownian motion \label{def:wm} $\WW_m = (\W^j_{m})_{j\in J}$ with mean vector defined by $\forall j\in J, \forall t\geq 0$,
$\E[\W^j_{m,t}] = \mu^j_o t$ and diagonal covariance matrix defined by \(\forall j_1,j_2 \in J\), $\forall t,s\geq 0$, 
$\cov(\W^{j_1}_{m,t}, \W^{j_2}_{m,s}) = \mathbbm{1}_{\{j_1 = j_2\}}
\mu^{j_1}_o (t\wedge s)$ 
where $o$ \label{def:petito} is the nature of the $m^{th}$ object and $\mu^j_o\geq 0$ is the drift. Therefore, the process coding for evidence accumulation in favor of category $j$ can be rewritten as 
\begin{equation} \label{eq def w}
   \forall t\geq 0, \quad \W^j_{m,t} \coloneqq \mu^j_o t + \sqrt{\mu^j_o} B^j_{m,t}
\end{equation}
where the $(B^j_{m,t})_{t\geq 0}$ are independent standard Brownian motions. 
The object is categorized in the category coded by the first process $(\W^j_{m,t})_{t\geq 0}$ among the processes of $\WW_m$ to reach 
a certain threshold $\thr>0$\label{def:theta} before a limit duration $T$\label{def:limitT}. 
We denote 
\begin{equation}\label{tau}
\tau^{\W^j}\coloneqq \inf\{u\geq 0; \W^j_{m,u} \geq \thr\}.
\end{equation}

\begin{remark}
The law of \(\tau^{\W^j}\) is the inverse Gaussian Distribution with parameters
\(\frac{\thr}{\mu^j_o}\) and \(\frac{\thr^2}{\mu^j_o}\)  \cite[p.687]{lovric2024international}.
\end{remark}
The winning process is then 
\begin{equation}\label{choice}
\hatj = \arg\min_{j\in J} \tau^{\W^j}.
\end{equation}
Note that $\hatj$ and $\tau^{\W^j}$ depend on $m$ but we drop the index for simplicity.
However, if 
$\tau^{\W^{\hatj}}\geq T$, it means that accumulation evidence is not sufficient to reach a decision and the 
model does not categorize the object. The corresponding reaction time of the whole system is then 
$\tau^W \coloneqq \tau^{\W^{\hatj}}\wedge T$. \label{def:reactimeW}
Note that the processes are indexed by time $t\in[0,T]$ which is not a `time', 
but a `present time' since it is set back to $0$ at each 
presentation of a new object. See Figure $\ref{fig objects}$ for an 
example of evidence accumulation modeled by the DDM.

 Note that in general, in drift diffusion models, the scaling is not necessarily the square root of the drift. We study this specific case because it allows the approximations by spiking neural models that we present further on. Note also that the drift and scaling do not evolve with the number of presented objects $m$: they represent prior knowledge about the categories and object natures and are fixed. The higher the drift $\mu^j_o$, the higher the average accumulation of evidence in favor of category $j$.

\subsubsection{Theoretical guarantees}

Let us study under which conditions the model predicts the correct category. We assume that the participant has prior knowledge about the categories, which takes the form of the following assumption. Let $\Dmu>0$. 
\vspace{2mm}
\begin{assumption}[Margin($\Dmu$)] \label{assump drift}
For every category $j^*\in J$, for every $j\neq j^*$, for every object nature $o\in j^*$ we have
     \[
    \mu^{j^*}_o > \mu^{j}_o + \Dmu.
    \]
\end{assumption}
\vspace{2mm}

\begin{theorem} \label{th ddm}
    Let $\alpha \in (0,1)$ and let $m$ be a fixed positive integer. %
    Let $\Dmu>0$ and suppose Assumption~\ref{assump drift} (Margin($\Dmu$)) holds. Then if $T \geq \Box \max\left( \frac{\norm{\mu}_{\infty}}{\Dmuu} \log\left(\frac{\norm{\mu}_{\infty}^2}{\Dmuu}\right), \norm{\mu}_\infty^{-1} \log\left(\frac{\abs{J}}{\alpha}\right)  
\right) $ and if the threshold $\thr$ satisfies
\begin{equation} \label{lower bound theta ddm}
    \thr >  \square \frac{\norm{\mu}_\infty}{\Dmuu}\max\left( \log\left(\frac{4\norm{\mu}_{\infty}^2}{\Dmuu}\right), \log(|J|\alpha^{-1})\right)
\end{equation}
and
\begin{equation} \label{upper bound theta ddm}
    \thr < \Dmu T -  \sqrt{(\norm{\mu}_{\infty} T+1)(\log(\norm{\mu}_{\infty}T+1)+2\log(\abs{J}\alpha^{-1}))}
\end{equation}
    then with probability larger than $1-\alpha$ we have that for all nature $o\in \mathcal{O}$, if we denote by $j^*\in J$ its category (i.e. $o\in j^*$), the choice $\hatj$ \eqref{choice} and the reaction time $\tau^{W^{\hat{j}}}$ \eqref{tau} that take place at the presentation of the $m^{th}$ object of nature $o$, satisfy 
    \[
    \hatj = j^* \quad \mbox{and}\quad \tau^{\W^{\hatj}} < T.
    \] 
\end{theorem}

In other words, this theorem states that for a fixed object $m$, if the limit duration $T$ is large enough then as soon as the threshold $\thr$ verifies \eqref{lower bound theta ddm} and \eqref{upper bound theta ddm}, object $m$ is correctly classified with probability more than $1-\alpha$. The lower bound \eqref{lower bound theta ddm} comes from the condition that the process coding for category $j^*$ should be the first to reach $\theta$, whereas the upper bound \eqref{upper bound theta ddm} in $O(T)$ comes from the condition that $\thr$ should be reached before limit duration $T$. Note that as expected, the larger $\Dmu$, \ie the larger the gap between the drift of the process coding for the true category and the other drifts, the lower the bound on $T$ and the larger the size of the thresholds interval enabling a correct categorization with high probability.

\subsection{Poisson counter model}

\subsubsection{The model}

For each presented object, the evidence in favor of the categories takes now the form of discrete events whose 
counts over time are modeled by a $\abs{J}-$dimensional homogeneous Poisson process 
\label{def:pim}
$\Pi_m = (\Pi^j_m)_{j\in J}$ 
with intensity vector $\gamma_o \coloneqq (\gamma^j_o)_{j\in J}$ depending only on the nature $o$ of the presented 
object. Similarly as for the DDM, the object is categorized in the category coded by the first process 
$(\Pi^j_{m,t})_{t\geq 0}$ among the processes of $\Pi_m$ to reach threshold $\thr>0$ before limit duration $T$.
We denote 
\begin{equation}\label{tauPois}
\tau^{\Pi^j} \coloneqq \inf\{u\geq 0; \Pi^j_{m,u}\geq \thr\}
\end{equation}
the hitting time of process $\Pi^j_m$. The winning process is then 
\begin{equation}\label{choicePois}
\hatj \coloneqq \arg\min_{j\in J} \tau^{\Pi^j}.
\end{equation}
However, if $\tau^{\Pi^{\hatj}}\geq T$, it means that evidence accumulation is not sufficient to reach a decision and the model does not categorize the object. Note that as before, the hitting time $\tau^{\Pi^{\hatj}}$ and choice $\hatj$ depend on $m$ but we drop the index for simplicity. The corresponding reaction time is then $\tau^\Pi \coloneqq \tau^{\Pi^{\hatj}}\wedge T$. \label{def:reactimepi}

Similar to the DDM, here the intensities do not evolve with the number of presented objects $m$: they represent prior knowledge about the categories and object natures and are fixed. The higher the intensity $\gamma^j_o$, the higher the average accumulation of evidence in favor of category $j$.

The Poisson process is commonly used to model the activities of spiking neurons coding information through their firing rate, and which do not interact as a network. Therefore, the Poisson counter model can have the following interpretation: each category $j$ is coded by a spiking neuron $j$ whose activity when presented with the $m^{th}$ object is $(\Pi^j_{m,t})_{t\geq 0}$. The spike count of neuron $j$ represents the accumulation of evidence in favor of category $j$.

\subsubsection{Theoretical guarantees}

We can conduct a similar analysis as for the DDM by assuming a similar assumption representing prior knowledge about the categories. Let $\Dgam>0$.
\vspace{2mm}
\begin{assumption}[Margin($\Dgam$)] \label{assump intensity poisson}
For every category $j^*\in J$, for every $j\neq j^*$, for every object nature $o\in j^*$ we have
     \[
    \gamma^{j^*}_o > \gamma^{j}_o + \Dgam.
    \]
\end{assumption}

\vspace{2mm}
\begin{theorem} \label{th poisson correct categ}
    Let $\alpha\in (0,1)$ and let $m$ be a fixed positive integer. %
    Let $\Dgam >0$ and suppose Assumption~\ref{assump intensity poisson} (Margin($\Dgam$)) holds. If the threshold $\thr$ satisfies
    \begin{equation} \label{eq thresh poisson}
         \frac{8}{3}\log\left(\frac{2\abs{J}}{\alpha}\right) \max\left(\frac{\norm{\gamma}_{\infty}}{\gamma_{\text{min}}},  4\left(\frac{\norm{\gamma}_{\infty}}{\Dgam}\right)^2\right)\leq \thr  \leq  \Dgam T -  \sqrt{\frac{8}{3} \norm{\gamma}_{\infty} T \log\left(\frac{2\abs{J}}{\alpha}\right)}
    \end{equation}
    where $\gamma_{\text{min}} \coloneqq \min_{o\in \Obj,j\in J}\{\gamma^i_o; \gamma^j_o >0\}$, then with probability
    larger than $1-\alpha$ we have that for all nature $o\in \mathcal{O}$, if we denote by $j^*\in J$ its category (i.e. $o\in j^*$), the choice $\hatj$ \eqref{choicePois} and the reaction time $\tau^{\Pi^{\hatj}}$ \eqref{tauPois} that take place at the presentation of the $m^{th}$ object of nature $o$, satisfy %
    \[
    \hatj = j^* \quad  \mbox{and} \quad  \tau^{\Pi^{\hatj}}<T.
    \]
\end{theorem}

Similarly as  Theorem~\ref{th ddm} about the DDM, this theorem provides conditions on the threshold $\thr$ under which the model predicts the correct category with high probability. Regarding the limit duration $T$, the upper bound has the same order of magnitude in $O(T)$. Here, the lower bound depends on the quantity $\gamma_{\text{min}}$, which was not the case in Theorem~\ref{th ddm}. Finally, it should be noted that similarly as in Theorem~\ref{th ddm}, the larger the constant $\Dgam$, the larger the size of the possible thresholds interval.

\subsection{Coupling between Drift Diffusion Model and Poisson Counter Model} \label{sec coupling ddm poisson}
 In this section, we explain the high similarity between the behaviors of the Poisson model and the Drift Diffusion 
 model by proving that one model can be strongly approximated by the other and that both models provides similar 
 reaction times, up to negligible terms. Indeed, under a certain framework, we can build a coupling between the 
 processes given by the two models. Here, we assume that we dispose of $n$ independent 
 \label{def:notan}
 copies of the Poisson counter 
 model: during the presentation of object $m$, each category $j$ is coded by $n$ independent Poisson processes, all 
 having intensity $\gamma^j_o$ where $o$ is the nature of object $m$. The evidence in favor of category $j$ is then 
 the sum of the count of the $n$ processes coding for category $j$ and is denoted by $\Pi^{j,n}_m = (\Pi^{j,n}_{m,t})_{t\geq 0}$. 
 By independence of the neurons coding for category $j$, the process $\Pi^{j,n}_{m}$ is then a homogeneous Poisson process 
 with intensity $n\gamma^j_o$. The presented object is classified in the category coded by the first process $\Pi^{j,n}_{m}$ 
 among the family of processes $\Pi^n_m\coloneqq(\Pi^{j,n}_m)_{j\in J}$ to reach the threshold $n\thr$  (or not classified at all 
 if none of the processes reaches $n\thr$). We chose here the threshold to be proportional to $n$ to allow the coupling to 
 take place. Up to this modification, the decision-making process is identical to \eqref{tauPois} and \eqref{choicePois}.

In the following theorem, we prove a coupling between this model and the drift diffusion model with drift $\mu^j_o = n\gamma^j_o$.
\vspace{2mm}
\begin{theorem} \label{th coupling ddm poisson}
Let $m\in \N^*$, $o\in \Obj$ the nature of the $m^{th}$ object and $n\in \N^*$. Let 
$\Pi^n_m = (\Pi^{j,n}_m)_{j\in J}$ a family of mutually independent Poisson processes with intensities 
$(n\gamma^j_o)_{j\in J}$. Suppose the vector $\gamma_o$ verifies $\gamma^j_o >0$ for every $j\in J$. Then there exist absolute constants $a,b,d>0$ and a $\abs{J}-$dimensional drifted 
Brownian motion $\WW^n_m = (\W^{j,n}_m)_{j\in J}$ defined by \eqref{eq def w}
such that for every $x>0$,
\begin{equation}
\label{eq coupling poisson ddm}
\P\left(\sup_{j\in J, t\geq 0} \frac{|\Pi^{j,n}_{m,t}-\W^{j,n}_{m,t}|}{\log(n\gamma^j_o t \vee 1+1)} \geq a+x\right)\leq b\abs{J} e^{-d x}.
\end{equation}
For $j\in J$, let us define the hitting times 
\label{def:taupi}$\taupi \coloneqq \inf \left\{u\geq 0; \Pi^{j,n}_{m,u} \geq n\thr \right\}$ and 
\label{def:tauw}
$\tauw \coloneqq \inf \left\{u\geq 0; \W^{j,n}_{m,u} \geq n\thr \right\}$. Consequently, there exist absolute positive constants $C,K$ and $\zeta$ such that for $\alpha \in (0,1)$, $\thr>0$ and $n\geq \frac{2}{\gamma_{\text{min}}}\left(\frac{\thr}{\gamma_{\text{min}}}+1\right)\log\left(\frac{10\abs{J}}{\alpha}\right)$, %
with probability more than $1-\alpha$, all the following inequalities hold jointly:
\begin{enumerate}
    \item The processes $\Pi^n_m$ and $\WW^n_m$ verify
    \[  \sup_{j\in J,t\geq 0} \frac{|\Pi^{j,n}_{m,t}-W^{j,n}_{m,t}|}{\log(n\gamma^j_o t \vee 1+1)} \leq C +\zeta\log(5K\abs{J}\alpha^{-1}).
    \]
    \item The hitting times $(\taupi)_{j\in J}$ verify
    \[
    \sup_{j\in J} \AAbs{\taupi - \frac{\thr}{\gamma^j_o}} \leq \sqrt{ \frac{8}{3n\gamma_{\text{min}}} \left(\frac{\thr}{\gamma_{\text{min}}} + 1\right) \log\left(\frac{10\abs{J}}{\alpha}\right)} + \frac{1}{n\gamma_{\text{min}}}.
    \]
\item The hitting times $(\tauw)_{j\in J}$ verify
\[
\sup_{j\in J}\AAbs{\tauw - \frac{\thr}{\gamma^j_o}} \leq \sqrt{
\frac{2}{n\gamma_{\text{min}}}\left(\frac{\thr}{\gamma_{\text{min}}} + 1\right) \log\left(\frac{10\abs{J}}{\alpha}\right)}.
\]
\item The difference between the hitting times $(\taupi)_{j\in J}$ and $(\tauw)_{j\in J}$ verifies
\[
\sup_{j\in J}\abs{\taupi - \tauw} \leq \Box_{\thr,\gamma_{\text{min}},\norm{\gamma}_\infty} \frac{\log(n)}{n} \log\left(\frac{\abs{J}}{\alpha}\right).
\]
\end{enumerate}
\end{theorem}

The proof of the theorem, provided in Appendix 
\ref{proof th coupling brown poiss}, relies on a coupling result between Poisson and Brownian processes established by \cite{ethier2009markov}, concentration inequalities, and a novel inequality for Brownian processes that we derived. 

With this theorem, we establish a coupling between the Poisson counter model and the drift diffusion model
by bounding the supremum of their trajectories on all $\R_+$. Since 
$\mathbb{E}[\Pi^{j,n}_{m,t}] = \mathbb{E}[\W^{j,n}_{m,t}] = n \gamma^j_o t$ are growing linearly with $n$, the 
bound in $O(\log(n))$ of \eqref{eq coupling poisson ddm} is very strong. Note that the theorem would still hold for
dependent processes $\Pi^n_m = (\Pi^{j,n}_m)_{j\in J}$ but the coupled processes 
$\WW^n_m=(W^{j,n}_m)_{j\in J}$ would have inherited from this dependence. Note also that the assumption 
$\gamma^j_o >0$ for every $j\in J$ is easy to satisfy: since the processes with 
intensity $\gamma^j_o = 0$ stay equal to zero and never reach $\theta$, we can consider the set of 
processes with positive intensities instead of $J$. 

Besides, in this framework where the threshold $n\thr$ is linear in $n$, we establish for every $j$ the 
convergence of the hitting times of both Poisson process $(\Pi^{j,n}_{m,t})_{t\geq 0}$ and Brownian 
process $(\W^{j,n}_{m,t})_{t\geq 0}$ to $\frac{\thr}{\gamma^j_o}$, which is the hitting time of the 
deterministic processes $(\mathbb{E}[\Pi^{j,n}_{m,t}])_{t\geq 0}$ and $(\mathbb{E}[\W^{j,n}_{m,t}])_{t\geq 0}$, with rate in $n^{-1/2}$.
 Furthermore, we prove that the difference between these two hitting times decreases faster than their 
 convergence rate. This ensures that for large $n$, both models select the same category $\hatj$ 
 (the one with larger $\gamma^j_o$), and we can apply the same bound to the difference between the reaction 
 times of the two models: we have
 \[\AAbs{\tau^\Pi - \tau^W}  = \AAbs{\tau^{\Pi^{\hatj}}_n \wedge T - \tau^{\Pi^{\hatj}}_n \wedge T} = 
 O\left(\frac{\log(n)}{n}\right)\]
 with high probability. %
 This result explains in particular how reaction times that are modeled by Poisson counter processes, and 
 which are more biologically relevant, might exhibit a behavior like the DDM and have therefore 
 similar adequation to real data.

\section{Hawkes counter model} \label{sec hawkes counter model}

\subsection{Framework}

In this section, we want to model a different decision-making task: the participant does not have prior knowledge about the categories anymore. 
Similarly as before, the participant has to categorize objects having natures belonging to a set $\mathcal{O}$ into one of several categories 
belonging to a set $J$. The participant sees the objects sequentially and is given the true answer after each categorization, which allows for 
learning the categories. Therefore, unlike in the previous task, after each presentation of an object, the participant knows better the categories: 
the average strength of the evidence accumulation in favor of the categories will change when presented with a new object. The total number of 
presented objects which allow the participant to learn the categories is denoted by $M$, 
\label{def:notaM}
and the sequential presentation of these $M$ objects 
is called a learning phase. We suppose that at the end of the learning phase, the participant has learned to recognize the object categories, 
and the average strength of the evidence accumulation in favor of the categories will not change anymore.

\subsection{The model}

We model this new experiment by a spiking neural network inspired by a previous work in discrete time 
and without reaction times  \citep{jaffard2024provable}. During the presentation of the $m^{th}$ 
object, the accumulation of evidence of the categories $j\in J$ is modeled by a $\abs{J}-$dimensional 
Hawkes process $N_m = (N^j_m)_{j\in J}$ \label{def:nj} with conditional intensity
 $\lambda_{m,t} = (\lambda^j_{m,t})_{j\in J}$ defined below. This multivariate process represents the 
 output nodes activity of a spiking neural network.

The $m^{th}$ object is presented for a duration of at least $T_{\text{min}}$ \label{def:tmin} and until either the spike 
count of one of the processes of $N_m=(N^j_m)_{j\in J}$ reaches the threshold $\thr$, or the maximum 
allowed duration $T > T_{\text{min}}$ is reached. For $j\in J$, We denote 
\begin{equation}\label{tauHaw}
\tau^{N^j} \coloneqq \inf\{u\geq 0; N^j_{m,u}\geq \thr\}.
\end{equation}
The winning process is then 
\begin{equation}\label{choiceHaw}
\hatj = \arg\min \tau^{N^j},
\end{equation}
and if $\tau^{N^{\hatj}} <T$ then object $m$ is categorized
in category $\hatj$. However, if $\tau^{N^{\hatj}} \geq T$ it means that accumulation evidence is 
not sufficient to reach a decision and the model does not categorize the object. The corresponding 
reaction time of the whole system is $\tau^N \coloneqq (T_{\text{min}}\vee \tau^{N^{\hatj}})\wedge T$. 
\label{def:reactimeHawkes}
Similarly as for the drift diffusion and Poisson counter models, the processes are indexed by time $t\in[0,T]$ which is set back to $0$ at each new object.

An illustration of the network is given on Figure~\ref{fig réseau fusées} for the case of the experiment described in Section~\ref{sec expe}. 
The input layer codes for features describing the object natures. The set of input neurons is denoted by $I$, 
\label{def:featureI}
and an element $i\in I$ 
denotes both an input neuron and the feature it is encoding. For instance, in the experiment of Section~\ref{sec expe}, 
the objects are rockets and the features are the shape of the head, body and fins and the number of flames. 
The output layer codes for the categories in which the objects are classified: the activity of the output neurons represent the evidence 
accumulation in favor of each category. 
Therefore, the set of output neurons is denoted by $J$ and an element $j\in J$\label{def:categoryj} denotes both an output neuron and the category it is coding for. 
\newline

\textbf{Input neurons activity.} During the presentation of the $m^{th}$ object, the input neuron $i\in I$ starts spiking as a 
homogeneous Poisson process $(\Pi^i_{m,t})_{t\geq 0}$ with intensity $\gamma^i_o \geq 0$ 
\label{def:gammai0}
depending only on the nature $o$ of 
the presented object and representing how pronounced feature $i$ is in object nature $o$. For instance, in the modeling of 
the experiment described in Section~\ref{sec expe}, we use a binary code: either the rocket has feature $i$ and  the input 
neuron $i$ starts spiking with a strictly positive fixed intensity, either the object does not have the feature and the 
corresponding input neuron $i$ stays silent (\ie its intensity is equal to $0$). We assume that the processes 
$((\Pi^i_{m,t})_{t\geq 0})_{i\in I, 1\leq m \leq M}$ are mutually independent. 
\newline

\textbf{Output neurons activity.} During the presentation of the $m^{th}$ object, evidence accumulation
in favor of category $j$ is modeled by the activity of the output neuron $j\in J$, which starts spiking
 as a Hawkes process 
 \label{def:Njmt} $(N^j_{m,t})_{t\geq 0}$, which, in contrast to the Poisson process, allows to take
  into account interactions with presynaptic neurons. Its conditional intensity at time $t$ of the presentation of the object is given by
\[
\lambda^j_{m,t} \coloneqq \sum_{i\in I} \int_0^t w^{i\to j}_m g(t-s) d\Pi^i_{m,s} \label{def:g}
\]
where $g\in L^1(\R_+)$ is a non negative function and $w^{i\to j}_m \geq 0$ is the synaptic weight between input neuron $i$ and output neuron $j$: 
it represents the strength of their connection during the presentation of the $m^{th}$ object. We denote by $\norm{g}_1\coloneqq \int_0^\infty g(s)ds$ the norm 
of $g$.%
 The weights $w^j_m \coloneqq (w^{i\to j}_m)_{i\in I}$ are in the simplex: for every $i\in I$, $w^{i\to j}_m \geq 0$ and $\sum_{i\in I} w^{i\to j}_m = 1$, and 
 they may be seen as a probability distribution over the set of presynaptic neurons of $j$. We denote by $w_m\coloneqq (w^j_m)_{j\in J}$ the total weight family. 
\begin{figure}
    \centering
    \includegraphics[width=1.\linewidth]{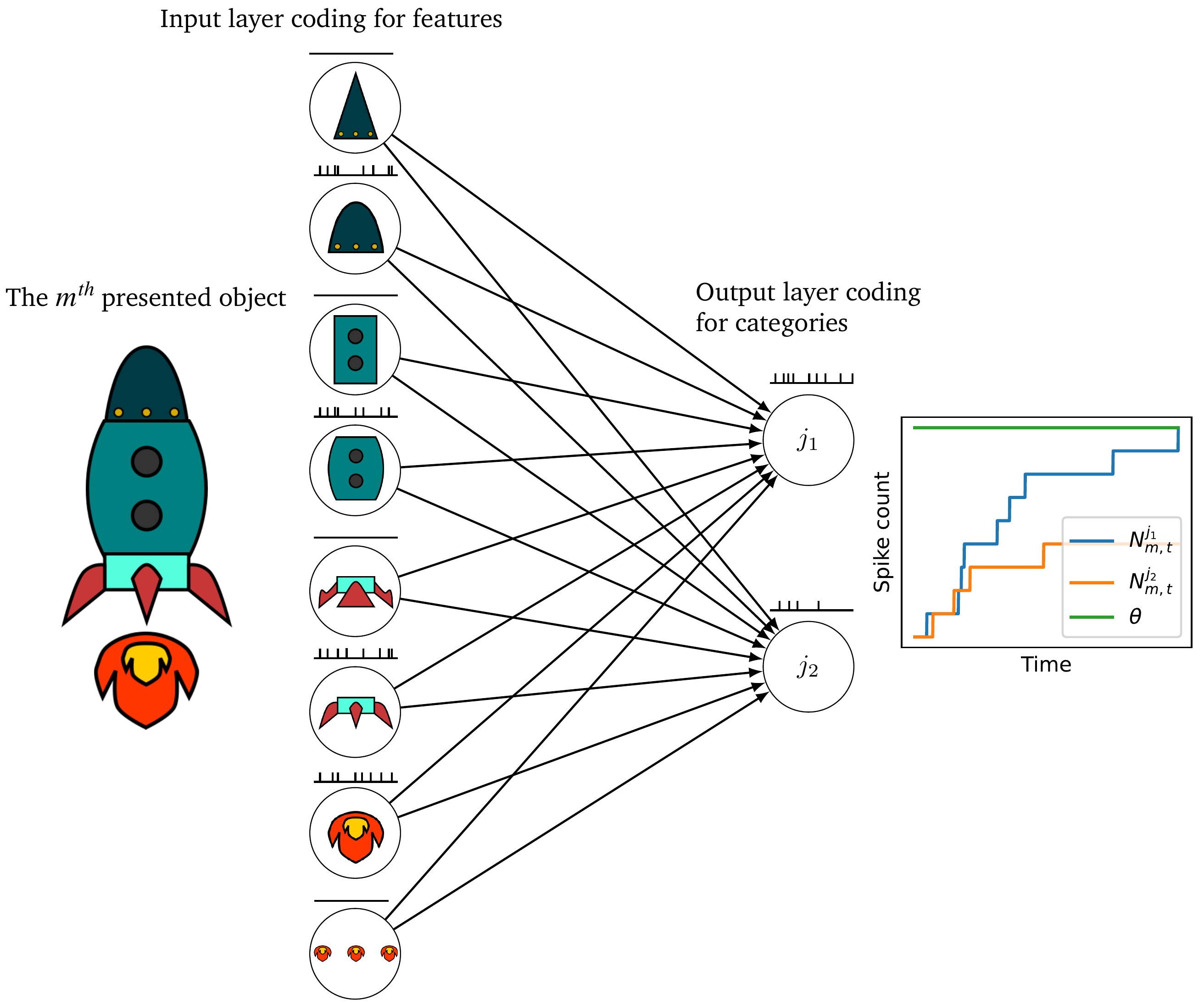}
    \caption{Illustrative example of the network used to model the experiment described in section. The task is to classify rockets into two categories, $j_1$ and $j_2$. The input neurons code for different shapes of head, body, fins and flames. The presented object excites the input neurons coding for its features, whereas other input neurons stay silent. The output neurons start spiking with a conditional intensity depending on their connection with the input neurons. The spike train of each neuron is represented on top of it. The first spike count to reach the threshold $\thr$ is the one of neuron $j_1$: therefore, the rocket is categorized in category $j_1$. }
    \label{fig réseau fusées}
\end{figure}
\newline

\textbf{Learning rule.} The network learns by updating its synaptic weights after each presented object by using an expert aggregation algorithm \citep{cesa2006prediction} 
thanks to this local paradigm: each presynaptic neuron $i$ can be seen as an expert and the strength of the connection $w^{i\to j}_m$ between $i$ and the postsynaptic 
neuron $j$ varies based on gains derived from these connections. After the presentation of the $m^{th}$ object, the postsynaptic neuron $j$ attributes the following 
gain to the presynaptic neuron $i$:
\[\label{def:gijm}
g^{i\to j}_m = \left\{
    \begin{array}{lll}
         \widehat{\gamma^i_m} \times \frac{M}{M^j} &\text{if } o \in j  \\ 
          \\
        - \widehat{\gamma^i_m} \!\times\! \frac{M}{M^{j'}} \!\times\! \frac{1}{\abs{J}-1} &\text{if } o \in j'\neq j 
    \end{array}
\right.
\]
where $o$ is the nature of the presented object, $M^j$ is the total number of presented objects belonging to category $j$ during the learning phase and
\begin{equation} \label{eq gamma hat}
    \widehat{\gamma^i_m} \coloneqq \Pi^i_{m,T_{\text{min}}} / T_{\text{min}}
\end{equation}
is the empirical firing rate of input neuron $i$ during duration $T_{\text{min}}$. This gain can be interpreted as follows: if the presented object belongs to category $j$, neuron $j$ should spike frequently to be the first to reach the threshold $\thr$. Therefore, it assigns positive gains to its presynaptic neurons, with these gain being proportional to their firing rates, thereby strengthening its connections with the input neurons that encode the most relevant features describing category $j$. Conversely, if the presented object does not belong to category $j$, neuron $j$ should spike less to avoid reaching the threshold $\thr$ before the neuron coding for the object's true category. In this case, it assigns negative gains (\ie losses) to its presynaptic neurons in order to reduce the strength of its connections with the most active neurons when presented with objects it should not encode. Note that neuron $j$ does not need to know what the other neurons $j'$ are doing to attribute the gains.  Also note that the gain depends on the ratio $M/M_j$, that is roughly speaking the proportion of objects in each category, and the larger $T_{\text{min}}$, the better the estimation of the input neurons firing rates.

We denote by $G^{i\to j}_m \coloneqq \sum_{m'=1}^m g^{i\to j}_{m'}$ 
\label{def:Gitojm}
the cumulated gain of input neuron $i$ w.r.t. output neuron $j$ until the $m^{th}$ object. Then after the presentation of the object, the weights of neuron $j$ are updated using the expert aggregation EWA (Exponentially Weighted Average)  \citep{cesa2006prediction}:
\begin{equation} \label{eq update w}
    w^{i\to j}_{m+1} \coloneqq \frac{\exp(\eta G^{i\to j}_m)}{\sum_{l\in I} \exp(\eta G^{l\to j}_m)}
\end{equation}
where the parameter $\eta >0$ is called the learning rate and determines how fast the network learns. In what follows, we choose 
\begin{equation}
\label{eq:defeta}
    \eta = \frac{\eta_0}{\sqrt{M}} \mbox{ where } \eta_0>0.
\end{equation}
Note that this learning rule is local: the gains used by neuron $j$ and therefore the update of the synaptic weights depend only on the activity of its presynaptic neurons.

We denote by 
\label{def:mathcalFm}$\mathcal{F}_m$ the $\sigma$-algebra generated by every event which happened until the end of the presentation of the $m^{th}$ object and $\mathcal{F}_0$ the trivial $\sigma-$algebra. Note that the weights $w^{i\to j}_{m+1}$ are $\mathcal{F}_m$-measurable.
\newline

\textbf{Additions over the previous model.} This new version differs from the model that we proposed in \cite{jaffard2024provable} by two main features. In this previous work, the neurons activities were defined in discrete time, and every object was presented during a fixed duration: the presented object was classified in the category of the output neuron with the highest spike count at the end of the presentation. Here, the neurons activities are defined in continuous time and the objects are presented for dynamic durations. To model behavior, we introduced a threshold and a limit duration, which allowed to model participant reaction times and the possibility that the participant will not decide between the different categories if not enough evidence has been accumulated.
\newline

\textbf{Comparison with the Poisson counter model and interpretation.} Similarly to the Poisson counter model, the evidence accumulation in favor of the different categories are interpreted as spike counts of neurons coding for these categories. However, unlike the Poisson counter model, the intensities of the output neurons are not given arbitrarily as prior knowledge about the categories, but are learned with a local rule as the network sees new objects. Our model can have the following interpretation: the input layer is a Poisson counter model which can recognize simple features, and an output layer is added and trained to recognize more complex concepts. This is biologically relevant since in the brain, elementary concepts can be depicted by the responses of individual neurons, whereas more intricate ones are represented by groups of interconnected neurons that work together \citep{singer1997neuronal}.

\subsection{Mathematical analysis}

In this section, we establish theoretical guarantees for our model. We begin by analyzing the asymptotic behavior of the network weights, which provides insights into the network behavior after the learning phase. Next, we outline the conditions necessary for the network to accurately categorize new objects following the learning process. This successful performance stems not from prior knowledge of the categories but from the learning phase itself. Finally, we demonstrate that the drift diffusion model can be derived from our model, establishing that it can be approximated by complex networks of interacting neurons having a learning mechanism.

\subsubsection{Asymptotic behavior}

Let us study the network's asymptotic behavior. We introduce a notation: given a set $E$ and a quantity $x_e \in \R$ indexed by $e\in E$, we denote its mean by $\brac{x_e}_{e\in E} \coloneqq \frac{1}{\abs{E}}\sum_{e\in E} x_e$.
\vspace{2mm}
\begin{definition}[Feature discrepancy]
    Let $i\in I$ and $j\in J$. The feature discrepancy of input neuron $i$ w.r.t. output neuron $j$ is 
    \[\label{def:ditoj}
    d^{i\to j} \coloneqq \brac{\gamma^i_o}_{o\in j} - \brac{\brac{\gamma^i_o}_{o\in j'}}_{j'\in J\setminus\{j\}}.
    \]
\end{definition}

The feature discrepancy of input neuron $i$ w.r.t. output neuron $j$ measures how sensible neuron $i$ is to category $j$ by comparing its average firing rate when presented with objects of category $j$ to its average firing rate when presented with objects of other categories. We can now define the set of input neurons which are the most sensible to category $j$ by $I^j \coloneqq \arg\max_{i\in I} d^{i\to j}$, 
\label{def:set of input most sensible to j} as well as the gap in discrepancy when $I^j \subsetneq I$
\[
\delta^j \coloneqq \max_{i\in I} d^{i\to j} - \max_{i\in I\setminus I^j} d^{i\to j}
\]
which measures how sensible the best input neurons are with respect to the others.
\vspace{2mm}
\begin{assumption} \label{assump nb obj}
    During the learning phase, each nature $o$ has been presented the same amount of times to the network, \ie $\frac{M}{\abs{\Obj}}$ times. 
\end{assumption}
\vspace{2mm}
\begin{proposition} \label{prop conv weights}
    Let $\alpha \in (0,1)$ and $\gamma_{\text{min}} \coloneqq \min_{o\in \Obj, i\in I}\{\gamma^i_o; \gamma^i_o>0\}$. 
    Suppose Assumption~\ref{assump nb obj} holds and
    \begin{equation} \label{eq M tmin}
          M T_{\text{min}} \geq \frac{8\abs{\Obj}}{3\gamma_{\text{min}}} \log\left(\frac{2\abs{I}\abs{J}}{\alpha}\right) .
    \end{equation}
    For $j\in J$, let $w^j_\infty = (w^{i\to j}_\infty)_{i\in I}$ the weight family such that $w^{i\to j}_\infty \coloneqq \frac{1}{\abs{I^j}}\mathbb{1}_{\{i\in I^j\}}$.\label{def:witojinfty} 
    Then with probability $1-\alpha$, at the end of the learning phase, the synaptic weights verify for all $j\in J$
    \begin{align*}
        \Norm{w^j_{M+1} - w^j_\infty}_2 \leq  2 \abs{\Obj}\eta_0 &\sqrt{\frac{8 \abs{I}}{3 T_{\text{min}}} \norm{\gamma}_{\infty} \log(2\abs{I}\abs{J}\alpha^{-1})} \\
        &+ \mathbb{1}_{\{I^j \subsetneq I\}}  \frac{\sqrt{\abs{I}}}{\abs{I^j}}\max\left(1, \frac{\abs{I} - \abs{I^j}}{\abs{I^j}} \right)\exp(- \eta_0\delta^j \sqrt{M}),
    \end{align*}
    where $\eta_0$ is defined by \eqref{eq:defeta}.
\end{proposition}
In Proposition~\ref{prop conv weights}, we prove that the network weights converge to a limit family $w_\infty \coloneqq (w^{i\to j}_\infty)_{i\in I, j\in J}$ 
and we provide rates of convergence. The error is twofold. The first part, in $O(\frac{1}{\sqrt{T_{\text{min}}}})$ where $T_{\text{min}}$ is 
the minimum duration for which an object is presented, comes from the randomness of the system and is derived using concentration inequalities 
to control how good is the estimation of $\gamma^i_o$ by $\widehat{\gamma^i_m}$ \eqref{eq gamma hat}. The second part, in $\exp(O(-\delta^j \sqrt{M})$ 
where $M$ is the number of objects presented during the learning phase, comes from the use of the expert aggregation algorithm EWA: the larger the gap 
in discrepancy, the faster the convergence. This limit family has the following interpretation: at the limit, output neuron $j$ is connected only and 
with equal strength to input neurons which are the most sensible to category $j$.

This result can be compared to Theorem $3.4$ of \cite{jaffard2024provable} which describes the asymptotic behavior of the network in discrete time with fixed object presentation durations. The present error has the same order of magnitude, where the number of time steps of each presented objects is replaced by the minimum duration of each object presentation $T_{\text{min}}$.

\subsubsection{Theoretical guarantees}

At the limit ($M$ and $T_{min}$ tend to infinity), the network's weights are the deterministic family $w_\infty$ given by Proposition~\ref{prop conv weights}. For $o\in \Obj$, let
\[
\Bar{\lambda}^j_{\infty, o} \coloneqq \sum_{i\in I}  w^{i\to j}_\infty \gamma^i_o
\]
which is the asymptotic average firing rate of neuron $j$ when objects with nature $o$ are presented. Therefore, the average number of spikes of neuron $j$ on $[0,t]$ with weights $w^j_\infty$ when presented with an object of nature $o$ is 
\[
\Bar{\lambda}^j_{\infty, o} \mathbb{G}(t)
\]
where $\mathbb{G} : t \mapsto \int_0^t \int_0^s g(s-u)du ds$. 
\label{def:mathbbG} Hence, on average, the first neuron whose spike count reaches threshold $\thr$ is the one with highest $\Bar{\lambda}^j_{\infty, o}$. This leads to this new notion of margin. Let $\Dlamb>0$.
\vspace{2mm}
\begin{assumption}[Margin($\Dlamb$)] \label{assump limit fam}
For every category $j^*\in J$, for every $j\neq j^*$, for every object nature $o\in j^*$ we have
      \[
    \Bar{\lambda}^{j^*}_{\infty, o} > \Bar{\lambda}^j_{\infty, o} + \Dlamb.
    \]
\end{assumption}
It means that at least in the limit the same gap as the ones for DDM and Poisson counter model holds.
Since we explicitly know the limit family $w_\infty$, this assumption can be interpreted as a condition on the categories $j \in J$ and the features represented by the input neurons  $i \in I$. This condition holds when the categories can be described as combinations of relevant features. In Section~\ref{sec expe modeling}, we prove that this assumption is satisfied in the task performed by the participants of the experiment. For a detailed characterization of the categories learnable by the discrete time network, we refer the reader to Section $4.2.4$ of \cite{jaffard2024chani}, where we describe the categories learnable by an extended version of our network that includes hidden layers, but no reaction time. 

This assumption can be related to Assumption~\ref{assump drift} (Margin($\Dmu$)) on the drift of the DDM
 and Assumption~\ref{assump intensity poisson} (Margin($\Dgam$)) on the intensities of the Poisson counter model. Indeed, similarly as in the theoretical studies of these two models, we need an assumption on the average evidence accumulation in favor of the several categories. However, unlike for the DDM and the Poisson counter model, this assumption does not reflect prior knowledge on the categories, but the ability of our network to code for the categories. It is only after the network has learned that it reaches these firing rates.
 
 Under these assumptions, we can formulate this result.

\vspace{2mm}
\begin{theorem} \label{th correct classif hawkes}
Let $\alpha\in(0,1)$ and $\eta_0>0$, and $\mathbb{G} : t \mapsto \int_0^t \int_0^s g(s-u)du ds$. Let $\Dlamb>0$ and suppose Assumptions~\ref{assump nb obj} and~\ref{assump limit fam} (Margin ($\Dlamb$)) hold. Let the parameters $M$, $T_{\text{min}}$, $T$, and $\thr$ verify
\begin{equation} \label{cond M}
    M\geq \Box \frac{1}{\eta_0^2 \delta_{\text{min}}^2} \log\left( \frac{\abs{I}}{\abs{I}_{\text{min}}} \right)^2 + \log\left(\frac{\norm{\gamma}_\infty}{\Dlamb}\right)^2
\end{equation}
where $\abs{I}_{\text{min}} \coloneqq \min_{j\in J} \abs{I^j}$ and $\delta_{\text{min}} \coloneqq \min_{j\in J} \delta^j$, 
\begin{equation} \label{cond Tmin}
    T_{\text{min}} \geq \Box \eta_0^2 \abs{\Obj}^2  \abs{I}^2 \frac{\norm{\gamma}_\infty^3}{ \Dlamb^2} \log\left(\frac{\abs{I}\abs{J}}{\alpha}\right),
\end{equation}
\begin{equation} \label{cond T}
   \mathbb{G}(T) > \Box_{\norm{g}_1} \log\left(\frac{\abs{J}}{\alpha}\right) \left(\frac{1}{\sqrt{\Dlamb}}\vee \frac{\norm{\gamma}_{\infty}}{\Dlamb}\right) ,
\end{equation}
\begin{equation} \label{cond thr inf}
    \thr \geq \left( \Box_{\norm{g}_1} \log\left(\frac{\abs{J}}{\alpha}\right) \left(\frac{\norm{\gamma}_\infty}{\sqrt{\Dlamb}}\vee \frac{\norm{\gamma}_{\infty}^2}{\Dlamb}\right)\right) \vee  \norm{\gamma}_\infty \mathbb{G}(T_{\text{min}})
\end{equation}
and
\begin{equation} \label{cond thr sup}
     \thr < \Box \Dlamb \mathbb{G}(T) - \Box_{\norm{\gamma}_\infty,\norm{g}_1} \left(\sqrt{ \mathbb{G}(T)\log\left(\frac{\abs{J}}{\alpha}\right)} + \log\left(\frac{\abs{J}}{\alpha}\right)\right).
\end{equation}
Then with probability larger than $1-\alpha$, we have that for all nature $o\in \mathcal{O}$ of the object presented after the learning phase, that is the $M+1$ object, if we denote by $j^*$ its category, the choice $\hatj$ \eqref{choiceHaw} and the reaction time $\tau^{N^{\hatj}}$ \eqref{tauHaw} satisfy
   \[
   \hatj = j^* \quad \mbox{and} \quad \tau^{N^{\hatj}} < T.
   \]
\end{theorem}

This theorem provides conditions on the minimal duration $T_{\text{min}}$, maximal duration $T$, number of presented objects during the learning phase $M$ and threshold $\thr$ required for the model to correctly categorize objects after learning. Let us start by commenting on these conditions. 
\begin{enumerate}
    \item \textbf{Sufficient learning phase duration:} The network effectively learns the categories if $M$ verifies \eqref{cond M}, ensuring a long enough learning phase.
    \item \textbf{Sufficient minimal observation time for learning objects:} The network has sufficient time to infer the input neurons firing rate to update its synaptic weights if $T_{\text{min}}$ verifies \eqref{cond Tmin}.
    \item \textbf{Sufficient threshold:} The spike count difference between the true category process $N^{j^*}_{M+1}$ and others is significant at decision time if $\thr$ verifies \eqref{cond thr inf}.
    \item \textbf{Sufficient maximal observation time:} 
    The network has sufficient time to detect differences in spike counts between the true category process $N^{j^*}_{M+1}$ and others if $T$ verifies \eqref{cond T}.
    \item \textbf{Sufficiently low threshold:} The network has a sufficiently low threshold to be able to take a decision before limit duration $T$ if $\thr$ verifies \eqref{cond thr sup}.
\end{enumerate}
Note that we always have $\mathbb{G}(t)\leq \norm{g}_1 t$ and if the function $g$ has compact support of negligible length compared $T$, as it is assumed in the following section, then $\mathbb{G}(T)\approx \norm{g}_1 T$. In this case, the upper bound on $\thr$ \eqref{cond thr sup} has the same order of magnitude regarding the limit duration $T$ as the those in Theorem~\ref{th ddm} for the DDM and Theorem~\ref{th poisson correct categ} for the Poisson counter model, \ie in $O(T)$. It is reasonable to make such an assumption because a participant's reaction time is typically around a few seconds and in practice we chose $T=5s$ in the experiment described in Section~\ref{sec expe}, whereas the effect of a presynaptic neuron's spike on a postsynaptic neuron, represented by the support of the function $g$, lasts only a few milliseconds.

The same way as for the constant $\Dmu$ quantifying the larger drift for the DDM and the constant $\Dgam$ quantifying the larger gap in intensity for the Poisson counter model, here, the larger $\Dlamb$, \ie the larger the gap between the limit firing rate of the output neuron coding for the true category and the other output neurons, the larger the size of the thresholds interval enabling a correct categorization with high probability. Similarly as in Theorem~\ref{th ddm}, we need the maximal duration $T$ to be sufficiently large \eqref{cond T} and unlike in Theorems~\ref{th ddm} and~\ref{th poisson correct categ} where there was no learning mechanism, additional assumptions are needed here regarding $T_{\text{min}}$ and $M$ to ensure a correct learning. Note that all these assumptions become weaker as $\Dlamb$ increases. 

The proof, detailed in Appendix~\ref{proof th correct classif hawkes}, relies on several key components. First, concentration inequalities are employed to prove that the processes $(N^j_{M+1,t})_{t\geq 0}$ are closely approximated by their compensators. Proposition~\ref{prop conv weights} is then used to show that these compensators are close to the average spike counts with limit weights. Finally, Assumption~\ref{assump limit fam} is applied to derive conditions under which the network correctly categorizes the new object.

\subsection{Coupling between the Drift Diffusion Model and the Hawkes Counter Model} 

Similarly as for the Poisson counter model (Section~\ref{sec coupling ddm poisson}), we can build a 
coupling between the processes given by the drift diffusion model and the Hawkes counter model under a 
certain framework. Here, we work conditionally to %
what happened until object $m$, whether it is during learning or post learning: the synaptic weights $w^{i\to j}_m$ are therefore
assumed to be deterministic. We suppose that we dispose of $n$ independent copies of the network, and 
each independent network is governed by the same weights $(w^{i\to j}_m)_{i\in I,j\in J}$. During 
the presentation of object $m$, for each feature $i$, there are $n$ input neurons spiking as a 
Poisson processes with intensity $\gamma^i_o$ where $o$ is the nature of object $m$: we denote by 
$(\Pi^{i,n}_{m,t})_{t\geq 0}$ the sum of their count, which is a Poisson process with intensity 
$n\gamma^i_o$ by independence. In the same way, there are $n$ output neurons coding for the same 
category $j\in J$. The evidence in favor of category $j$ is then the sum of their counts and is 
denoted $(N^{j,n}_{m,t})_{t\geq 0}$, which is a Hawkes process with intensity
\begin{equation} \label{eq intensity hawkes n}
    \lambda^{j,n}_{m,t} = \sum_{i\in I} \int_0^t w^{i\to j}_m g(t-s) d\Pi^{i,n}_{m,s}.
\end{equation}
Note that two other frameworks lead to the same processes: %
only one network of size $|I|\times |J|$ with 
input neurons intensities $(n\gamma^i_o)_{i\in I,o\in \Obj}$, but also one network with 
$n\times \abs{I}$ input neurons and only $\abs{J}$ output neurons.

During the presentation of the $m^{th}$ object, the average evidence accumulation in favor of category $j$ at present time $t$ of object $m$ is then
\[
\E[N^{j,n}_{m,t}] = n\Bar{\lambda}^j_m \mathbb{G}(t)
\]
where $\Bar{\lambda}^j_m \coloneqq \sum_{i\in I} w^{i\to j}_m \gamma^i_o$ with $o$ the nature of the presented object.
\vspace{2mm}
\begin{theorem} \label{th coupling hawkes ddm}
Let $\alpha\in (0,1)$, $m\in \N^*$, $n\in \N^*$ and $o\in \Obj$ the nature of the $m^{th}$ object. Let $N^n_m = (N^{j,n}_m)_{j\in J}$ be a $\abs{J}-$dimensional Hawkes
 process with intensity function defined by \eqref{eq intensity hawkes n} and let $\Bar{\lambda}_{\text{min}}\coloneqq \min_{j\in J} \Bar{\lambda}^j_m$. We assume that the function $g$ has compact support included in $[0,n^{-1/2}]$, that its norm $\norm{g}_1$ does not depend on $n$, and that $\Bar{\lambda}_{\text{min}}>0$. Then if $ n\geq \Box_{\norm{\gamma}_\infty,\norm{g}_1, \Bar{\lambda}_{\text{min}},\thr} \log\left(\frac{\abs{I}+\abs{J}}{\alpha}\right)^{3/2}$, there exists a $\abs{J}$-dimensional drifted Brownian motion with correlated noise 
 $\WW_m^n = (W^{j,n}_m)_{j\in J}$ with mean vector defined by $\forall j\in J, \forall t\geq 0, \quad \E[W^{j,n}_{m,t}] = n\Bar{\lambda}^j_m\mathbb{G}(t)$,
and covariance matrix defined by $\forall j_1,j_2\in J, \forall t,s\geq 0$,
\[
\cov(W^{j_1,n}_{m,t}, W^{j_2,n}_{m,s}) = n (t\wedge s) \left( \sum_{i\in I}w^{i\to j_1}_m w^{i\to j_2}_m \gamma^i_o + \mathbb{1}_{j_1 = j_2} \norm{g}_1 \Bar{\lambda}^{j_1}_m \right )
\]
such that with probability more than $1-\alpha$, all the following inequalities hold jointly:
\begin{enumerate}
    \item The processes $\Pi^n_m$ and $W^n_m$ verify
   \begin{equation} \label{eq sup N - W th}
    \sup_{j\in J, t> 0}\frac{ \AAbs{N^{j,n}_{m,t} - W^{j,n}_{m,t}}}{n^{1/4}(t\vee 1)^{1/4}\log(nt + 1)^{3/4}} \leq \Box_{\norm{g}_1,\norm{\gamma}_\infty} \log\left(\frac{\abs{I}+\abs{J}}{\alpha}\right)^{3/4}.
\end{equation}
\item The hitting times $(\taun)_{j\in J}$ defined for $j\in J$ by 
\label{def:taun}$\taun \coloneqq \inf\{u\geq 0; N^{j,n}_{m,u} \geq n\thr\}$ verify
\[
\sup_{j\in J}\AAbs{\taun - \frac{\thr}{\Bar{\lambda}^j_m \norm{g}_1}} \leq \Box_{\norm{\gamma}_\infty,\norm{g}_1, \Bar{\lambda}_{\text{min}},\thr}  n^{-1/2}\log\left(\frac{\abs{I}+\abs{J}}{\alpha}\right)^{3/4}.
\]
\item The hitting times $(\tauw)_{j\in J}$ defined for $j\in J$ by $\tauw \coloneqq \inf\{u\geq 0; W^{j,n}_{m,u} \geq n\thr\}$ verify
\[
\sup_{j\in J}\AAbs{\tauw - \frac{\thr}{\Bar{\lambda}^j_m \norm{g}_1}} \leq\Box_{\norm{\gamma}_\infty,\norm{g}_1, \Bar{\lambda}_{\text{min}},\thr} n^{-1/2}\log\left(\frac{\abs{J}}{\alpha}\right)
.
\]
\item The difference between the hitting times $(\taun)_{j\in J}$ and $(\tauw)_{j\in J}$ verifies
\begin{align} \label{eq taun - tauw}
   \sup_{j\in J}  \abs{\taun - \tauw} &\leq \Box_{\norm{\gamma}_\infty,\norm{g}_1, \Bar{\lambda}_{\text{min}},\thr}   \left(\frac{\log(n)}{n}\right)^{3/4} \log\left(\frac{\abs{I}+\abs{J}}{\alpha}\right)^{3/4}  .
\end{align}
\end{enumerate}

\end{theorem}

With this theorem, we establish a coupling between the Hawkes counter model and the drift diffusion model by bounding the supremum of their trajectories over all $\mathbb{R}_+$.
 
The proof, provided in 
Appendix~\ref{sec proof the coupling hawkes ddm}, builds on the same coupling result between Poisson 
and Brownian processes established by \cite{ethier2009markov}, which we previously used to prove 
Theorem~\ref{th coupling ddm poisson} and new inequalities that we derived for Brownian and Hawkes 
processes. 

To achieve this, we had to assume that $\Bar{\lambda}_{\text{min}}>0$, \ie that $\Bar{\lambda}^j_m>0$ for every $j\in J$. This is easy to satisfy: indeed, if $\Bar{\lambda}^j_m =0$, it means that neurons coding for category $j$ are not linked to any input neuron currently active, therefore we are in the non interesting case where the process $N^{j,n}_{m}$ stays equal to zero and never reach $n\thr$. We can then consider the set of processes such that $\Bar{\lambda}^j_m>0$ instead of $J$ with no loss of generality. We also had to assume that $g$ has compact support decreasing at a rate of $n^{-1/2}$. This 
implies that the number of spikes influencing a group of $n$ neurons encoding the same category $j$ does
not scale linearly with $n$, as it would if $g$ had a fixed support, but rather grows at a rate of 
$n^{1/2}$. This assumption further ensures that $\mathbb{G}(t) \approx \norm{g}_1 t$, meaning the mean vector of the 
drifted Brownian motion given by the theorem closely resembles that of classical diffusion processes. 
Unlike in the model described in Section~\ref{sec ddm} and in Theorem~\ref{th coupling ddm poisson}, where we established a coupling between the DDM and the Poisson counter model, the drift $\mu^j_m \coloneqq  n\Bar{\lambda}^j_m$ of process $W^{j,n}_m$ does not depend only on the nature $o$ of the presented object, but also on the presentation number $m$ through its dependency in the weights $w^j_m$. The covariance matrix of the process $W^n_m$ is also not diagonal: the processes of $(W^{j,n}_m)_{j\in J}$ are correlated, reflecting the natural dependencies within neural networks. They can be expressed as solutions of the $|J|$-dimensional stochastic differential equation, which is defined for each process $j$ as follows:
\[
dW^{j,n}_{m,t} =  n\Bar{\lambda}^j_m d\mathbb{G}(t) + \sqrt{n\Bar{\lambda}^j_m\norm{g}_1} dB^j_t + \sum_{i\in I} w^{i\to j}_m \sqrt{n\gamma^i_o} dB^i_t .
\]
Each $B^j_t$ represents an independent noise source, while the collection $(B^i_t)_{i \in I}$ introduces correlated noise terms across all processes. This correlation arises from the shared influence of the same input neurons on each output neuron in our model. The obtained bound on the trajectories \eqref{eq sup N - W th} is in $O(n^{1/4} \log(n)^{3/4})$, which is less tight than the $O(\log(n))$ bound established in Theorem~\ref{th coupling ddm poisson} for the trajectories of the Poisson counter model and the drift diffusion model. This is due to the greater complexity of the Hawkes counter model.

Besides, similarly as in Theorem~\ref{th coupling ddm poisson} for Poisson and Brownian processes, we establish for every $j$ the convergence of the hitting times of both Hawkes process $(N^{j,n}_{m,t})_{t\geq 0}$ and Brownian process $(W^{j,n}_{m,t})_{t\geq 0}$ to $\frac{\thr}{\Bar{\lambda}^j_m\norm{g}_1}$, which is the hitting time of the deterministic processes $(\E[N^{j,n}_{m,t}])_{t\geq 0}$ and $(\E[\W^{j,n}_{m,t}])_{t\geq 0}$. However, unlike in Theorem~\ref{th coupling ddm poisson} where the limit hitting time depends on prior knowledge $\gamma^j_o$, here the limit hitting time depends on $w^j_m$, \ie on the learning stage of the model. 
Furthermore, we prove that the difference between these two hitting times decreases faster than $n^{-1/2}$, which is an upper bound on  their respective  convergence rate to their common limit. The resulting bound \eqref{eq taun - tauw} is less tight than the similar bound of Theorem~\ref{th coupling ddm poisson} due to the fact that the bound on the trajectories \eqref{eq sup N - W th} is also less tight in this case.

Therefore, this theorem ensures that both models select the same category $\hat{j}$, and we can apply the same bound to the difference between the reaction times of the two models: we have
\[
\AAbs{\tau^N - \tau^W}= \AAbs{ (T_{\text{min}}\vee\tau^{N^{\hatj}_n}) \wedge T- (T_{\text{min}}\vee\tau^{W^{\hatj}}_n)\wedge T} = O\left(\left(\frac{\log(n)}{n}\right)^{3/4}\right)
\]
establishing that both models provide similar reaction times and categorizations.

Consequently, this theorem confirms that the dynamics of evidence accumulation for category 
decisions are similar in both models. Extending beyond the results established in 
Section~\ref{sec coupling ddm poisson}, this finding demonstrates that DDMs can emerge from complex models of interacting spiking neurons with a learning mechanism. This provides new insights into the biological interpretation of DDMs and supports the idea that the prior knowledge assumed in these models can arise through a learning process. %

\section{The Experiment} \label{sec expe}

To validate our model, we designed a categorization task in the form of an online video game, accessible at \url{https://3ia-demos.inria.fr/mel/en/}. The goal of the game was for participants to learn which rockets displayed on the screen were capable of reaching the moon and which were not. We opted for a video game format rather than a classical categorization task to make the activity more engaging, particularly for younger participants. The game was coded to allow flexibility in the difficulty of the categorization task. For this experiment, participants were presented with the easiest level (details provided in the Stimuli and Categories Section~\ref{sec stim and cat}).

\subsection{Participants}

A total of $99$ participants, including $45$ middle school students and $54$ undergraduate students, were recruited for the study. Middle school students ($23$ male, $22$ female) ranged in age from $11$ to $15$ years (M = $12.13$, SD = $1.15$), while undergraduate students were all aged between $18$ and $25$. All participants reported normal or corrected-to-normal vision.
The study was approved by the ethics committee of Université Côte d’Azur and was conducted in accordance with relevant guidelines and regulations (ethics approval number $2022-098$). All participants provided informed consent prior to participation. Participants were recruited via one of their teachers. Middle school students performed the task at school whereas undergraduate students performed it at home.

\subsection{Stimuli and Categories} \label{sec stim and cat}

The set of stimuli consisted of 16 unique rockets, each defined by four characteristics, with two possible features for each: the shape of the head (sharp or round), the shape of the body (straight or round), the shape of the fins (straight or curved), and the number of flames (one or three). These rockets are illustrated in Figure~\ref{fig 16 rockets}.
As mentioned above, participants were administered the easiest level of the categorization task. At this difficulty level, the categories were formed based on a single characteristic selected randomly at the beginning of the task. Rockets were divided into two groups: those with one feature of the chosen characteristic and those with the alternative feature. For example, if the selected characteristic was the shape of the head, rockets with a sharp head were categorized as capable of landing on the moon, while rockets with a round head were categorized as unable to land on the moon. Alternatively, rockets with a round head could also have been assigned to the ``capable of landing on the moon" category, as the assignment was made randomly.
From these two categories, five rockets from each group were randomly selected, resulting in 10 unique rockets presented during the learning phase. The remaining six rockets were used in the transfer phase.
\begin{figure}
    \hspace{1cm}
    \includegraphics[width=0.8\linewidth]{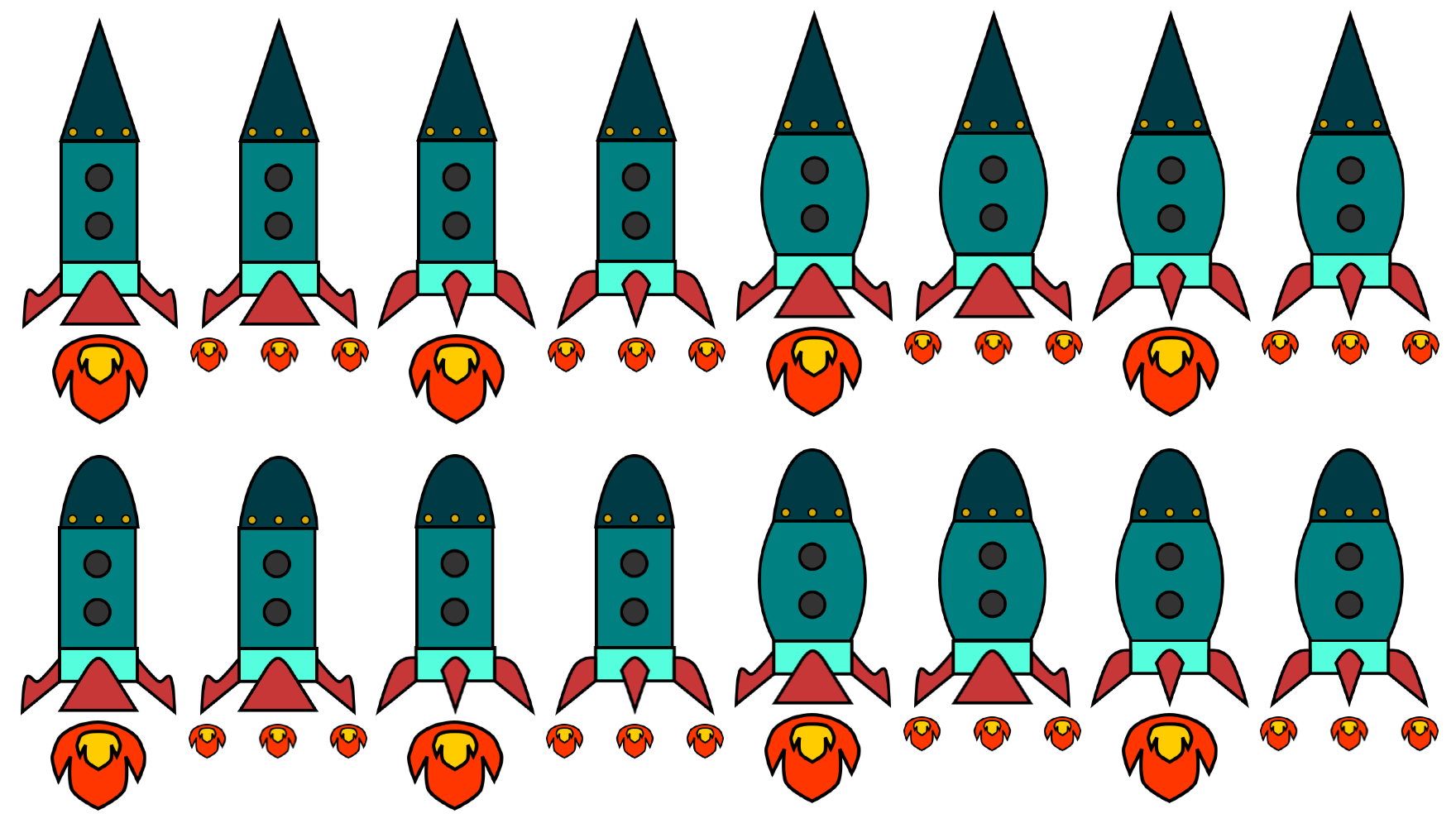}
    \caption{The $16$ objects.}
    \label{fig 16 rockets}
\end{figure}

\subsection{Procedure}

Participants performed the task on a computer. Middle school students were supervised by their teacher, and undergraduate students performed the task at home. Prior to the experiment, participants completed a series of demographic questions and received detailed instructions on how to perform the task. The experiment consisted of two phases: a learning phase and a transfer phase.
In the learning phase, participants were asked to determine which rockets were capable of landing on the moon and received immediate feedback on the correctness of their responses. Stimuli were presented one at a time at the center of the screen for 5 seconds. To help participants track time, a clock icon was displayed. After viewing each stimulus, participants made their decision by selecting one of two buttons labeled `Yes' or `No' to indicate whether the rocket could land on the moon. Upon selecting a button, feedback on the correctness of the response was presented for $2$ seconds, including a visual cue in the form of a green checkmark for correct responses or a red cross for incorrect ones. Additionally, an animation briefly depicted whether the rocket was able to land on the moon. If participants did not provide a response within 5 seconds, a `time is up' icon appeared on the screen.
To reinforce learning, a progression bar was displayed at the top left of the screen, with a green checkmark appearing each time a correct response was given. Participants were required to correctly classify 15 consecutive rockets to pass the learning phase. The 10 unique rockets were presented in a random order and repeated until the participant reached the learning criterion, with a new random order at each repetition.
After the learning phase, participants proceeded to the transfer phase, where they were instructed to classify new rockets that had not been presented in the learning phase. Stimuli were presented in the same manner as in the learning phase. In this phase, six new rockets were presented three times in a random order, resulting in a total of 18 trials. No feedback was provided during the transfer phase.
Participants could choose to withdraw from the experiment at any time by selecting a cross icon located at the top left of the screen.

\subsection{Modeling} \label{sec expe modeling}

We modeled this experiment with the model as described in Figure~\ref{fig réseau fusées}. Object natures are separated between a set of learning natures $\Obj_l$ containing the $10$ learning rockets and a set of transfer natures $\Obj_{tr}$ containing the $6$ transfer rockets. The set of feature $I$ contains the $8$ possible features of the rockets (two shapes for the head, two shapes for the body, two shapes for the fins and two possible type of flames). The two categories denoted $j_1$ and $j_2$ are characterized by a certain characteristic declined in two features $i_1$ and $i_2\in I$: every object nature $o$ has either feature $i_1$ or feature $i_2$, and category $j_1$ contains every nature $o$ with $i_1$ whereas category $j_2$ contains every nature $o$ with $i_2$. 

When an input neuron $i$ is presented with a rocket having nature $o\in \Obj_l\cup\Obj_{tr}$, it starts spiking with intensity $\gamma^i_o \coloneqq \mathbb{1}_{\{\text{object $o$ has feature $i$}\}}\gamma$ with $\gamma>0$. In other words, input neuron $i$ spikes with firing rate $\gamma$ if the object has feature $i$ and stays silent otherwise.

\begin{proposition} \label{prop cas particulier}
    For the set of object natures $\Obj_l$ and input neuron intensities $\gamma^I$ defined above, the sets defined in Proposition~\ref{prop conv weights} are $I^{j_1} = \{i_1\}$ and $I^{j_2} = \{i_2\}$, and the limit family $w_{\infty}$ verifies Assumption~\ref{assump limit fam}: it enables to categorize well the objects and we can choose any constant in $(0,\gamma)$ for $\Dlamb$.
\end{proposition}

\subsection{Analysis}

To analyze the learning behavior of participants, we infer two parameters of our model for each participant: the learning rate $\eta$ used to update the synaptic weights as described in \eqref{eq update w} and indicates the participant's learning speed, and the threshold $\thr$, which reflects their response speed. For this estimation, we employ the Simulation-Based Inference method \citep{papamakarios2016fast} and we use data from both the learning phase and the transfer phase. This approach is tailored for models with intractable likelihoods that can nonetheless be simulated. The process involves providing the model as a probabilistic program that is easy to simulate, along with a prior distribution over the parameters to be estimated. A Bayesian neural network is then trained via maximum likelihood on samples generated by the program, enabling the learning of a posterior distribution from which samples can be drawn. To simulate data for the transfer phase, we freeze the weights at their final update, $w^j_{M+1}$, as participants have already completed learning to categorize the objects. Because the lengths of participant data and simulated data did not always match, we extended the transfer phase data by appending the mean answer time of the transfer phase. This ensured that the data had the required length for applying the SBI method. We denote by $\hat{\eta}$ and $\hat{\thr}$ the median value of samples of the posterior given by the SBI method for each participant. The method is illustrated using both simulated and real data in Figures~\ref{fig hist sim data} and~\ref{fig hist real data}, respectively. On simulated data, the method accurately recovers the true parameters.

\begin{figure}
    \centering
    \includegraphics[width=1\linewidth]{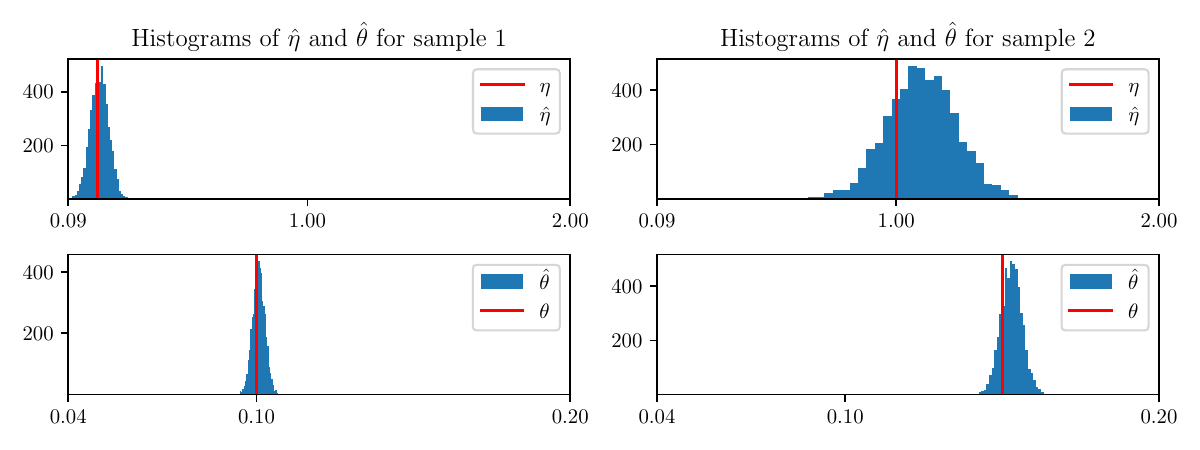}
    \caption{Samples of the posteriors given by the SBI method for two sets of simulated data. The prior used for the parameters tuple $(\eta,\thr)$ is a uniform distribution on $[0.09,2]\times[0.04,0.2]$. In red, the true values of $\eta$ and $\thr$ used to simulate data. In blue, the histogram of the values $\hat{\eta}$ and $\hat{\thr}$ given by the posteriors given by the SBI method.}
    \label{fig hist sim data}
\end{figure}

\begin{figure}
    \centering
    \includegraphics[width=1\linewidth]{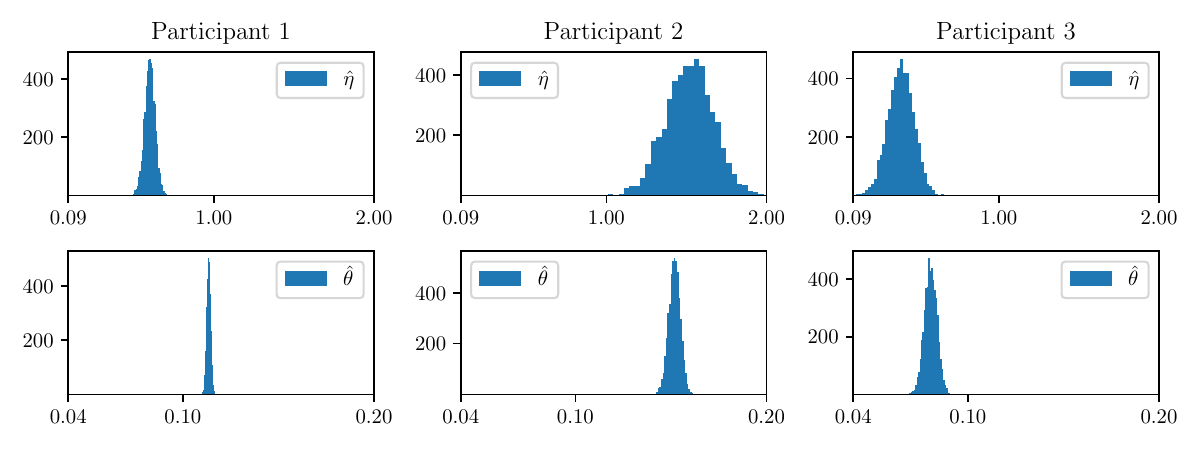}
    \caption{Samples of the posteriors given by the SBI method for three sets of real data. The same priors as the ones of Figure~\ref{fig hist sim data} are used for the SBI method.}
    \label{fig hist real data}
\end{figure}

\subsection{Single participant analysis}

In Figure~\ref{fig cum times}, the cumulative answer times as a function of the number of answers are compared to simulated cumulative answer times for six participants, using parameters estimated by the SBI method. Here, we used data from both the learning and transfer phase. The results indicate that the method predicts answer times accurately.

To see more precisely if one data simulation given by our network is close to real data, we represented the answer times of three participants (on top) and of three sets of simulated data (on the bottom) during the learning phase, using parameters inferred by the SBI method on Figure~\ref{fig eventplot}. Each blue tick represents a correct answer, and each red tick represents a wrong answer. The model again predicts answer times similar to those of the participants, albeit with greater regularity. However, it sometimes seems to make fewer errors than the participants, as in the case of participant $5$. This can be attributed to the model’s reliance on synaptic weight updates: once sufficient weight has been assigned to relevant feature neurons, the probability to make a mistake diminishes significantly. In contrast, humans may continue to make errors even after grasping the underlying rule thanks to other cognitive mechanisms such as distraction. This difference also explains why the model sometimes underestimate the total number of answers during the learning phase. Each error resets the count of consecutive correct answers to zero, requiring the participant to categorize at least 15 new objects to complete the learning phase. Consequently, each mistake significantly elongates the learning phase.

\begin{figure}
    \centering
    \includegraphics[width=1\linewidth]{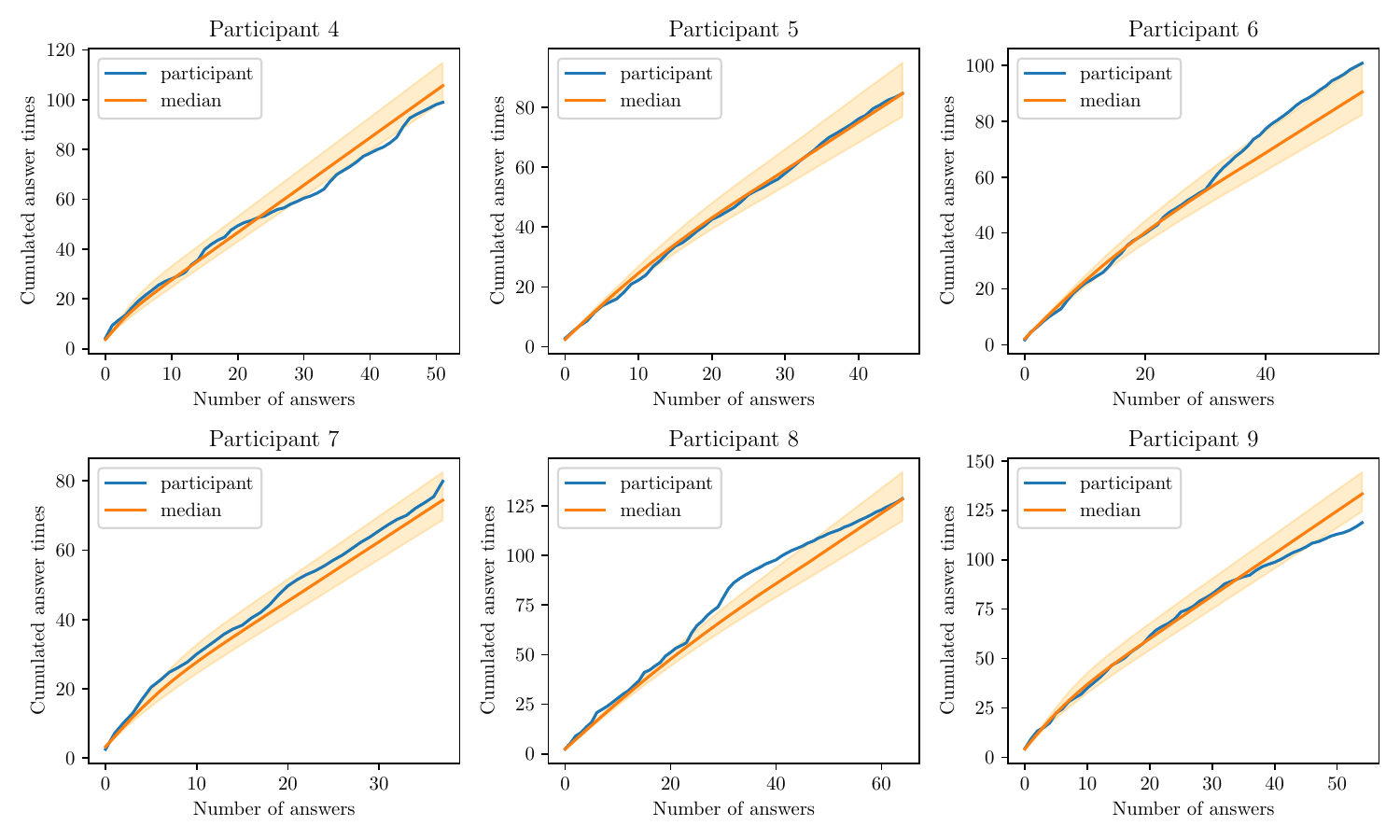}
    \caption{For each participant, the blue curve represents the cumulated answer times in function of the number of presented objects. The orange curves represent the median of the cumulated answer times of $100$ simulations of the model with confidence interval of level $10\%$, with parameters $\hat{\eta}$ and $\hat{\thr}$ being the median value of samples of the posterior given by the SBI method. The same priors as the ones of Figure~\ref{fig hist sim data} are used for the SBI method.}
    \label{fig cum times}
\end{figure}

\begin{figure}
    \centering
    \includegraphics[width=1\linewidth]{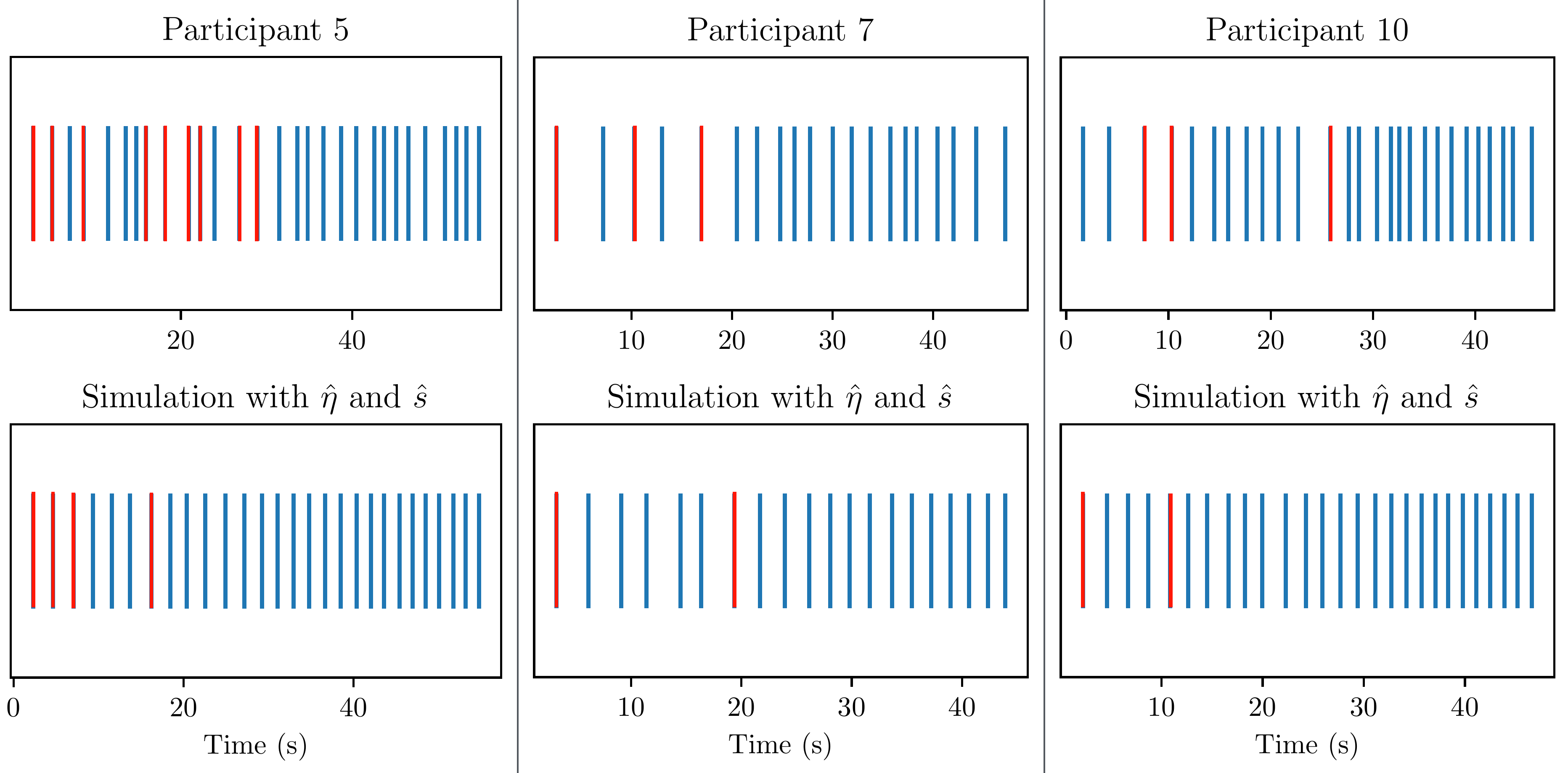}
    \caption{On the top, answer times of three participants during their learning phase. Blue ticks represent correct answers while red ticks represent wrong answers. On the bottom, same plots for simulated data with parameters $\hat{\eta}$ and $\hat{\thr}$ being the median value of samples of the posterior given by the SBI method for each participant. The same priors as the ones of Figure~\ref{fig hist sim data} are used for the SBI method.}
    \label{fig eventplot}
\end{figure}

\subsection{Group analysis} \label{sec group analysis}

Figure~\ref{fig hist groups} presents histograms of the parameters inferred by the SBI method, grouped by participants' current educational level (middle school or university) and for all participants combined. The learning rate $\eta$ appears to take either low or high values, with middle school students exhibiting more extreme values compared to university students. In contrast, the histograms for the parameter $\thr$ show little variation across the different groups.

\begin{figure}
    \centering
    \includegraphics[width=1\linewidth]{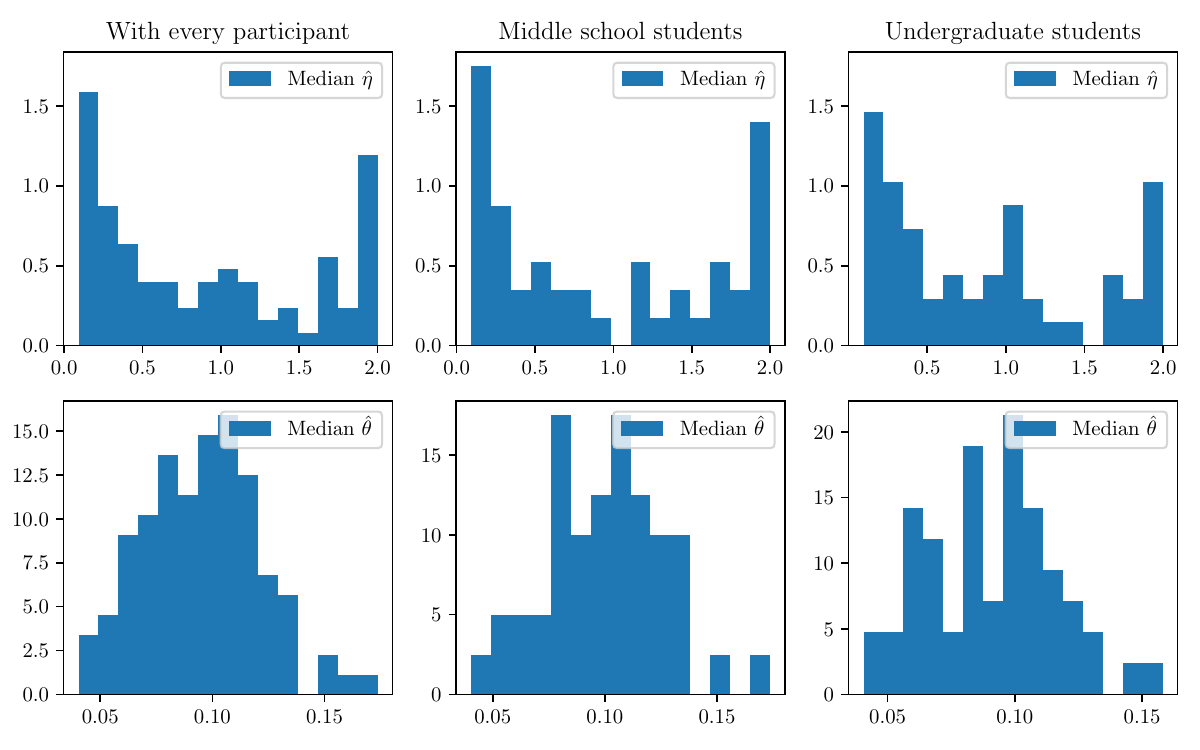}
    \caption{Histogram of median values of the parameters $\hat{\eta}$ and $\hat{\theta}$ given by the SBI method with on the left, every participant, on the middle, middle school students and on the right, undergraduate students. The parameters used for the SBI method are the same as in Figure~\ref{fig hist sim data}.}
    \label{fig hist groups}
\end{figure}

 To conduct a more detailed analysis, we applied $k$-means clustering to the two-dimensional data consisted on the parameters $\hat{\eta}$ and $\hat{\thr}$ estimated by the SBI method for each participant. The optimal number of clusters was determined using the silhouette method (details are provided in Appendix~\ref{sec append cluster}), which suggested two reasonable options: $2$ or $3$ clusters. Figure~\ref{fig cluster eta s} illustrates the results of the $k$-means algorithm applied to the estimated parameters $\hat{\eta}$ and $\hat{s}$ for both $2$ and $3$ clusters. 
\newline

\textbf{Clustering with 2 clusters.} Cluster $0$ is characterized by a low learning rate and threshold. It corresponds to participants who learn at an average or slower pace and require less evidence to make a categorization decision. This cluster includes $51\%$ of middle school students and $63\%$ of undergraduate students. Cluster $1$, on the other hand, has a high learning rate and a slightly above-average threshold. It represents fast learners who require a higher accumulation of evidence before categorizing and comprises $49\%$ of middle school students and $37\%$ of undergraduate students. 

A chi-square test was conducted to assess whether cluster membership is independent of the participants’ current educational level. The test yielded a p-value of $0.33$, which is insufficient to reject the null hypothesis, indicating no significant evidence of dependence.
\newline

\textbf{Clustering with $3$ clusters.} Clusters $0$ and $1$ retain approximately the same centers as in the $2$-cluster analysis. Cluster $0$ includes $18\%$ of middle school students and $32\%$ of undergraduate students, while cluster $1$ comprises $47\%$ of middle school students and $27\%$ of undergraduate students. Cluster $2$, primarily derived from points in the former cluster $0$ along with a few from cluster $1$, has a center characterized by a low learning rate and an average threshold. This cluster represents participants who learn at an average or slower pace and require a moderate accumulation of evidence before categorizing, encompassing $35\%$ of middle school students and $41\%$ of undergraduate students.

A chi-square test was conducted to evaluate whether cluster membership is independent of current educational level. The test resulted in a p-value of $0.11$, which, while not statistically significant, provides some indication that cluster membership might depend on the current level of study. Specifically, slow-learning participants requiring less evidence to categorize appear more frequently among undergraduate students, while fast-learning participants requiring higher evidence accumulation tend to be more prevalent among middle school students.
\begin{figure}
    \centering
    \includegraphics[width=1\linewidth]{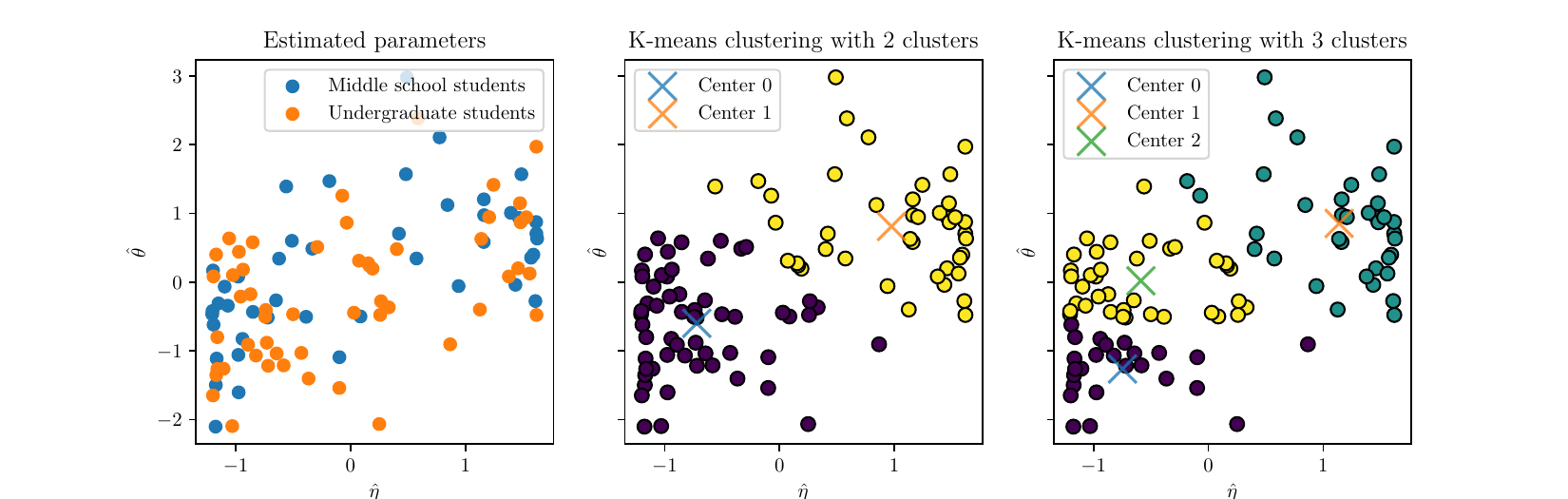}
      \begin{tabular}{| l | c | c | } \hline
    Two clusters & Middle school students & Undergraduate students  \\ \hline
   Percentage in cluster $0$ & $51 \%$ & $63\%$ \\
   Percentage in cluster $1$ & $49\%$ & $37\%$ \\ \hline
 \end{tabular}

\vspace{2 mm}

    \begin{tabular}{| l | c | c | } \hline
    Three clusters & Middle school students & Undergraduate students  \\ \hline
   Percentage in cluster $0$ & $18 \%$ & $32\%$ \\
   Percentage in cluster $1$ & $47\%$ & $27\%$ \\
   Percentage in cluster $2$ & $35\%$ & $41\%$ \\ \hline
 \end{tabular}
    \caption{On the left, $\hat{s}$ as a function $\hat{\eta}$ for every participant, where $\hat{\thr}$ and $\hat{\eta}$ are the median values given by the SBI method. On the right, $3$-means clustering of the same data. On the bottom, students distribution by cluster. The parameters used for the SBI method are the same as in Figure~\ref{fig hist sim data}.}
    \label{fig cluster eta s}
\end{figure}
\newline

Finally, we examined the relationship between the threshold estimated by the SBI method and each participant's average answer time during the learning phase, as shown in Figure~\ref{fig time s}. The results indicate that the average answer time appears to increase linearly with $\hat{\thr}$, which is consistent with expectations.
\begin{figure}
    \centering
    \includegraphics[width=0.5\linewidth]{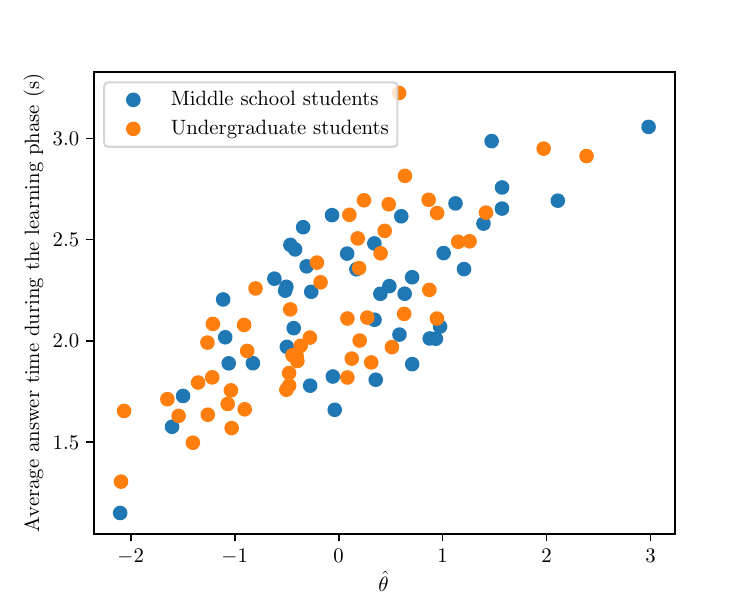}
    \caption{Average answer time of each participant during the learning phase (including the $15$ objects to categorize perfectly) as a function of $\hat{\thr}$. The parameters used for the SBI method are the same as in Figure~\ref{fig hist sim data}.}
    \label{fig time s}
\end{figure}

\section{Conclusion}

This paper outlines key elements that establish a connection between drift diffusion models (DDMs), commonly used in cognitive science to simulate decision-making tasks, and biologically plausible networks of spiking neurons. First, we conducted a mathematical analysis of both the drift diffusion model and the Poisson counter model, specifying explicit conditions under which these models yield accurate decisions. This analysis demonstrated the similarity in their behavior. Subsequently, we established a coupling result between the processes of the two models and derived a bound on their hitting times, confirming that both models produce comparable reaction times and decisions. These findings represent an initial step towards demonstrating that DDMs can be derived from spiking neuron-based models.

To go further and in order to establish that DDMs can be derived from biologically realistic networks of spiking neurons which can learn to make correct decisions, we proposed a novel spiking neural network model tailored for decision-making task which is provably close to the drift diffusion model. Unlike the drift diffusion and Poisson counter models, the model does not assume any prior knowledge of categories; its accuracy emerges from a learning phase. Additionally, the model is more biologically relevant than the Poisson counter model, as it incorporates interacting spiking neurons that learn to represent complex concepts, rather than individual neurons directly encoding decisions. We conducted an asymptotic analysis of the network's behavior and identified conditions under which it achieves accurate decisions. Moreover, we established a coupling result between our model and the DDM, providing strong evidence for the biological plausibility of DDMs and validating their widespread use in cognitive science. In the literature, coupling between Poisson and Brownian processes are well-known results. However, coupling between Hawkes and Brownian processes are less known. Recent results have been obtained in one dimension \citep{besanccon2024diffusive} or for mean field limit of systems of interacting neurons \citep{erny2023strong}. However, to the best of our knowledge, this is the first time such a result has been established for a network with a `perceptron'-like structure. 

Additionally, we developed a cognitive experiment to evaluate our model's predictions, enabling us to derive meaningful insights into the participant's cognitive behavior and demonstrating the practical applicability of our model.

This study represents an important advancement in incorporating biologically relevant neural mechanisms into cognitive models, allowing us to better understand how neural dynamics lead to certain behaviors. There are several promising avenues for further improving the model. We showed that the model tends to produce more regular reaction times and make fewer errors compared to humans, highlighting the potential value of exploring ways to make it more realistic. Additionally, applying our model to more complex experiments could provide valuable insights into intricate behaviors. This seems feasible, as we showed in \cite{jaffard2024chani} that introducing hidden layers enables the modeling of more complex concepts. Finally, an ambitious direction for future research would be to extend this work to other spiking neural networks, thereby demonstrating that DDMs can be approximated by alternative network types commonly used in neuroscience.

\backmatter

\subsection*{Supplementary information}
The code and data used to produce the numerical results can be found at the following link \href{https://github.com/SophieJaffard/MEL}{https://github.com/SophieJaffard/MEL}.

\subsection*{Acknowledgements}
We thank the 3iA institute, and particularly Stéphane Petiot for implementing our experiment. We thank Valérie Lemesle, 
Sébastien Kerner, Noé Szymczak and Maëva Lossouarn for presenting the experiment to their students, and we also thank all 
the students who participated. This research was supported by the French government, through CNRS (eXplAIn team), the UCAJedi 
and 3iA Côte d’Azur Investissements d’Avenir managed by the National Research Agency (ANR-15 IDEX-01 and ANR-19-P3IA-0002), directly by the ANR project ChaMaNe (ANR-19-CE40-0024-02), and finally by the interdisciplinary Institute for Modeling in Neuroscience and Cognition (NeuroMod).

\subsection*{Statements and declarations}
The authors have no competing interests to declare that are relevant to the content of this article.

\begin{appendices} %

\section{Table of notations} \label{sec notations}

For parameters $p_1,\dots, p_k$, we denote by $\Box_{p_1,\dots,p_k}$ any positive constant depending only on parameters $p_1,\dots, p_k$. The value of $\Box_{p_1,\dots,p_k}$ can change from line to line and it is not necessarily equal on both sides of an inequality.

For $a,b\in \R$, we denote $a\vee b \coloneqq \max(a,b)$ and $a\wedge b \coloneqq \min(a,b)$. Also when quantities have several indices, the dropping of one index corresponds to the sequence on this index. For instance, $\gamma = (\gamma^j_o)_{j\in J, o\in \Obj}$ whereas $\gamma_o = (\gamma^j_o)_{j\in J}$ and $\gamma^j = (\gamma^j_o)_{o\in \Obj}$. If two sets of indices are available, typically $I$ and $J$, to avoid confusion, we use the set instead of the index. For instance, $\gamma^I=(\gamma^i_o)_{i\in I, o\in \Obj}$. Also $ \| \|_{\infty}$ is the infinite norm.

For a set $E$ and a quantity $x_e \in \R$ indexed by $e\in E$, we denote by $\brac{x_e}_{e\in E} \coloneqq \frac{1}{\abs{E}}\sum_{e\in E} x_e$ its empirical mean.

\begin{longtable}{|l|l|} \caption{Table of notations} \label{tab_notations} \\
  \hline
   Notation & Description (page of introduction) \\
   \hline
      $\Obj$ & set of object natures (page~\pageref{def:obj})\\
      \hline
      $o$ & nature of an object (page~\pageref{def:petito})\\
     \hline
      $J$ & set of categories and output neurons (page~\pageref{def:categoryj})\\
    \hline
    $j$ & index of an output neuron or a category (page~\pageref{def:categoryj})\\
    \hline
      $T$ & maximal duration of an object presentation (page~\pageref{def:limitT})\\
    \hline
     $m$ & index of an object (page~\pageref{def:m})\\
     \hline
     $\thr$ & threshold (page~\pageref{def:theta})\\
     \hline
     $W_m$ & $\abs{J}$-dimensional drifted Brownian motion $W_m = (\W^j_{m})_{j\in J}$ 
     (page~\pageref{def:wm})\\
     \hline
     $\W^j_{m,t}$ & evidence accumulation in favor of category $j$ at time $t$ of object $m$ \\ & for the DDM (page~\pageref{eq def w})\\
     \hline
     $\mu^j_o$ & drift of the process $\W^j_{m}$ when the nature of object $m$ is $o$ 
     (page~\pageref{eq def w})\\
     \hline
     $\tau^{Z^j}$ & hitting time of process $Z^j_{m}$ for $Z\in \{W,\Pi,N\}$ 
     (pages~\pageref{tau}, \pageref{tauPois} and \pageref{tauHaw})\\
     \hline
     $\hatj$ & index of the first process to reach threshold $\thr$ %
     (pages~\pageref{choice}, \pageref{choicePois} and \pageref{choiceHaw})
     \\
     \hline
     $\tau^Z$ & reaction time of the model for $Z\in \{W,\Pi,N\}$ 
     (pages~\pageref{def:reactimeW}, \pageref{def:reactimepi} and \pageref{def:reactimeHawkes}) \\
     \hline
      $\Pi_m$ & $\abs{J}$-dimensional Poisson process $\Pi_m = (\Pi^j_{m})_{j\in J}$ 
      (page~\pageref{def:pim})
      \\
     \hline
     $\Pi^j_{m,t}$ & 
     spike count of neuron $j$ on $[0,t]$
     for object $m$ in the Poisson  \\ 
      & counter model, \ie  (page~\pageref{def:pim}) \\
     \hline
     $\gamma^j_o$ & intensity of output neuron $j$ when the presented object \\
     & has nature $o$ for the Poisson counter model (page~\pageref{def:pim}) \\
     \hline 
    $n$ & number of neurons coding for each category in the framework of  \\ & the coupling results (page~\pageref{def:notan})\\
    \hline
     $M$ & number of objects presented to the network during \\
     & the learning phase  (page~\pageref{def:notaM})\\
    \hline
    $I$ & set of features and input neurons (page~\pageref{def:featureI})\\
    \hline
    $i$ & index of an input neuron or a feature (page~\pageref{def:featureI})\\
    \hline
    $\gamma^i_o$ & intensity of input neuron $i$ when presented with an object having  \\
    & nature $o$ for our model (page~\pageref{def:gammai0})\\
    \hline
\(I^j\) & set of input which are the most sensible to category $j$ (page~\pageref{def:set of input most sensible to j})\\
    \hline
    \(\widehat{\gamma}^i_o\) & empirical firing rate of input neuron $i$ on $[0,T_{\text{min}}]$(page~\pageref{eq gamma hat})\\
    \hline
    $T_{\text{min}}$ & minimum duration of the presentation of an object  (page~\pageref{def:tmin})\\
    \hline
     $N_m$ & $\abs{J}$-dimensional Hawkes process $N_m = (N^j_{m})_{j\in J}$ (page~\pageref{def:nj})\\
     \hline
    $N^j_{m,t}$ &  evidence accumulation in favor of category $j$ at time $t$ of object $m$ 
    \\
    &  \ie spike count of neuron $j$ until time $t$ (page~\pageref{def:Njmt})\\
    \hline
    $\lambda^j_{m,t}$ & conditional intensity of output neuron $j$ at time $t$ of object $m$ for  \\ 
    & our model (page~\pageref{def:g})\\
    \hline
    $w^{i\to j}_m$ & synaptic weight between input neuron $i$ \\& and output neuron $j$ (page~\pageref{eq update w})\\
    \hline
    $g$ & self-exciting function of the Hawkes process (page~\pageref{def:g})\\
    \hline
    \(h_t(u)\) & \(h_t(u) = \int_u^t g(s-u) ds\) (page~\pageref{def:htu})\\
    \hline
    \(\mathbb{G}(t)\) & \(\mathbb{G}(t) = \int_0^t \int_0^s g(s-u)du ds\) (page~\pageref{def:mathbbG})\\
    \hline
    $\Lambda^j_{m,t}$ & compensator of the process $(N^j_{m,s})_{s\geq 0}$ at time $t$
    (page~\pageref{def:lambdajm})\\
    \hline
    $g^{i\to j}_m$& gain of input neuron $i$ w.r.t. output neuron $j$ after the  \\ & 
    presentation of the $m^{th}$ object (page~\pageref{def:gijm})\\
    \hline
    $G^{i\to j}_m$& cumulated gain of input neuron $i$ w.r.t. output neuron $j$ after the \\ & presentation of the $m^{th}$ object $\sum_{m'=1}^m g^{i\to j}_{m'}$ (page~\pageref{def:Gitojm})\\
    \hline
    $\eta=\eta_0/\sqrt{M}$ & learning rate of the learning rule given by EWA (\eqref{eq:defeta} page~\pageref{eq:defeta})\\
    \hline
    $\mathcal{F}_m$ & $\sigma-$algebra generated by every event which happened until the  \\ & end of the presentation of the $m^{th}$ object 
    (page~\pageref{def:mathcalFm})\\
    \hline
    $d^{i\to j}$ & feature discrepancy of input neuron $i$ w.r.t. output neuron $j$\\
&    (page~\pageref{def:ditoj})\\
    \hline
\(I^j\) & \(I^j \coloneqq \arg\max_{i\in I} d^{i\to j}\) (page~\pageref{def:ditoj})
\\
\hline
    \(\delta^j\) & \(\max_{i\in I} d^{i\to j} - \max_{i\in I\setminus I^j} d^{i\to j}\)(page~\pageref{def:ditoj})\\
    \hline
    $w^{i\to j}_\infty$ & limit synaptic weight between input neuron $i$ and output\\
    & neuron $j$ 
    (page~\pageref{def:witojinfty})\\
    \hline
     $\Dmu$ & discrepancy between the drift of the true category process and the \\
      & other processes for the DDM (page~\pageref{assump drift})\\
    \hline
    $\Dgam$ & discrepancy between the firing rate of the true category process  \\
    & and the other processes for the Poisson counter model (page~\pageref{assump intensity poisson}) \\
    \hline
    $\Dlamb$ & discrepancy between the firing rate of the true category process  \\ 
    & and the other processes for our model (page~\pageref{assump limit fam})\\
    \hline
   \end{longtable}

\section{Details on the clustering analysis} \label{sec append cluster}

The number of clusters used to perform the $k-$means algorithm in Section~\ref{sec group analysis} was determined using the silhouette method implemented in scikit-learn. The silhouette plot, available in Figure~\ref{fig silhouette}, illustrates how well each point aligns with its assigned cluster compared to neighboring clusters. Silhouette coefficients close to 1 indicate that a sample is well-separated from neighboring clusters, while a coefficient of 0 suggests that the sample lies near the boundary between two clusters. The dotted red line represents the average silhouette score across all samples for a given number of clusters. In order to choose a meaningful number of clusters, two criteria should be respected: each clusters should have points above the average silhouette score, and the several clusters should have comparable sizes. Therefore, $k=2$ and $k=3$ are sensible choices.

\begin{figure}[h!]
    \centering
    \includegraphics[width=1\linewidth]{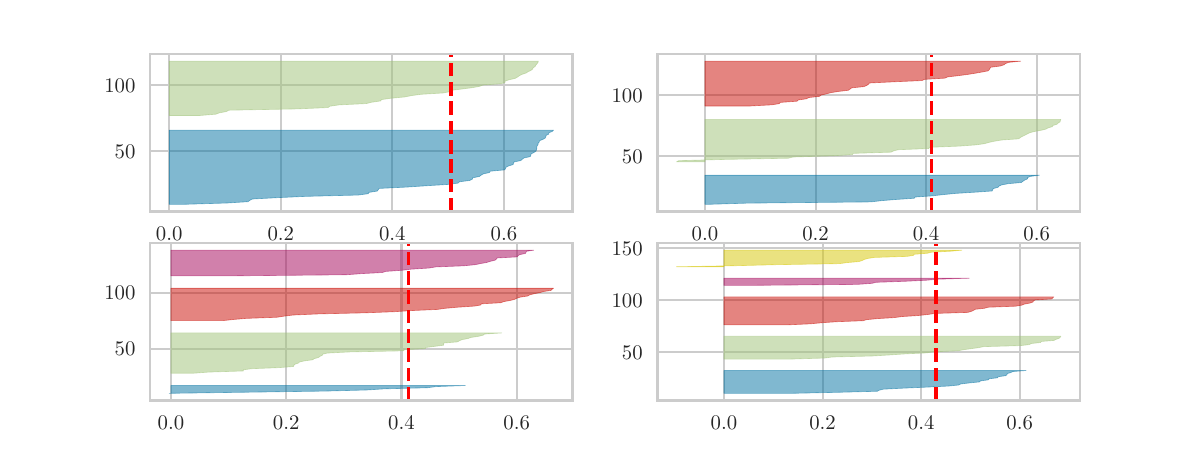}
    \caption{Silhouette analysis of the data with $2$, $3$, $4$ and $5$ clusters.}
    \label{fig silhouette}
\end{figure}

\section{Proofs} \label{sec proofs}

\subsection{Preliminary lemmas}

\begin{lemma} \label{lemma control z - mu}
    Let $(\W_t)_{t\geq 0}$ be a Brownian motion with drift $\mu$ and scaling $\sigma$. Let $T>0, \beta >0$. Then
    \[
    \mathbb{P}\left(
    \sup_{0\leq t \leq T} \abs{\W_t - \mu t} < \sqrt{2T\sigma^2\log\left(\frac{2}{\beta}\right)}
    \right)
     \geq 1 - \beta.
    \]
\end{lemma}

\begin{proof}
    The process $\frac{\W_t - \mu t}{\sigma}$ is a standard Brownian motion. Let $z>0$. According to Doob's inequality (see for instance Proposition $1.5$ p.53 of \cite{revuz2013continuous}), we have with probability more than $1-2e^{-\frac{z^2}{2T}}$ that
    \[
    \sup_{0\leq t \leq T} \frac{\abs{\W_t - \mu t}}{\sigma} < z.
    \]
    Let us choose $z$ such that $2e^{-\frac{z^2}{2T}} = \beta$, \ie $z = \sqrt{2T \log\left(\frac{2}{\beta} \right)}$. Then with probability more than $1-\beta$ we have
    \[
     \sup_{0\leq t \leq T} \abs{\W_t - \mu t} < \sqrt{2T\sigma^2\log\left(\frac{2}{\beta}\right)}.
    \]
\end{proof}

The following Lemma is an adaptation of Corollary $5.2$ p. 357 of \cite{ethier2009markov}.

\begin{lemma}
\label{original_coupling}
There exists absolute constants $a,b,d>0$ and a measurable function $G: S(\mathbb{R}_+)\times [0,1]\to S(\mathbb{R}_+)$, with $S(\mathbb{R}_+)$ the Skorohod space such that if 
 $(\Pi_t)_{t\geq 0}$ is an homogeneous counting Poisson process of rate 1 and $V$ is an independent uniform random variable on $[0,1]$, then 
 \begin{itemize}
\item  $(W_t)_{t\geq 0}= G((\Pi_t)_{t\geq 0},V)$ is a Brownian motion starting in 0, with drift 1 and variance 1 (i.e. $(B_t)_{t\geq 0}$ such that $B_t=W_t-t$ is a standard centered Brownian motion), and 
\item for all $x>0$
\begin{equation*}
\P\left(\sup_{t\geq 0} \frac{|\Pi_t-W_t|}{\log(t \vee 1+1)} \geq a+x\right)\leq b e^{-d x}.
\end{equation*}
\end{itemize}
\end{lemma}

\begin{proof}
We start with the Koml\'os-Major-Tusn\'ady inequality \citep{komlos1975approximation, komlos1976approximation} 
(see Theorem $5.1$ p.356 of \cite{ethier2009markov} for the same formulation as here) applied to the i.i.d. sequence $(\xi_k)_{k\in \N}$ of Poisson r.v. with parameter \(1\).
KMT says that there exists such a sequence as well as a Brownian motion $W_t$ such that $\E(W_1)=1$ and 
$\var(W_1)=1$ and positive constants \(a\), \(b\) and \(d\) such that for all $x>0$ and all $n \in \N, n\geq 1$,
\[
\P(\max_{1\leq k \leq n} |S_k-W_k| > a \log(n) +x) \leq b e^{-dx},
\]
with $S_k=\sum_{i=1}^k \xi_i$.

Remark that we can also write for all $t\geq 0$, $W_t=t+B_t$ where $(B_t)_{t\geq 0}$ is a standard Brownian (meaning a centered Brownian such that $\cov(B_s,B_t)=s\wedge t$).

Conditionally to all the $\xi_k$'s, one can construct a Poisson process $(\Pi_t)_{t\geq 0}$ such that for any $k$, $\Pi_k=S_k$. 
Indeed, for each $k$, conditionally to each $\xi_k$, and independently of anything else, we can pick $\xi_k$ i.i.d. uniform variables 
on $[0,1]$ and order them. They constitute the times of the jumps of a Poisson process $(\tilde{\Pi}^k_t)_{t\in[0,1]}$, such that $\tilde{\Pi}^k_1=\xi_k$. 

Then, we define   
$\Pi_t=S_{\lfloor t\rfloor}+\tilde{\Pi}^{\lfloor t\rfloor}_{t-\lfloor t\rfloor}$  for any $t\geq0$.
The process \((\Pi_t)_{t\geq 0}\) is therefore a Poisson process of intensity \(1\) such that for all $k\in \N$, $S_k=\Pi_k$.

Let us fix $T\geq 0$ and $n=\lfloor T\rfloor$. Then for all $0\leq t\leq T$
\begin{align*}
 |\Pi_t-W_t| & \leq |\Pi_{\lfloor t \rfloor}-W_{\lfloor t \rfloor}|+ |\Pi_t-\Pi_{\lfloor t \rfloor}| + |W_t-W_{\lfloor t \rfloor}|\\
 & \leq |\Pi_{\lfloor t \rfloor}-W_{\lfloor t \rfloor}| + \xi_{\lfloor t \rfloor+1}+ \zeta_{\lfloor t \rfloor},
 \end{align*}
 with $\zeta_k= \sup_{s\in [0,1)}|W_{k+s}-W_k|.$

So 
\[
\sup_{t\in [0,T)} |\Pi_t-W_t| \leq \max_{k=1,\cdots,n} |S_k-W_k| + \max_{k=1,\cdots,n+1} \xi_k +\max_{k=0,\cdots,n} \zeta_k.
\]
But the  $\xi_k$'s are iid variables with finite exponential moment of every order. 
So by union bound and Markov inequality, we have for every $y\geq 0$
\[
\P(\max_{k=1,\cdots,n+1} \xi_k \geq y ) = \P(\max_{k=1,\cdots,n+1} e^{d\xi_k} \geq e^{dy}) \leq (n+1) \E(e^{d\xi_1}) e^{-dy}.
\]
The same holds for the $\zeta_k$'s since they also have exponential moment of every order.

Therefore taking $y= \log(n+1)/d +x/3, $ and using KMT with $x/3$, we obtain that with probability larger than $1- (b+\E(e^{d\xi_1})+\E(e^{d\zeta_1})) e^{-dx/3}$,
\[
\sup_{t\in [0,T)} |\Pi_t-W_t| \leq (a+2/d) \log(n+1) + x.
\]
In other words, there exist constants $a',b', d'$ such that for all $T>0$
\[
\P( \sup_{t\in [0,T)} |\Pi_t-W_t| \geq a' \log(T+1) +x) \leq b' e^{-d'x}.
\]
Now we want to get rid of $T$. We can decompose the event as follows
\begin{align*}
    &\left\{ \sup_{t\geq 0} \frac{|\Pi_t-W_t|}{\log(t \vee 1+1)} \geq y\right\} \\
    & =  \left\{\sup_{t\in [0,2)} \frac{|\Pi_t-W_t|}{\log(t \vee 1+1)} \geq y\right\}\cup\bigcup_{n\geq 1} \left\{ \sup_{t\in [2^n, 2^{n+1})} \frac{|\Pi_t-W_t|}{\log(t \vee 1+1)} \geq y\right\}.
\end{align*}
But the probability of the first event is bounded, thanks to the previous computations, by
\[
\P(\sup_{t\in [0,2)} |\Pi_t-W_t| \geq y \log(2)) \leq b' e^{-d'(y \log(2)- a'\log(3))},
\]
For the other events
\begin{align*}
   & \P(\sup_{t\in [2^n,2^{n+1})} |\Pi_t-W_t| \geq y \log(2^n)) \\
    &\leq \P(\sup_{t\in [0,2^{n+1})} |\Pi_t-W_t| \geq y \log(2^n)) \leq b' e^{-d'(y\log(2^n)-a'\log(2^{n+1}+1))}.
\end{align*}
But $\log(2^{n+1}+1)=\log(2^n)+\log(2+2^{-n})\leq \log(2^n)+\log(3)$ for all $n\geq 0$, so
\begin{align*}
    \P\left(\sup_{t\geq 0} \frac{|\Pi_t-W_t|}{\log(t \vee 1+1)} \geq y\right) \leq b'&   e^{-d'(y-a') \log(2) + d'a' \log(3)}  \\
    &+ \sum_{n\geq 1} b' e^{-d'(y-a') \log(2^n)+ d'a'\log(3)}
\end{align*}
so the upper bounds reads
\[
b' e^{d'a'\log(3)} \left( e^{-d'(y-a') \log(2)} + \sum_{n\geq 1} 2^{-n d'(y-a')}\right).
\]
If $y=a'+u$, with $u>0$, we get
\[
b' e^{d'a'\log(3)} \left( 2^{-d'u} + 2^{-d'u}(1-2^{-d'u})^{-1}\right).
\]
By taking $u=1/d'+x$ with $x>0$, we get that $(1-2^{-d'u})^{-1}\leq 1/2$, so that there exists $a",b", d" >0$ such that

\begin{equation}
\label{depart}
\P\left(\sup_{t\geq 0} \frac{|\Pi_t-W_t|}{\log(t \vee 1+1)} \geq a"+x\right)\leq b" e^{-d"x}.
\end{equation}
Now to go further, let us quote Lemma 3.12 of the Arxiv version~\cite{prodhomme2020arxiv} of \cite{prodhomme2023strong}.
\begin{lemma}
Let $E_1$ and $E_2$ be two complete separable metric spaces and let $\mu$ be a probability on $(E_1\times E_2, \mathcal{B}(E_1)\otimes\mathcal{B}(E_2))$. Let $\mu_1$ be the first marginal of $\mu$. Then there exists a measurable function $G: E_1 \times [0,1]\to E_2$ such that if $(X_1,V)\sim \mu_1\otimes \mathcal{U}([0,1])$, then $(X_1,G(X_1,V))\sim \mu.$
\end{lemma}
Both  $(\Pi_t)_{t\geq0}$ and $(W_t)_{t\geq 0}$ live in the Skorohod space $\mathcal{D}_{\R}([0,+\infty))$ which is separable and complete (polish).
So,  one can apply the lemma of Prodhomme: for any Poisson process $(\Pi_t)_{t\geq 0}$ with rate \(1\), by  adding an extra uniform variable $V$, independent of anything else,  
one construct a drifted Brownian $(W_t)_{t\geq 0}=G((\Pi_t)_{t\geq 0},V)$, such that $(\Pi_t)_{t\geq 0}$ and $(W_t)_{t\geq 0}$ satisfy \eqref{depart}.
\end{proof}

\begin{corollary} \label{cor kurtz new}
     Let $\alpha>0$ and $(\Pi_t)_{t\geq 0}$ a homogeneous Poisson process with intensity $\mu>0$. Then there exists a Brownian motion $(W_t)_{t\geq 0}$ with drift $\mu$ and scaling $\sqrt{\mu}$ and absolute positive constants $C, \zeta, K$ such that
     \begin{equation*}
\P\left(\sup_{t\geq 0} \frac{|\Pi_t-W_t|}{\log(\mu t \vee 1+1)} \geq C +\zeta\log(K\alpha^{-1})\right)\leq \alpha.
\end{equation*}
\end{corollary}

\begin{proof}
Let $\alpha>0$. Let $\Pi_t$ a homogeneous Poisson process with intensity $\mu >0$. Then $\Tilde{\Pi}_t \coloneqq \Pi_{\mu^{-1} t}$ is a standard Poisson process so acccording to Lemma~\ref{original_coupling}, there exists a standard Brownian motion $B_t$ and absolute constants $a,b,d >0$ such that for all $x>0$
\begin{equation*}
\P\left(\sup_{t\geq 0} \frac{|\Tilde{\Pi}_t-t - B_t|}{\log(t \vee 1+1)} \geq a+x\right)\leq b e^{-d x}.
\end{equation*}
Let $W_t \coloneqq B_{\mu t} + \mu t$. Then
\begin{equation*}
\P\left(\sup_{t\geq 0} \frac{|\Pi_t-W_t|}{\log(\mu t \vee 1+1)} \geq a+x\right)\leq b e^{-d x}.
\end{equation*}
By choosing $x$ such that $b e^{-d x} = \alpha$, \ie $x=\frac{1}{d}\log(b/\alpha)$, $C = a$, $K = b$ and $\zeta = d^{-1}$ we have the result.
\end{proof}

The following Lemma and its proof can be found in \cite{ost2023neural} (Lemma $7.1$).
\begin{lemma} \label{lemma ost rb}
    Let $X$ be a Poisson random variable with parameter $\gamma >0$. For any $x\geq 0$,
    \begin{equation} \label{ineg poisson}
         \mathbb{P}(X \geq \gamma(1+x)) \leq \exp\left(-\frac{\gamma x^2}{2(1+\frac{x}{3})}\right)
    \end{equation}
    and for any $0\leq x \leq \gamma$,
    \begin{equation} \label{ineg poisson 2}
        \mathbb{P}(X \leq \gamma- x) \leq \exp\left(-\frac{x^2}{2\gamma}\right).
    \end{equation}
\end{lemma}

\begin{corollary} \label{cor concent poisson}
    Let $\alpha >0$. 
    Let $X$ be a Poisson random variable with parameter $\gamma >0$.
    Then 
    \begin{enumerate}
\item \label{item1_corpoisson}    If $\gamma \geq 2\log(\alpha^{-1})$, we have
    \begin{equation} \label{eq concent theta - X}
         \mathbb{P}\left(X \geq \gamma- \sqrt{2\gamma \log(\alpha^{-1})}\right) \geq 1-\alpha.
    \end{equation}
\item  \label{item2_corpoisson}  If $\gamma \geq \dfrac{8}{3}\log(\alpha^{-1})$, we have
    \begin{equation} \label{eq concent X - theta}
         \mathbb{P}\left(X \leq \gamma +  \sqrt{\frac{8}{3}\gamma \log(\alpha^{-1})}\right) \geq 1- \alpha.
    \end{equation}
\item \label{item3_corpoisson}   If  $\gamma \geq \dfrac{8}{3}\log(2\alpha^{-1})$, we have
    \begin{equation} \label{eq concent abs X - theta}
     \mathbb{P}\left(\abs{X-\gamma} \leq  \sqrt{\frac{8}{3}\gamma \log(2\alpha^{-1})}\right) \geq 1 - \alpha.
\end{equation}
\end{enumerate}
\end{corollary}

\begin{proof}
\ref{item1_corpoisson} Assume $\gamma \geq 2\log(\alpha^{-1})$, then \(x\coloneqq \sqrt{2\gamma \log(\alpha^{-1})}\) satisfies \(0\leq x \leq \gamma\).
We apply \eqref{ineg poisson 2} and deduce 
\[
\mathbb{P}\left(X \leq \gamma- \sqrt{2\gamma \log(\alpha^{-1})}\right) \leq \alpha,
\]
and deduce \eqref{eq concent theta - X}.

Proof of~\ref{item2_corpoisson}: Thanks to the assumption, \(u\coloneqq \sqrt{\dfrac{8}{3}\gamma \log(\alpha^{-1})} \leq \gamma\).
The inequality \eqref{ineg poisson} can be rewritten as
\[
 \mathbb{P}(X \geq \gamma + u)  \leq \exp\left(-\frac{u^2}{2\gamma +\frac{2u}{3}}\right) 
 \leq \exp\left(-\frac{u^2}{2\gamma +\frac{2u}{3}}\right) \leq \exp\left(-\frac{3u^2}{8\gamma}\right).
\]

Proof of~\ref{item3_corpoisson}: We apply \eqref{eq concent theta - X} and \eqref{eq concent X - theta}
with \(\alpha/2\) and conclude the proof.
\end{proof}

\begin{lemma} \label{lemma sup poisson}
        Let $\Pi_t$ be a homogeneous Poisson process with parameter $\mu>0$. Then for every $x \geq 0$ we have
         \[
        \mathbb{P}\left(\sup_{0\leq t \leq T}( \Pi_t - \mu t) ) \geq \mu T x\right) \leq \exp\left(
        - \frac{\mu T x^2}{2(1 + x /3)}.
        \right)
        \]
        and for every $x\in [0, \mu T]$ we have
        \[
        \mathbb{P}\left(\sup_{0\leq t \leq T}( \mu t - \Pi_t ) \geq x\right) \leq \exp\left(
        - \frac{x^2}{2\mu T}.
        \right)
        \]
    \end{lemma}
    
    \begin{proof}
        The process $(\Pi_t - \mu t)_{t\leq 0}$ is a right-continuous martingale. Let $C>0, \xi >0$. Then $(\exp(\xi(\Pi_t - \mu t))_{t\leq 0}$ is a right-continuous submartingale. 
        Therefore, according to the first submartingale inequality (Theorem $3.8$ page 13 of \cite{karatzasbrownian}), we have
        \[
        \mathbb{P}\left(\sup_{0\leq t \leq T} \exp(\xi(\Pi_t - \mu t) ) \geq C\right) \leq \frac{\mathbb{E}[e^{\xi (\Pi_T - \mu T) }]}{C}.
        \]
        We choose \(C = \exp(\xi\mu T x)\) and  mimic the proof of Lemma~\ref{lemma ost rb} to obtain:
        \[
        \mathbb{P}\left(\sup_{0\leq t \leq T}( \Pi_t - \mu t) ) \geq \mu T x\right) \leq \exp\left(
        - \frac{\mu T x^2}{2(1 + x /3)}.
        \right)
        \]
        Similarly, the process $(\mu t - \Pi_t)_{t\leq 0}$ is a right-continuous martingale so for $\xi >0$, the process $(\exp(\xi(\mu t - \Pi_t))_{t\leq 0}$ is a right-continuous submartingale. Therefore 
      \[
        \mathbb{P}\left(\sup_{0\leq t \leq T} \exp(\xi(\mu t - \Pi_t) ) \geq C\right) \leq \frac{\mathbb{E}[e^{\xi (\mu T - \Pi_T) }]}{C}.
        \]
        Again, we can mimic the proof of Lemma~\ref{lemma ost rb} and we get the same bound for $x\in [0, \mu T]$:
        \[
        \mathbb{P}\left(\sup_{0\leq t \leq T}( \mu t - \Pi_t ) \geq x\right) \leq \exp\left(
        - \frac{x^2}{2\mu T}.
        \right)
        \]
    \end{proof}

    \begin{corollary} \label{cor sup poisson}
         Let $\beta >0$ and $T\geq \frac{8}{3\mu}\log(2/\beta)$. Let $\Pi_t$ be a homogeneous Poisson process with parameter $\mu>0$. Then
          \[
        \mathbb{P}\left(\sup_{0\leq t \leq T}\abs{\Pi_t - \mu t} ) \geq \sqrt{\frac{8}{3}\mu T \log\Big(\frac{2}{\beta}\Big)}\right) \leq \beta.
    \]
    \end{corollary}
\begin{proof}
    Similarly as in the proof of \eqref{eq concent abs X - theta}, by combining the two inequalities given by Lemma~\ref{lemma sup poisson}, we have that for every $0\leq x \leq \mu T$
      \[
        \mathbb{P}\left(\sup_{0\leq t \leq T}\abs{\Pi_t - \mu t} ) \geq x\right) \leq 2\exp\left(
        - \frac{3 x^2}{8\mu T}.
        \right)
    \]
    Let us choose $x$ such that $2\exp\left(
        - \frac{3 x^2}{8\mu T}\right) = \beta$, \ie $x = \sqrt{\frac{8}{3}\mu T \log\Big(\frac{2}{\beta}\Big)}$. The condition $x\leq \mu T$ becomes $T\geq \frac{8}{3\mu}\log(2/\beta)$.
\end{proof}

\begin{lemma}
\label{Control_brownian_2times}
There exist  constants $c,c',c">0$ such that if 
 $(B_t)_{t\geq 0}$ a standard centered Brownian motion
then for all $x>0$
\begin{equation*}
\P\left(\sup_{t, u \geq 0} \frac{|B_t-B_u|}{\sqrt{(|t-u|+2)[\log(t+1)+\log(u+1)+2x]}} \geq c\right)\leq c' e^{-c'' x}.
\end{equation*}
\end{lemma}

This Lemma resembles Lemma~6 of \cite{chevallier2021diffusion} which is probably tighter, but is valid only for $t,u\leq T$ whereas our version is on all $\R_+$.
\begin{proof}
Remark that
\begin{align*}
    &\sup_{t,u \geq 0} \frac{|B_t-B_u|}{\sqrt{(|t-u|+2)[\log(t+1)+\log(u+1)+2x]}} \\
    &=\sup_{t\geq u \geq 0} \frac{|B_t-B_u|}{\sqrt{(|t-u|+2)[\log(t+1)+\log(u+1)+2x]}},
\end{align*}
because the quantity on which we are taking the supremum is invariant by exchange of $t$ and $u$.

We start from the article \cite{robbins1970boundary}, where Example 3 page 88 says that for all $y>0$,
\begin{equation}\label{RobbinsSiegmuns70}
\P\left(\sup_{t\geq 0} \frac{|B_t|}{\sqrt{(t+1)[\log(t+1)+2y]}}\geq 1\right) \leq e^{-y}.
\end{equation}

Let us now look at the event
\begin{multline*}
\left\{\sup_{t \geq u \geq 0} \frac{|B_t-B_u|}{\sqrt{(|t-u|+2)[\log(t+1)+\log(u+1)+2x]}} \geq a
\right\}= \\
\bigcup_{n\in \mathbb{N}} \left\{\sup_{t \geq u \geq 0, u\in[n,n+1)} \frac{|B_t-B_u|}{\sqrt{(|t-u|+2)[\log(t+1)+\log(u+1)+2x]}} \geq a
\right\}.
\end{multline*}
But for fixed $n\in \mathbb{N}$, for every $t\geq u, u\in [n,n+1)$
\begin{multline*}
\frac{|B_t-B_u|}{\sqrt{(|t-u|+2)[\log(t+1)+\log(u+1)+2x]}} \leq \\ \frac{|B_t-B_n|}{\sqrt{(|t-u|+2)[\log(t+1)+\log(u+1)+2x]}}+\frac{|B_u-B_n|}{\sqrt{(|t-u|+2)[\log(t+1)+\log(u+1)+2x]})} \leq \\
\frac{|B_t-B_n|}{\sqrt{(t-n+1)[\log(t+1)+\log(n+1)+2x]}}+\frac{|B_u-B_n|}{\sqrt{(|u-n|+1)[\log(n+1)+\log(u+1)+2x]})}.
\end{multline*}
This leads to
\begin{multline*}
\left\{\sup_{t \geq u \geq 0, u\in[n,n+1)} \frac{|B_t-B_u|}{\sqrt{(|t-u|+2)[\log(t+1)+\log(u+1)+2x]}} \geq a
\right\} \subset\\
\left\{\sup_{t \geq n} \frac{|B_t-B_n|}{\sqrt{(t-n+1)[\log(t+1)+\log(n+1)+2x]}} \geq a/2
\right\} \\
\bigcup 
\left\{\sup_{u \geq n} \frac{|B_u-B_n|}{\sqrt{(|u-n|+1)[\log(n+1)+\log(u+1)+2x]})} \geq a/2
\right\}
\end{multline*}
But in each case, with $a\geq 2$ we can apply \eqref{RobbinsSiegmuns70} to respectively the Brownian $(B_t-B_n)_{t\geq n}$ and the Brownian $(B_u-B_n)_{t\geq n}$ with $y=\frac{a}{2}(\log(n+1)/2+x)$.
Hence by union bound, we get that
\begin{align*}
\P\left(\sup_{t \geq u \geq 0} \frac{|B_t-B_u|}{\sqrt{(|t-u|+2)[\log(t+1)+\log(u+1)+2x]}} \geq a \right) &\leq 
\sum_{n=0}^{\infty} 2 e^{- \frac{a}{2}(\log(n+2)/2+x)} \\
&\leq 2 \sum_{n=0}^{\infty} (n+1)^{-a/4} e^{- ax/2}
\end{align*}
By taking $a=8$ for instance, the previous sum is finite and we get exactly the result.
\end{proof}

\subsection{Proof of Theorem~\ref{th ddm}}

We start from inequality \eqref{RobbinsSiegmuns70} which says that for all $y>0$, if $(B_t)_{t\geq 0}$ is a standard Brownian motion then
\begin{equation*}
\P\left(\sup_{t\geq 0} \frac{|B_t|}{\sqrt{(t+1)[\log(t+1)+2y]}}\geq 1\right) \leq e^{-y}.
\end{equation*}

 Let $\alpha >0$,  $m\in \N^*$,  $o\in \Obj$ the nature of object $m$ and $j^*$ its category. Let $\Tilde{J} \coloneqq \{j\in J \text{ such that } \mu^j_o >0\}$. Indeed, if $\mu_o^j=0$, then for all $t$ we have $W^j_{m,t} = 0$ and the process $W^j_m$ will never reach $\thr$. According to Assumption~\ref{assump drift}, we have necessarily $j^*\in \Tilde{J}$. Up to consider $\Tilde{J}$ instead of J, let us assume $J = \Tilde{J}$.
 
 For every $j\in J$, the process $W^j_{m,t}$ can be written as \eqref{eq def w}. By choosing $y = \log(\abs{J}\alpha^{-1})$, according to \eqref{RobbinsSiegmuns70}, by union bound, we have with probability more than $1-\alpha$ that for every $j\in J$, 
\[
\sup_{t\geq 0} \frac{|B^j_t|}{\sqrt{(t+1)[\log(t+1)+2y]}}\leq 1. 
\]
In particular, for every $t\geq 0$ we have
\begin{align*}
    W^{j^*}_{m,t} &\geq \mu^{j^*}_{o} t - \sqrt{(\mu^{j^*}_{o}t+1)(\log(`t+1)+2y)} \eqqcolon I(\mu^{j^*}_{o},t)
\end{align*}
and for $j\neq j^*$,%
we get
\[
W^{j}_{m,t} \leq \mu^j_{o} t + \sqrt{(\mu^j_{o} t+1)(\log(\mu^j_{o} t+1)+2y)} \eqqcolon S(\mu^j_{o},t).
\]
A sufficient condition for $W^{j^*}_m$ to reach $\thr$ before limit duration $T$ 
and to be the first to do so is the existence of $t_0\in [0,T]$ such that $I(\mu^{j^*}_{o},t_0) = \thr$ and for every $j\neq j^*$, $S(\mu^{j}_{o},t_0) < I(\mu^{j^*}_{o},t_0)$. It would mean that at time $t_0$, $W^{j^*}_m$ has already reached $\thr$ and none of the other processes have done so.

Let $t_0\geq 0$. The condition 
$S(\mu^{j}_{o},t_0) < I(\mu^{j^*}_{o},t_0)$ is equivalent to
\[
(\mu^{j^*}_{o} - \mu^j_{o})t_0 >  \sqrt{(\mu^j_{o} t_0+1)(\log(\mu^j_{o} t_0+1)+2y)} + \sqrt{(\mu^{j^*}_{o}t_0+1)(\log(\mu^{j^*}_{o}t_0+1)+2y)}
\]
which is implied by
\[
\Dmu t_0 > 2 \sqrt{(\norm{\mu}_{\infty} t_0+1)(\log(\norm{\mu}_{\infty} t_0+1)+2y)}.
\]
which is equivalent to
\[
\Dmuu t_0^2 > 4 (\norm{\mu}_{\infty} t_0+1)(\log(\norm{\mu}_{\infty} t_0+1)+2y).
\]
Suppose $t_0>\frac{1}{\norm{\mu}_{\infty}}$. Then a sufficient condition is
\begin{equation} \label{eq interm}
    \Dmuu t_0^2 > 4 (2\norm{\mu}_{\infty} t_0)(\log(2\norm{\mu}_{\infty} t_0)+2y).
\end{equation}
Let us denote $v\coloneqq 2\norm{\mu}_{\infty} t_0$. Then \eqref{eq interm} can be rewritten as
\[
\frac{\Dmuu}{4\norm{\mu}_{\infty}^2} v > 4 \log(v) + 8 y 
\]
Let $x \coloneqq \frac{\Dmuu}{4\norm{\mu}_{\infty}^2} v $. %
Then a sufficient condition is
\begin{align} \label{eq x}
    x &\geq 4 \log(x) + 4 \log\left(\frac{4\norm{\mu}_{\infty}^2}{\Dmuu} \right) + 8 y %
\end{align}
We can find a constant $x_0$, independent of everything, such that for every $x\geq x_0$, $\frac{x}{2}\geq 4 \log(x)$. 
Then if $x>x_0$ and $x/2> 4 \log\left(\frac{4\norm{\mu}_{\infty}^2}{\Dmuu} \right) + 8 y$, then \eqref{eq x} holds.

To summarize, since $y=\log(|J|\alpha^{-1})$, if 
\[
t_0\geq \square \frac{1}{\Dmuu}\max( \log(\frac{4\norm{\mu}_{\infty}^2}{\Dmuu}), \log(|J|\alpha^{-1}))\eqqcolon t_-,
\]
we have that $I(\mu^{j^*}_{o},t_0)> S(\mu^j_{o},t_0)$ for all $j\neq j^*$.

By the assumptions of the theorem, we have that 
\[\theta \geq \norm{\mu}_\infty t_- \geq I(\mu^{j^*}_{o},t_-)\] 
and
\[\theta \leq \Dmu T - \sqrt{(\norm{\mu}_\infty T +1)(\log(\norm{\mu}_\infty T +1)+2\log(|J|\alpha^{-1}))} \leq I(\mu^{j^*}_{o},T).
\]
So we have that $I(\mu^{j^*}_{o},t_-)\leq I(\mu^{j^*}_{o},T)$ and that $\theta \in [I(\mu^{j^*}_{o},t_-),I(\mu^{j^*}_{o},T)]$. 
By continuity of 
$I(\mu^{j^*}_{o},.)$, there exists $t_0$ in $[t_-,T]$ such that $I(\mu^{j^*}_{o},t_0)=\theta$ 
and since $t_0\geq t_-$, we also have that $S(\mu^j_{o},t_0)<I(\mu^{j^*}_{o},t_0)$ for all $j\neq j^*$, which concludes the proof.

\subsection{Proof of Theorem~\ref{th poisson correct categ}}

Let $\alpha >0$,  $m\in \N^*$,  $o\in \Obj$ the nature of object $m$, $j^*$ its category and $\Dgam$ the constant given by 
Assumption~\ref{assump intensity poisson}. Let $\Tilde{J} \coloneqq \{j\in J \text{ such that } \gamma^j_o >0\}$. 
According to Assumption~\ref{assump intensity poisson}, we have necessarily $j^*\in \Tilde{J}$. Up to consider $\Tilde{J}$ instead of J, let us assume $J = \Tilde{J}$. For $t\geq 0$ and $x\geq 0$, let
\[
S(t) \coloneqq (\gamma^{j^*}_o - \Dgam) t +  \sqrt{\frac{8}{3} \gamma^{j^*}_o t\log\left(\frac{2\abs{J}}{\alpha}\right)} 
\]
and
\[
I(t) \coloneqq \gamma^{j^*}_o t -  \sqrt{\frac{8}{3} \gamma^{j^*}_o t \log\left(\frac{2\abs{J}}{\alpha}\right)}.
\]
Let $t_- \coloneqq \frac{8}{3}\log(\frac{2\abs{J}}{\alpha})\max\left(\gamma_{\text{min}}^{-1},  4\gamma^{j^*}_o \Dgam^{-2}\right)$ and $t_0 \in \left[t_-,T \right]$ to choose later. Let $\Omega$ be the event
    \[
    \Omega \coloneqq \left\{ \forall j\in J,  \abs{\Pi^j_{m,t_0} - \gamma^j_o t_0} < \sqrt{\frac{8}{3} t_0 \gamma^j_o \log\left(\frac{2\abs{J}}{\alpha}\right)} \right\}.
    \]
    Since the processes $(\Pi^j_{m,t})_{t\geq 0}$ are Poisson processes with intensity $\gamma^j_o>0$ and $t_0  \geq \frac{8}{3\gamma^j_o}\log(\frac{2\abs{J}}{\alpha})$ 
    by lower bounding $\gamma^j_o$ by $\gamma_{\text{min}}$, by union bound on the $j\in J$, according to 
    inequality \eqref{eq concent abs X - theta} from Corollary~\ref{cor concent poisson},
    the event $\Omega$ has probability more than $1-\alpha$. Let us work on $\Omega$. We have
    \[
    \Pi^{j^*}_{m,t_0} \geq \gamma^{j^*}_o t_0 -  \sqrt{\frac{8}{3} \gamma^{j^*}_o t_0 \log\left(\frac{2\abs{J}}{\alpha}\right)} = I(t_0)
    \]
    and for every $j\neq j^*$, by bounding $\gamma^j_o$ by $\gamma^{j^*}_o - \Dgam$ and $\gamma^{j^*}_o$ we have 
    \[
    \Pi^j_{m,t_0} \leq \gamma^j_o t_0 +  \sqrt{\frac{8}{3} \gamma^j_o t_0 \log\left(\frac{2\abs{J}}{\alpha}\right)} < (\gamma^{j^*}_o - \Dgam) t_0 +  \sqrt{\frac{8}{3} \gamma^{j^*}_o t_0\log\left(\frac{2\abs{J}}{\alpha}\right)} = S(t_0).
    \]
    thanks to Assumption~\ref{assump intensity poisson}. Let us study under which condition we have $S(t_0) < I(t_0)$. This is equivalent to the condition
    \[
    t_0 \geq  \frac{32}{3\Dgamm} \gamma^{j^*}_o \log\left(\frac{2\abs{J}}{\alpha}\right)
    \]
    which we assumed. Suppose $\thr\in [I(t_-), I(T)]$. By continuity of $t\mapsto I(t)$, there exists $t_0\in [t_-,T]$ such that $I(t_0) =\thr$. Let us choose such $t_0$. Therefore, at time $t_0$, we have $\Pi^{j^*}_{m,t_0}\geq \thr$ and for every $j\neq j^*$, $\Pi^j_{m,t_0} < \thr$. This means that at time $t_0$, the process $(\Pi^{j^*}_{m,t})_{t\geq 0}$ has already reached $\thr$ and since the processes $(\Pi^j_{m,t})_{t\geq 0}$ are increasing, every other process $(\Pi^j_{m,t})_{t\geq 0}$ where $j\neq j^*$ has not reached $\thr$ yet. Therefore, on $\Omega$, under the condition $\thr\in [I(t_-), I(T)]$, the process  $(\Pi^{j^*}_{m,t})_{t\geq 0}$ reaches $\thr$ before limit duration $T$ and is the first to do so. Since $I(t_-) \leq  \gamma^{j^*}_o t_- $ by definition of the function $I$, sufficient conditions on $\thr$ are
    \[
    \frac{8}{3}\log\left(\frac{2\abs{J}}{\alpha}\right)\max\left(\frac{\gamma^{j^*}_o}{\gamma_{\text{min}}},  4\left(\frac{\gamma^{j^*}_o}{\Dgam}\right)^2\right)\leq \thr
    \]
    and
    \[
    \thr \leq  \gamma^{j^*}_o T -  \sqrt{\frac{8}{3} \gamma^{j^*}_o T \log\left(\frac{2\abs{J}}{\alpha}\right)}.
    \]
    In order to get conditions independent of $o$ and $j^*$, we upper bound $\gamma^{j^*}_o$ by $\norm{\gamma}_{\infty}$ and lower bound it by $\Dgam$.

\subsection{Proof of Theorem~\ref{th coupling ddm poisson}} \label{proof th coupling brown poiss}

Let $m\in \N^*$, $o$ the nature of object $m$ and $n\geq 1$. Then \eqref{eq coupling poisson ddm} is a direct consequence of Corollary~\ref{cor kurtz new} by union bound on the $j\in J$ and applying the time change $t\mapsto n\gamma^j_o t$. 

Let $\alpha\in (0,1)$ and $U \coloneqq \max_{j\in J}\{\frac{\thr}{\gamma^j_o}\} + 1 = \frac{\thr}{\gamma_{\text{min}}} + 1$.%
\newline

$\bullet$ Control of the difference $\abs{\Pi^{j,n}_{m,t} - \W^{j,n}_{m,t}}$. 
\newline

For each $j\in J$, let us apply Corollary~\ref{cor kurtz new} to the Poisson process $(\Pi^{j,n}_{m,t})_{t\geq 0}$ having intensity $n\gamma^j_o$. By choosing uniform variables $V^j$ independent of everything else, there exist Brownian motions $W^{j,n}_{m,t}$ with drift $n\gamma^j_o$ and scaling $\sqrt{n\gamma^j_o}$ depending only on $(\Pi^{j,n}_{m,t})_{t\geq 0}$ and $V^j$ such that the event
\begin{equation} \label{eq kurtz}
   \Omega_1 \coloneqq \left\{ \forall j\in J, \quad \sup_{t\geq 0} \frac{|\Pi^{j,n}_{m,t}-W^{j,n}_{m,t}|}{\log(n\gamma^j_o t \vee 1+1)} \leq C +\zeta\log(5K\abs{J}\alpha^{-1}) \right\}
\end{equation}
has probability more than $1-\frac{\alpha}{5}$ by union bound, where $C$, $K$ and $\zeta$ are positive constants independent of the parameters of the problem. Let us work on $\Omega_1$.
\newline

$\bullet$ Study of the hitting times of $\W^{j,n}_{m,t}$ and $\Pi^{j,n}_{m,t}$. 
\newline

We introduce the first hitting times  $\taupibis \coloneqq \inf \left\{u\in [0,U] \text{ s.t. } \Pi^{j,n}_{m,u} \geq n\thr \right\}$ and $\tauwbis = \inf \left\{u\in [0,U] \text{ s.t. } \W^{j,n}_{m,u} \geq n\thr \right\}$,  and we set \(\taupibis, \tauwbis = \infty\) if  the previous sets are empty. First, let us study under which conditions $\taupibis < \infty$ and $\tauwbis < \infty$. 

Let $\Omega_2$ be the event
\[
\Omega_2 \coloneqq \left\{\forall j\in J, \quad \Pi^{j,n}_{m,U} \geq  n\gamma^j_o U - \sqrt{2n\gamma^j_o U \log(5\abs{J}\alpha^{-1})}\right\}.
\]
Then $\Pi^{j,n}_{m,U}$ is a Poisson variable with parameter $n\gamma^j_o U$. 
Let us assume that for every $j\in J$ we have $n\gamma^j_o U \geq 2 \log(\abs{J}\beta_n^{-1})$. 
Then according to \eqref{eq concent theta - X} of Corollary~\ref{cor concent poisson}, by union bound 
$\mathbb{P}(\Omega_2) \geq 1 - \frac{\alpha}{5}$. Let us work on $\Omega_1 \cap \Omega_2$. 
Then a sufficient condition to have $\taupibis < \infty$ is to have $ n\gamma^j_o U - 
\sqrt{2n\gamma^j_o U \log(5\abs{J}\alpha^{-1})} \geq n\thr$, which is equivalent to the condition
\begin{equation} \label{eq cond 1 U}
    \frac{\thr}{\gamma^j_o} \leq U - \sqrt{\frac{2U}{n\gamma^j_o} \log(5\abs{J}\alpha^{-1})}.
\end{equation}
which is fulfilled since we assumed that $n$ verifies
\begin{equation} \label{eq cond n}
    n \geq \frac{2}{\gamma_{\text{min}}}
    \left(\frac{\thr}{\gamma_{\text{min}}} + 1\right) \log\left( \frac{10\abs{J}}{\alpha}\right).
\end{equation}
Let $\Omega_3$ be the event
\[
\Omega_3 \coloneqq \left\{ \forall j\in J, \quad \sup_{0\leq u \leq U} \abs{\W^{j,n}_{m,t} - n\gamma^j_o t} < \sqrt{2Un\gamma^j_o \log(10\abs{J}\alpha^{-1})} \right\}
\]
Then according to Lemma~\ref{lemma control z - mu}, we have $\mathbb{P}(\Omega_3) \geq 1-\frac{\alpha}{5}$. Let us work on $\Omega_1\cap \Omega_2\cap\Omega_3$. We have then
\[
\W^{j,n}_{m,U} \geq n\gamma^j_o U - \sqrt{2U n\gamma^j_o\log(10\abs{J}\alpha^{-1})}.
\]
A sufficient condition to have $\tauwbis < \infty$ is to have 
$n\gamma^j_o U - \sqrt{2Un\gamma^j_o\log(10\abs{J}\alpha^{-1})}\geq n\thr$, which is fulfilled since $n$ verifies \eqref{eq cond n}.
Therefore, on $\Omega_1 \cap \Omega_2 \cap \Omega_3$, we have $\taupibis < \infty, \taupibis<\infty$ 
and consequently $\taupi = \taupibis$ and $\tauw = \tauwbis$. Therefore, for every $j\in J$ and 
$Z\in\{\Pi,\W\} $, we have $\tau^{Z^j}_n\leq U$, and controlling the trajectory 
of $Z^{j,n}_m$ on $[0,U]$ is sufficient to control it on $[0,\tau^{Z^j}_n]$.

Using this fact, let us study the 
convergence of $\taupi$ and $\tauw$ to $\frac{\thr}{\gamma^j_o}$.
\newline

$1$. Convergence of $\taupi$. According to Corollary~\ref{cor sup poisson}, for $U \geq \frac{8}{3n\gamma^j_o}\log(10\abs{J}\alpha^{-1})$, which holds for large $n$, the event
\[
\Omega_4 \coloneqq \left\{ \forall j\in J, \quad \sup_{0\leq t \leq U} \abs{\Pi^{j,n}_{m,t} - n\gamma^j_o t} \leq \sqrt{\frac{8}{3}n\gamma^j_o U \log(10\abs{J}\alpha^{-1})} \right\}
\]
has probability more than $1-\frac{\alpha}{5}$. Let us work on $\Omega_1 \cap \Omega_2\cap\Omega_3\cap \Omega_4$.
Since $\Pi^{j,n}_{m,\taupi}$ is an integer and the process $(\Pi^{j,n}_{m,t})_{t\geq 0}$ cannot make two jumps at the same time, we have $\Pi^{j,n}_{m,\taupi} = \lceil n\thr \rceil$. Therefore, for every $j\in J$ we have
\[
\abs{n\thr - n\gamma^j_o \taupi} \leq \abs{n\thr - \lceil n\thr \rceil} + \abs{\Pi^{j,n}_{m,\taupi} 
- n \gamma^j_o \taupi} \leq 1 + \sqrt{\frac{8}{3}n\gamma^j_o U \log(10\abs{J}\alpha^{-1})} 
\]
since we work on $\Omega_4$. Hence
\begin{align*}
    \AAbs{\taupi - \frac{\thr}{\gamma^j_o}} &\leq \frac{1}{n\gamma^j_o} +  
\sqrt{\frac{8 U}{3 n\gamma^j_o}\log(10\abs{J}\alpha^{-1})} .
\end{align*}
and we get the upper bound given by the theorem by lower bounding $\gamma^j_o$ by $\gamma_{\text{min}}$ and replacing $U$ by its value.

$2$. Convergence of $\tauw$. Since we work on $\Omega_3$ we have
\[
\abs{\W^{j,n}_{m,\tauw} - n \gamma^j_o \tauw} \leq \sqrt{2Un\gamma^j_o \log(10\abs{J}\alpha^{-1})}.
\]
Since the trajectories of $(\W^{j,n}_{m,t})_{t\geq 0}$ are continuous, by definition of 
$\tauw$ we have $\W^{j,n}_{m,\tauw} = n\thr$ so
\[
\AAbs{\tauw - \frac{\thr}{\gamma^j_o}} \leq \sqrt{\frac{2U\log(10\abs{J}\alpha^{-1})}{n\gamma^j_o}}.
\]
and we get the upper bound given by the theorem by lower bounding $\gamma^j_o$ by $\gamma_{\text{min}}$.
\newline

$3$. Study of the difference $\abs{\taupi - \tauw}$. Let us begin by bounding $\AAbs{\W^{j,n}_{m,\taupi} - \W^{j,n}_{m,\tauw}}$. We have
\begin{align}
    \AAbs{W^{j,n}_{m,\taupi} - W^{j,n}_{m,\tauw}} &= \AAbs{W^{j,n}_{m,\taupi} - n\thr} \nonumber\\
    &\leq \AAbs{W^{j,n}_{m,\taupi} - \lceil n\thr \rceil } + \abs{ \lceil n\thr \rceil - n\thr} \nonumber\\
    &\leq  \AAbs{W^{j,n}_{m,\taupi} - N^{j,n}_{m,\taupi}} + 1 \nonumber \\
    &\leq \log(n\norm{\gamma}_\infty U + 1) (C + \zeta \log(5K\abs{J}\alpha^{-1}) + 1 \nonumber\\
    & \leq \Box_{\norm{\gamma}_\infty, \theta, \gamma_{\text{min}}} \log(n)\log(\abs{J}\alpha^{-1}) \label{eq W - W poisson}
\end{align}
since we are on $\Omega_1$ \eqref{eq kurtz} and since $\taupi$ is bounded by $U$.

Now, let us apply Lemma~\ref{Control_brownian_2times} to the standard Brownian motion $B^j_t \coloneqq W^{j,n}_{m,(n\gamma^j_o)^{-1} t} - t$. By union bound, the event
\[ \Omega_5 \coloneqq \left\{\forall j\in J, 
\sup_{t, u \geq 0} \frac{|B^j_t-B^j_u|}{\sqrt{(|t-u|+2)[\log(t+1)+\log(u+1)+\frac{2}{c''}\log(5c'\abs{J}\alpha^{-1})]}} \leq c\right\}
\]
has probability more than $1-\frac{\alpha}{5}$. Let us work on $\Omega_1\cap\Omega_2\cap\Omega_3\cap\Omega_4\cap\Omega_5$. Since $\taupi$ and $\tauw$ are bounded by $U$ we have
\begin{align*}
   & \abs{B^j_{n\gamma^j_o \taupi} - B^j_{n\gamma^j_o\tauw}} \leq \Box_{\thr,\gamma_{\text{min}},\norm{\gamma}_\infty} \sqrt{((n |\taupi-\tauw|)\vee 1)\log(n\abs{J}\alpha^{-1})}.
\end{align*}
If $((n |\taupi-\tauw|)\vee 1) = 1$ then $\abs{\taupi - \tauw} \leq \frac{1}{n}$ and we have the result. Otherwise we have by definition of $B^j$
\begin{align*}
    \AAbs{W^{j,n}_{m,\taupi} - W^{j,n}_{m,\taupi} - n\gamma^j_o(\taupi - \tauw)} \leq  \Box_{\thr,\gamma_{\text{min}},\norm{\gamma}_\infty} \sqrt{n |\taupi-\tauw|\log(n\abs{J}\alpha^{-1})}
\end{align*}
thus
\begin{align*}
    \AAbs{W^{j,n}_{m,\taupi} - W^{j,n}_{m,\taupi}
    } \geq  n\gamma^j_o\abs{\taupi - \tauw} - \Box_{\thr,\gamma_{\text{min}},\norm{\gamma}_\infty} \sqrt{n |\taupi-\tauw|\log(n\abs{J}\alpha^{-1})}.
\end{align*}
Combining this with \eqref{eq W - W poisson}, we obtain the following inequality:
\[
X^2 - a(n) X - b(n) \leq 0
\]
with $X \coloneqq \sqrt{n\abs{\taupi - \tauw}}$, $a(n) \coloneqq  \Box_{\thr,\gamma_{\text{min}},\norm{\gamma}_\infty} \sqrt{\log(n\abs{J}\alpha^{-1})}$ and $b(n)\coloneqq\Box_{\norm{\gamma}_\infty, \theta, \gamma_{\text{min}}} \log(n)\log(\abs{J}\alpha^{-1}) $. This polynomial has two real roots, the largest one being $x_0 = \frac{a(n)+\sqrt{a(n)^2 + 4b(n)}}{2}$, so necessarily we have $X\leq x_0$, which implies that $X\leq a(n) + \sqrt{b(n)}$, \ie
\begin{align*}
     \sqrt{n\abs{\taupi - \tauw}} &\leq   \Box_{\thr,\gamma_{\text{min}},\norm{\gamma}_\infty} \left(\sqrt{\log(n\abs{J}\alpha^{-1})} + \sqrt{ \log(n)\log(\abs{J}\alpha^{-1})} \right) \\
     &\leq \Box_{\thr,\gamma_{\text{min}},\norm{\gamma}_\infty}  \sqrt{ \log(n)\log(\abs{J}\alpha^{-1})} 
\end{align*}
therefore
\[
\abs{\taupi -\tauw} \leq  \Box_{\thr,\gamma_{\text{min}},\norm{\gamma}_\infty} \frac{\log(n)}{n} \log(\abs{J}\alpha^{-1}).
\]
To conclude, we note that $\mathbb{P}(\Omega_1\cap\Omega_2\cap\Omega_3\cap\Omega_4\cap\Omega_5)\geq 1-\alpha$.

\subsection{Proof of Proposition~\ref{prop conv weights}}

To ensure clarity, we sometimes denote by $o_m$ the nature of object $m$. We will need the following lemma. 

\begin{lemma} \label{lemma inputs}
    Let $\alpha>0$ and $\gamma_{\text{min}} \coloneqq \min_{i\in I, o\in \Obj} \{ \gamma^i_o \text{ such that } \gamma^i_o >0 \}$. Suppose Assumption~\ref{assump nb obj} holds and that the following condition is fulfilled:
    \[
      M T_{\text{min}} \geq \frac{8}{3} \log\left(\frac{2\abs{I}\abs{J}}{\alpha}\right) \frac{\abs{\Obj}}{\gamma_{\text{min}}}.
    \]
    Then with probability $1-\alpha$, for every $i\in I$ and $j\in J$ we have
    \begin{equation*}
        \AAbs{\sum_{m \text{ s.t. } o_m \in j} \Pi^i_{m,T_{\text{min}}} - \frac{M T_{\text{min}}}{\abs{\Obj}} \sum_{o\in j} \gamma^i_o} \leq \sqrt{\frac{8 M T_{\text{min}}}{3\abs{\Obj}} \sum_{o\in j} \gamma^i_o \log(2\abs{I}\abs{J}\alpha^{-1})}.
    \end{equation*}
\end{lemma}

\begin{proof}
  Let $\beta>0$, $j\in J$, $i\in I$. Then by independence of the processes $((\Pi^i_{m,t})_{t\geq 0})_{1\leq m \leq M}$, 
  the variable $\sum_{m \text{ s.t. } o_m \in j} \Pi^i_{m,T_{\text{min}}}$ follows a Poisson distribution with parameter $\sum_{m \text{ s.t. } o_m \in j} \gamma^i_{o_m} T_{\text{min}}$. Besides, according to Assumption~\ref{assump nb obj}, we have
  \begin{align*}
      \sum_{m \text{ s.t. } o_m \in j} \gamma^i_{o_m} &= \sum_{o\in j} \sum_{m, o_m = o} \gamma^i_o 
      = \frac{M}{\abs{\Obj}} \sum_{o\in j} \gamma^i_o.
  \end{align*}
  Let $\Omega_{i,j}$ be the event
  \[
  \Omega_{i,j} \coloneqq \left\{ \AAbs{\sum_{m \text{ s.t. } o_m \in j} \Pi^i_{m,T_{\text{min}}} - \frac{M T_{\text{min}}}{\abs{\Obj}} \sum_{o\in j} \gamma^i_o} \leq \sqrt{\frac{8 M T_{\text{min}}}{3\abs{\Obj}} \sum_{o\in j} \gamma^i_o \log(2\beta^{-1})} \right\}.
  \]
If $\sum_{o\in j} \gamma^i_o = 0$ then $\sum_{m \text{ s.t. } o_m \in j} \Pi^i_{m,T_{\text{min}}} = 0$ a.s. and $\mathbb{P}(\Omega_{i,j}) = 1$. Let us assume $\sum_{o\in j} \gamma^i_o > 0$. Then according to inequality \eqref{eq concent abs X - theta} of Corollary~\ref{cor concent poisson}, under the condition 
\begin{equation} \label{eq proof input}
    \frac{M T_{\text{min}}}{\abs{\Obj}} \sum_{o\in j} \gamma^i_o \geq \frac{8}{3} \log(2\beta^{-1})
\end{equation}
we have $\mathbb{P}(\Omega_{i,j})\geq 1-\beta$.  

Let us assume \eqref{eq proof input} holds for every $i,j$ such that $\sum_{o\in j} \gamma^i_o > 0$. Then 
\begin{align*}
    \mathbb{P}(\cup_{i\in I, j\in J} \Omega_{i,j}^c) &\leq \abs{I}\abs{J} \beta.
\end{align*}
  Let us choose $\beta = \frac{\alpha}{\abs{I}\abs{J}}$. Then \eqref{eq proof input} is implied by the condition
  \[
  M T_{\text{min}} \geq \frac{8}{3} \log\left(\frac{2\abs{I}\abs{J}}{\alpha}\right) \frac{\abs{\Obj}}{\gamma_{\text{min}}}
  \]
  so under this condition $\mathbb{P}\left(\cap_{i\in I, j\in J} \Omega_{i,j}\right) \geq 1-\alpha$ and we have the result.
\end{proof}

Now let us prove the proposition. Let $\alpha >0$ and suppose \eqref{eq M tmin} holds. Let us work on the event $\Omega \coloneqq \cap_{i\in I, j\in J} \Omega_{i,j}$ where the $\Omega_{i,j}$ are defined in the proof of Lemma~\ref{lemma inputs}. Then according to 
Lemma~\ref{lemma inputs}, $\mathbb{P}(\Omega)\geq 1-\alpha$. Let $i\in I$, $j\in J$ and for $m\in \{1,\dots,M\}$ let 
\[\Bar{g}^{i\to j}_m = \left\{
    \begin{array}{lll}
         \gamma^i_{o_m} \times \frac{M}{M^j} &\text{if } o \in j  \\ 
          \\
        - \gamma^i_{o_m} \!\times\! \frac{M}{M^{j'}} \!\times\! \frac{1}{\abs{J}-1} &\text{if } o \in j'{\neq j} 
    \end{array}
\right.
\]
Let $\Bar{G}^{i\to j}_M \coloneqq \sum_{m=1}^M \Bar{g}^{i\to j}_m$ and let $\Bar{w}^{i\to j}_{M+1} \coloneqq \dfrac{\exp(\eta \Bar{G}^{i\to j}_M)}{\sum_{l\in I} \exp(\eta \Bar{G}^{l\to j}_M)}$.
\newline

$\bullet$ Study of the difference $\abs{G^{i\to j}_M - \Bar{G}^{i\to j}_M}$.
\newline

According to Assumption~\ref{assump nb obj} we have
\begin{align*}
    \Bar{G}^{i\to j}_M &= \frac{M}{M^j}\sum_{m \text{ s.t. }o_m\in j} \gamma^i_{o_m} - \frac{1}{\abs{J}- 1} \sum_{j'\neq j} \frac{M}{M^{j'}} \sum_{m \text{ s.t. }o_m\in j'} \gamma^i_{o_m} \\
    &= \frac{M}{M^j} \frac{M}{\abs{\Obj}} \sum_{o\in j} \gamma^i_o -  \frac{1}{\abs{J}- 1} \sum_{j'\neq j} \frac{M}{M^{j'}}  \frac{M}{\abs{\Obj}} \sum_{o\in j'} \gamma^i_o
\end{align*}
However, according to Assumption~\ref{assump nb obj}, for any $j\in J$ we have
\[M^j =  n^j \times \frac{M}{\abs{\Obj}}\]
where $n^j$ is the number of object natures belonging to category $j$. Therefore we have
\begin{align*}
    \Bar{G}^{i\to j}_M &= M \left(\frac{1}{n^j} \sum_{o\in j} \gamma^i_o - \frac{1}{\abs{J}-1}\sum_{j'\neq j} \frac{1}{n^{j'}}\sum_{o\in j'} \gamma^i_o \right) \\
    &= M d^{i\to j}
\end{align*}
and since we work on $\Omega_{i,j}$ we have, where $o$ is the nature of object $m$:
\begin{align*}
    \abs{G^{i\to j}_M - \Bar{G}^{i\to j}_M} &\leq 2 \max_{j'\in J} \frac{\abs{\Obj}}{n^{j'}} \AAbs{ \sum_{m\text{ s.t. } o_m \in j'} (\widehat{\gamma^i_o} - \gamma^i_o) } \\
    & \leq  2  \abs{\Obj} \max_{j'\in J} \AAbs{ \sum_{m\text{ s.t. } o_m \in j'} \left(\frac{\Pi^i_{m,T_{\text{min}}}}{T_{\text{min}}} - \gamma^i_o\right) } \\
   & \leq 2 \sqrt{\frac{8 M \abs{\Obj}}{3 T_{\text{min}}} \sum_{o\in j} \gamma^i_o \log(2\abs{I}\abs{J}\alpha^{-1})} \\
   &\leq 2 \abs{\Obj} \sqrt{\frac{8 M}{3 T_{\text{min}}} \norm{\gamma}_{\infty} \log(2\abs{I}\abs{J}\alpha^{-1})}
\end{align*}
where we bound $\frac{1}{n^{j'}}$ by $1$ to go from the first line to the second, and we bound $ \sum_{o\in j} \gamma^i_o $ by $\abs{\Obj}\norm{\gamma}_{\infty}$ to go from the third line to the fourth.
\newline

$\bullet$ Study of the difference $\norm{\Bar{w}^{j}_{M+1} - w^{j}_{M+1}}_2$.
\newline

According to \cite{gao2018properties} (Proposition $4$), the softmax function defined on $\R^d$ with temperature $\eta$ is $\eta$-Lipschitz, so we have
\begin{align*}
    \norm{w^{j}_{M+1} - \Bar{w}^{j}_{M+1}}_2 &\leq \eta \sqrt{\abs{I}} \times  2 \abs{\Obj} \sqrt{\frac{8 M}{3 T_{\text{min}}} \norm{\gamma}_{\infty} \log(2\abs{I}\abs{J}\alpha^{-1})} \\
    &=  2 \abs{\Obj}\eta_0 \sqrt{\frac{8 \abs{I}}{3 T_{\text{min}}} \norm{\gamma}_{\infty} \log(2\abs{I}\abs{J}\alpha^{-1})} 
\end{align*}
by replacing $\eta$ by its value $\eta_0 M^{-1/2}$.
\newline

$\bullet$ Study of the difference $\norm{\Bar{w}^{j}_{M+1} - w^{j}_\infty}_2$.
\newline

We apply Proposition $30$ of \cite{jaffard2024chani}. For $j\in J$, if $I^j = I$ then for every $i\in I$
\[
\Bar{w}^{i\to j}_{M+1} = w^{i\to j}_\infty.
\]
Otherwise if \(I^j\neq I\), we have
\[
\abs{w^{i\to j}_\infty - \Bar{w}^{i\to j}_{M+1}} \leq \frac{1}{\abs{I^j}}\max\left(1, \frac{\abs{I} - \abs{I^j}}{\abs{I^j}} \right)\exp(- \eta_0\delta^j \sqrt{M})
\]
and 
\[
\norm{\Bar{w}^{j}_{M+1} - w^{j}_\infty}_2 \leq  \frac{\sqrt{\abs{I}}}{\abs{I^j}}\max\left(1, \frac{\abs{I} - \abs{I^j}}{\abs{I^j}} \right)\exp(- \eta_0\delta^j \sqrt{M}).
\]
We get the result by combining the error terms.

\subsection{Proof of Theorem~\ref{th correct classif hawkes}}  \label{proof th correct classif hawkes}

We will need several lemmas. Let $\Lambda^{j}_{m,t} \coloneqq \int_0^t \lambda^{j}_{m,s}ds $ \label{def:lambdajm}
be the compensator of the process $(N^{j}_{m,t})_{t\geq 0}$. We denote $\Bar{\Lambda}^j_{m,t} \coloneqq \mathbb{E}[\Lambda^j_{m,t} 
\mid \mathcal{F}_{m-1}]$.

\begin{lemma} \label{lemma control compens}
    Let $x > 0$, $m\in \{1,\dots,M+1\}$, $j\in J$, $t\geq 0$. Then we have
    \begin{equation} \label{eq lamb - e}
     \mathbb{P}\left( \Lambda^j_{m,t} - \Bar{\Lambda}^j_{m,t} \geq \norm{g}_1 \frac{x}{3}  +  \sqrt{2x\norm{g}_1 \Bar{\Lambda}^j_{m,t} }  \mid \mathcal{F}_{m-1}\right) \leq e^{-x}
\end{equation}
and
\begin{equation} \label{eq e - lamb}
     \mathbb{P}\left( \Bar{\Lambda}^j_{m,t} - \Lambda^j_{m,t} \geq \norm{g}_1 \frac{x}{3}  +  \sqrt{2x\norm{g}_1 \Bar{\Lambda}^j_{m,t} }  \mid \mathcal{F}_{m-1}\right) \leq e^{-x}.
\end{equation}
\end{lemma}

\begin{proof}
    Let $x>0$. Let us use the Cramér-Chernoff method. Note that the weights $w^j_m$ are $\mathcal{F}_{m-1}-$measurable. Let $h_t(u) \coloneqq \int_u^t g(s-u) ds$, $\xi >0$. 
    \label{def:htu}Then according to Fubini-Tonelli theorem, since $g$ is positive and the random measures defined by the processes $(\Pi^i_{m,t})_{t\geq 0}$ are $\sigma$-finite we have almost surely
\begin{align*}
    \Lambda^j_{m,t} &=\sum_{i\in I} w^{i\to j}_m \int_{s=0}^t \int_{u = 0}^s g(s-u)d\Pi^i_{m,u}ds \\
    &= \sum_{i\in I} w^{i\to j}_m \int_{u=0}^t\int_{s=u}^t g(s-u) ds d\Pi^i_{m,u} \\
    &= \sum_{i\in I} w^{i\to j}_m \int_{u=0}^t h_t(u) d\Pi^i_{m,u}.
\end{align*}
Therefore
\[
\mathbb{E}[e^{\xi \Lambda^j_{m,t}}\mid \mathcal{F}_{m-1}] = \mathbb{E}\left[ \exp\left( \sum_{i\in I} w^{i\to j}_m \xi \int_{u=0}^t h_t(u) d\Pi^i_{m,u} \right)\mid \mathcal{F}_{m-1}\right]
\]
and by independence of the processes $((\Pi^i_{m,t})_{t\geq 0})_{i\in I}$ we have
\[
\mathbb{E}[e^{\xi \Lambda^j_{m,t}}\mid \mathcal{F}_{m-1}] = \prod_{i\in I}\mathbb{E}\left[ \exp\left( w^{i\to j}_m \xi \int_{u=0}^t h_t(u) d\Pi^i_{m,u} \right)\mid \mathcal{F}_{m-1}\right].
\]
Besides, according to Proposition $6$ of \cite{reynaud2003adaptive}, 
 \begin{align*}
    \mathbb{E}[e^{\xi \Lambda^j_{m,t}} \mid \mathcal{F}_{m-1}] &= \prod_{i\in I} \exp\left( \gamma^i_o \int_0^t(e^{\xi w^{i\to j}_m h_t(u)}- 1)du \right) \\
    &=  \exp\left( \sum_{i\in I} \gamma^i_o \int_0^t(e^{\xi w^{i\to j}_m h_t(u)}- 1)du \right) \\
    &=  \exp\left( \sum_{i\in I} \gamma^i_o \int_0^t \sum_{k=1}^\infty \frac{(\xi w^{i\to j}_m h_t(u))^k}{k!} du \right) \\
    & = \exp(\xi \Bar{\Lambda}^j_{m,t})\exp\left( \sum_{i\in I} \gamma^i_o \int_0^t \sum_{k=2}^\infty \frac{(\xi w^{i\to j}_m h_t(u))^k}{k!} du \right) \\
    & \leq \exp(\xi \Bar{\Lambda}^j_{m,t})\exp\left( \Bar{\lambda}^j_m \mathbb{G}(t)\sum_{k=2}^\infty \frac{\xi^k\norm{g}_1^{k-1}}{k!}  \right) 
\end{align*}
where we exchanged the sum and integral because the series of functions converges uniformly, we bounded $h_t(u)$ by $\norm{g}_1$ and $(w^{i\to j}_m)^k$ by $w^{i\to j}_m$ (because $w^{i\to j}_m$ is bounded by $1$). 
\newline

Suppose $\xi \norm{g}_1 < 1$. Since $\Bar{\lambda}^j_m \mathbb{G}(t) = \Bar{\Lambda}^j_{m,t}$ we have
\begin{align*}
    \mathbb{E}[e^{\xi(\Lambda^j_{m,t} - \Bar{\Lambda}^j_{m,t})}\mid \mathcal{F}_{m-1}] &\leq \exp\left(  \Bar{\Lambda}^j_{m,t} \frac{\xi^2\norm{g}_1}{2}\sum_{k=1}^\infty \frac{(\xi \norm{g}_1)^{k-2}}{3^{k-2}}\right) \\
    &= \exp\left(   \Bar{\Lambda}^j_{m,t} \frac{\xi^2\norm{g}_1}{2} \frac{1}{1 - \frac{\xi \norm{g}_1}{3}}\right) .
\end{align*}
By applying Markov's inequality to the variable $e^{\xi(\Lambda^j_{m,t} - \Bar{\Lambda}^j_{m,t})}$ conditionally to the $\sigma$-algebra $\mathcal{F}_{m-1}$, we get
\begin{align*}
    \mathbb{P}\left( \Lambda^j_{m,t} - \Bar{\Lambda}^j_{m,t} \geq x \mid \mathcal{F}_{m-1}\right) &\leq e^{-\xi x} \exp\left(   \Bar{\Lambda}^j_{m,t} \frac{\xi^2\norm{g}_1}{2} \frac{1}{1 - \frac{\xi \norm{g}_1}{3}}\right)\\
    &= \exp\left( -\xi x +  \frac{ \Bar{\Lambda}^j_{m,t} \xi^2\norm{g}_1
    }{2\left(1 - \frac{\xi\norm{g}_1}{3}\right)}\right).
\end{align*}
Then similarly as in section $2.4$ of \cite{massartconcentration}, we have
\[
\sup_{\xi \in (0,3\norm{g}_1^{-1})} \left\{\xi x -  \frac{\Bar{\Lambda}^j_{m,t} \xi^2\norm{g}_1}{2\left(1 - \frac{\xi\norm{g}_1}{3}\right)} \right\} = \frac{9}{\norm{g}_1} \Bar{\Lambda}^j_{m,t} f\left( \frac{x}{3  \Bar{\Lambda}^j_{m,t}  } \right)
\]
where $f : v \mapsto 1 + v - \sqrt{1+2v}$. Therefore
\begin{align*}
     &\mathbb{P}\left( \Lambda^j_{m,t} - \Bar{\Lambda}^j_{m,t}
     \geq x \mid \mathcal{F}_{m-1}\right) \\
     & \leq \exp\left( -\frac{9}{\norm{g}_1} \Bar{\Lambda}^j_{m,t} f\left( \frac{x}{3  \Bar{\Lambda}^j_{m,t}  } \right) \right).
\end{align*}

Besides, $f$ is an increasing function from $\R^*_+$ to $\R^*_+$ with inverse $f^{-1} : v \mapsto v + \sqrt{2v}$ so we have
\[
 \mathbb{P}\left( \Lambda^j_{m,t} - \Bar{\Lambda}^j_{m,t} \geq 3 \Bar{\Lambda}^j_{m,t} f^{-1}\left(\frac{\norm{g}_1 x}{9\Bar{\Lambda}^j_{m,t}}\right) \mid \mathcal{F}_{m-1}\right) \leq e^{-x}
\]
and by integrating we get \eqref{eq lamb - e}. Now, let us prove \eqref{eq e - lamb}. Let us consider $-\Lambda^j_{m,t}$ instead of $\Lambda^j_{m,t}$. Then similarly we have 
\[
\mathbb{E}[e^{ -\xi \Lambda^j_{m,t}}\mid \mathcal{F}_{m-1}] \leq  
\exp(-\xi \Bar{\Lambda}^j_{m,t})\exp\left( \Bar{\lambda}^j_m \mathbb{G}(t)\sum_{k=2}^\infty \frac{\xi^k\norm{g}_1^{k-1}}{k!} du \right) 
\]
hence
\[
\mathbb{E}[e^{ \xi( \Bar{\Lambda}^j_{m,t} - \Lambda^j_{m,t} )}\mid \mathcal{F}_{m-1}] \leq \exp\left( \Bar{\lambda}^j_m \mathbb{G}(t)\sum_{k=2}^\infty \frac{\xi^k\norm{g}_1^{k-1}}{k!} du \right) 
\]
so the same computations apply to bound $\mathbb{E}[e^{ \xi( \Bar{\Lambda}^j_{m,t} - \Lambda^j_{m,t} )}\mid \mathcal{F}_{m-1}]$, and by applying Markov's inequality to the variable $e^{ \xi( \Bar{\Lambda}^j_{m,t} - \Lambda^j_{m,t} )}$ we get the same bound as in \eqref{eq lamb - e} for \eqref{eq e - lamb}.

Now that we proved \eqref{eq lamb - e} and \eqref{eq e - lamb}, we can combine them to have the result.
\end{proof}

\begin{lemma} \label{lemma control hawkes}
 Let $x > 0$, $m\in \{1,\dots,M+1\}$, $j\in J$, $t\geq 0$. Let $e(x,t)\coloneqq\norm{g}_1 \frac{x}{3}  +  \sqrt{2x\norm{g}_1 \Bar{\Lambda}^j_{m,t} }$. Then we have
\begin{equation} \label{eq concent N - e}
    \mathbb{P}\left( 
 N^j_{m,t} -  \Bar{\Lambda}^j_{m,t} \leq \sqrt{2( \Bar{\Lambda}^j_{m,t} + e(x,t))x} + \frac{x}{3} + e(x,t)
\right) \geq 1 - 2 e^{-x}
\end{equation}
and
\begin{equation} \label{eq concent e - N}
    \mathbb{P}\left(  \Bar{\Lambda}^j_{m,t} - N^j_{m,t}  \leq \sqrt{2( \Bar{\Lambda}^j_{m,t} + e(x,t)) x} + \frac{x}{3} + e(x,t) \right) \geq 1 - 3e^{-x}.
\end{equation}
\end{lemma}

\begin{proof}
Let $x>0$, $j\in J$, $t\geq 0$. Let us prove \eqref{eq concent N - e}. Let us work conditionally to $\mathcal{F}_{m-1}$. Let us use inequality $3.3$ of \cite{hansen2015lasso} with $N = (N^l_{m,t})_{l\in I\cup J}$, $H = (H^l)_{l\in I\cup J}$ with $H^j \equiv 1$ and $H^l \equiv 0$ if $l\neq j$, and $B = 1$. Let $\rho \in \R$ to choose later. Then we have
\[
\mathbb{P}\left( N^j_{m,t} - \Lambda^j_{m,t} \geq \sqrt{2\rho x} + \frac{x}{3} \ \text{and} \ \Lambda^j_{m,t} \leq \rho \mid \mathcal{F}_{m-1}\right) \leq e^{-x}.
\]
Let us choose $\rho \coloneqq \Bar{\Lambda}^j_{m,t} + e(x,t)$ which is $\mathcal{F}_{m-1}$-measurable. Let $\Omega_1$ be the event
\[
\Omega_1 \coloneqq \left\{ N^j_{m,t} - \Lambda^j_{m,t} \geq \sqrt{2\rho x} + \frac{x}{3} \right\}
\]
and $\Omega_2$ be the event
\[
\Omega_2 \coloneqq \left\{ \Lambda^j_{m,t} \leq \rho\right\}.
\]
Then $\mathbb{P}(\Omega_1 \cap \Omega_2^c)\leq e^{-x}$ and according to Lemma~\ref{lemma control compens} inequality \eqref{eq lamb - e}, $\mathbb{P}(\Omega_2^c)\leq e^{-x}$. Besides,
\[
\Omega_1^c \cap \Omega_2 \subset \left\{ N^j_{m,t} -  \Bar{\Lambda}^j_{m,t} \leq \sqrt{2\rho x} + \frac{x}{3} + e(x,t) \right\}
\]
and by integrating we get
\[
\mathbb{P}(\Omega_1 \cup \Omega_2^c) = \mathbb{P}(\Omega_1 \cap \Omega_2) + \mathbb{P}(\Omega_2^c) \leq 2e^{-x}
\]
therefore
\[
\mathbb{P}\left( 
 N^j_{m,t} -  \Bar{\Lambda}^j_{m,t} \leq \sqrt{2\rho x} + \frac{x}{3} + e(x,t)
\right) \geq \mathbb{P}(\Omega_1^c \cap \Omega_2) \geq 1 - 2 e^{-x}
\]
By replacing $\rho$ by its value we get \eqref{eq concent N - e}.
\newline

Now, let us prove \eqref{eq concent e - N}. Similarly, let us use inequality $3.3$ of \cite{hansen2015lasso} with $H^j \equiv -1$ and $H^l \equiv 0$ if $l\neq j$. Let $\rho\in \R$ to choose later. Then we get 
\[
\mathbb{P}\left(  \Lambda^j_{m,t} - N^j_{m,t} \geq \sqrt{2\rho x} + \frac{x}{3} \ \text{and} \ \Lambda^j_{m,t} \leq \rho \mid \mathcal{F}_{m-1}\right)\leq e^{-x}.
\]
Let us choose the same $\rho$ as before: $\rho \coloneqq \Bar{\Lambda}^j_{m,t} + e(x,t)$. Let $\Omega_3$ be the event
\[
\Omega_3 \coloneqq \left\{ \Lambda^j_{m,t} - N^j_{m,t} \geq \sqrt{2\rho x} + \frac{x}{3} \right\}
\]
and $\Omega_4$ be the event
\[
\Omega_4 \coloneqq \left\{ \Bar{\Lambda}^j_{m,t} -\Lambda^j_{m,t} \geq e(x,t)\right\}.
\]
Then $\mathbb{P}(\Omega_3 \cap \Omega_2)\leq e^{-x}$ and according to Lemma \eqref{lemma control compens} inequality \eqref{eq e - lamb}, $\mathbb{P}(\Omega_4^c)\leq e^{-x}$. Besides,
\[
\Omega_3^c\cap \Omega_2 \cap \Omega_4^c \subset \Omega_3^c\cap \Omega_4^c \subset  \left\{   \Bar{\Lambda}^j_{m,t} - N^j_{m,t}  \leq \sqrt{2\rho x} + \frac{x}{3} + e(x,t) \right\}
\]
and by integrating we have
\[
\mathbb{P}(\Omega_3\cup \Omega_2^c \cup \Omega_4) \leq \mathbb{P}(\Omega_4) + \mathbb{P}(\Omega_3 \cap \Omega_2) + \mathbb{P}(\Omega_2^c) \leq 3e^{-x}
\]
so we have
\[
\mathbb{P}\left(  \Bar{\Lambda}^j_{m,t} - N^j_{m,t}  \leq \sqrt{2\rho x} + \frac{x}{3} + e(x,t) \right) \geq 1 - 3e^{-x}.
\]
By replacing $\rho$ by its value we get \eqref{eq concent e - N}.
\end{proof}

\begin{lemma} \label{lemma control hawkes avec box}
    Let $m\in \{1,\dots,M+1\}$, $t\geq 0$, $\alpha>0$. Then with probability more than $1-\alpha$, we have
        \begin{equation*} 
\forall j\in J, \quad \AAbs{ N^j_{m,t} - \Bar{\Lambda}^j_{m,t} } \leq \Box \left(
\sqrt{\Bar{\Lambda}^j_{m,t}(\norm{g}_1 \vee 1) \log\left(\frac{\abs{J}}{\alpha}\right)} + (\norm{g}_1 \vee 1) \log\left(\frac{\abs{J}}{\alpha}\right)
\right).
\end{equation*}

\end{lemma}

\begin{proof}
    By combining inequalities \eqref{eq concent N - e} and \eqref{eq concent e - N} of Lemma~\ref{lemma control hawkes}, we get
\begin{equation*} 
    \mathbb{P}\left( 
 \AAbs{N^j_{m,t} -  \Bar{\Lambda}^j_{m,t}} \leq \sqrt{2( \Bar{\Lambda}^j_{m,t} + e(x,t))x} + \frac{x}{3} + e(x,t)
\right) \geq 1 - 5 e^{-x}.
\end{equation*}
Let us choose $x_\alpha = \log(5\abs{J}\alpha^{-1})$, such that by union bound
\begin{equation*} 
    \mathbb{P}\left( \forall j\in J,
 \AAbs{N^j_{m,t} -  \Bar{\Lambda}^j_{m,t}} \leq \sqrt{2( \Bar{\Lambda}^j_{m,t} + e(x_\alpha,t))x_\alpha} + \frac{x_\alpha}{3} + e(x_\alpha,t)
\right) \geq 1 - \alpha
\end{equation*}
which concludes the proof.

\end{proof}

\begin{lemma} \label{lemma control compens limite}
    Let $\alpha>0$, $t\in [T_{\text{min}},T]$ and $\eta = \frac{\eta_0}{\sqrt{M}}$ where $\eta_0>0$. Suppose Assumption~\ref{assump nb obj} holds. Let $\Bar{\Lambda}^j_{\infty,t} \coloneqq \sum_{i\in I} w^{i\to j}_\infty \gamma^i_o  \mathbb{G}(t)$. 
    Then there exists an absolute constant $\Box > 0$ such that with probability more than $1-\alpha$, for every $j\in J$ we have
    \[
     \AAbs{\Bar{\Lambda}^j_{M+1,t} - \Bar{\Lambda}^j_{\infty,t}} \leq \Box \Bar{B}(\alpha)\sqrt{\abs{I}} \norm{\gamma}_{\infty} \mathbb{G}(t)
    \]
    where
    \[
   \Bar{B}(\alpha)\coloneqq 
         \abs{\Obj}\eta_0 \sqrt{\frac{\abs{I}}{ T_{\text{min}}} \norm{\gamma}_{\infty} \log\left(\frac{\abs{I}\abs{J}}{\alpha}\right)}
         + \frac{\abs{I}^{3/2}}{\abs{I}_{\text{min}}} e^{-\eta_0 \delta_{\text{min}}\sqrt{M}}
    \]
    with $\abs{I}_{\text{min}} \coloneqq \min_{j\in J} \abs{I^j}$ and $\delta_{\text{min}} \coloneqq \min_{j\in J} \delta^j$.
\end{lemma}

\begin{proof}
    Let $\alpha>0$, $t\in [T_{\text{min}},T]$, $j\in J$. Then according to Proposition~\ref{prop conv weights}, there exists an event $\Omega$ with probability more than $1-\alpha$ such that on $\Omega$, for every $j\in J$ we have 
 \begin{align*}
        \Norm{w^j_{M+1} - w^j_\infty}_2 &\leq  2 \abs{\Obj}\eta_0 \sqrt{\frac{8 \abs{I}}{3 T_{\text{min}}} \norm{\gamma}_{\infty} \log(2\abs{I}\abs{J}\alpha^{-1})} \\
        & \hspace{1cm} + \mathbb{1}_{\{I^j \subsetneq I\}} \frac{\sqrt{\abs{I}}}{\abs{I^j}}\max\left(1, \frac{\abs{I} - \abs{I^j}}{\abs{I^j}} \right)\exp(- \eta_0\delta^j \sqrt{M}) \\
         &\leq \Box \left(
         \abs{\Obj}\eta_0 \sqrt{\frac{\abs{I}}{ T_{\text{min}}} \norm{\gamma}_{\infty} \log\left(\frac{\abs{I}\abs{J}}{\alpha}\right)}
         + \frac{\abs{I}^{3/2}}{\abs{I}_{\text{min}}} e^{-\eta_0 \delta_{\text{min}}\sqrt{M}}
         \right) \\
         &= \Box \Bar{B}(\alpha)
    \end{align*}
Let us work on $\Omega$. We have
\begin{align*}
    \AAbs{\Bar{\Lambda}^j_{M+1,t} - \Bar{\Lambda}^j_{\infty,t}} 
    &= \AAbs{\sum_{i\in I} (w^{i\to j}_{M+1} - w^{i\to j}_\infty) \gamma^i_o  \mathbb{G}(t)} \\
    &\leq \norm{w^j_{M+1} - w^j_\infty}_2 \sqrt{\abs{I}} \norm{\gamma}_{\infty} \mathbb{G}(t)
\end{align*}
which concludes the proof.
\end{proof}

\begin{lemma} \label{lemma} \label{lemma control hawkes limit}
Let $\alpha >0$ and $\eta = \frac{\eta_0}{\sqrt{M}}$ where $\eta_0>0$. Then with probability more than $1-\alpha$, for every $j\in J$ we have
\begin{align*}
    &\AAbs{ N^j_{M+1,t}- \Bar{\Lambda}^j_{\infty,t}} \leq \Theta^j_t
\end{align*}
where
\begin{align*}
    \Theta^j_t \coloneqq  D \left(
 \Bar{B}\left(\frac{\alpha}{2}\right)\sqrt{\abs{I}} \norm{\gamma}_{\infty} \mathbb{G}(t) + 
\sqrt{\Bar{\Lambda}^j_{m,t}(\norm{g}_1 \vee 1) \log\left(\frac{2\abs{J}}{\alpha}\right)} \right.\\
\left. + (\norm{g}_1 \vee 1) \log\left(\frac{2\abs{J}}{\alpha}\right)
\right)
\end{align*}
where $D>0$ is an absolute constant.
\end{lemma}

\begin{proof}
    Let $\alpha>0$, $t\in [T_{\text{min}},T]$. By combining the results of Lemmas~\ref{lemma control hawkes avec box} and~\ref{lemma control compens limite} for every $j\in J$ with $\frac{\alpha}{2}$ instead of $\alpha$ we have the result.
\end{proof}

Now we can prove the theorem. Let $\alpha\in (0,1)$, $t_0\in[T_{\text{min}},T]$ to choose later. Let us work on the event defined in the proof of Lemma~\ref{lemma control hawkes limit} for $t=t_0$. Let
\[
I(t_0) \coloneqq \Bar{\Lambda}^{j^*}_{\infty,t_0} - \Theta^{j^*}_{t_0},
\]
and for $j\in J$,
\[
S^j(t_0) \coloneqq \Bar{\Lambda}^j_{\infty,t_0} + \Theta^j_{t_0}.
\]
such that 
\[
N^{j^*}_{M+1,t_0} > I(t_0)
\]
and for every $j\in J$,
\[
N^j_{M+1,t_0} < S^j(t_0).
\]
We would like to choose $t_0$ such that for every $j\neq j^*$, $S^j(t_0) \leq I(t_0)$ and $I(t_0) = \thr$. It would imply that at time $t_0$, $N^{j^*}_{M+1}$ has already reached $\thr$ and for every $j\neq j^*$, $N^j_{M+1}$ has not reached $\thr$ yet. The condition $S^j(t_0) \leq I(t_0)$ is implied by
\[
\Bar{\Lambda}^{j^*}_{\infty,t_0} - \Bar{\Lambda}^j_{\infty,t_0} \geq \Theta^{j^*}_{t_0} + \Theta^{j}_{t_0}
\]
which is implied by
\begin{align*}
    &\Dlamb \mathbb{G}(t_0) \geq 2 \Theta^{j^*}_{t_0}.
\end{align*}
Let $R \coloneqq \Dlamb - 2D\Bar{B}\left(\frac{\alpha}{2}\right)\sqrt{\abs{I}} \norm{\gamma}_{\infty}$. Suppose $R \geq \frac{\Dlamb}{2}$. Then a sufficient condition is
\[
R\mathbb{G}(t_0) \geq 2D \left(\sqrt{\Bar{\lambda}^{j^*}_{M+1}\mathbb{G}(t_0)(\norm{g}_1 \vee 1) \log\left(\frac{2\abs{J}}{\alpha}\right)} + (\norm{g}_1 \vee 1) \log\left(\frac{2\abs{J}}{\alpha}\right)
\right).
\]
This inequality has the form $aX^2 - bX - c \geq 0$ with $X = \sqrt{\mathbb{G}(t_0)}$, and $a,b,c$ positive constants. The polynomial has two real roots, the largest one being $\frac{b+\sqrt{b^2 + 4ac}}{2a}$, so a sufficient condition is $X \geq \frac{b}{a} + \sqrt{\frac{c}{a}}$, \ie
\begin{align} \label{eq G(t_0)}
   \mathbb{G}(t_0) \geq \Box  (\norm{g}_1 \vee 1) \log\left(\frac{2\abs{J}}{\alpha}\right) \left(
   R^{-1/2} + \norm{\gamma}_{\infty} R^{-1}
   \right)^2.
\end{align}
Let $t_-$  the minimal $t\geq T_{\text{min}}$ satisfying \eqref{eq G(t_0)}, such that
\[
  \mathbb{G}(t_-) = \left( \Box  (\norm{g}_1 \vee 1) \log\left(\frac{2\abs{J}}{\alpha}\right) \left(
   R^{-1/2} + \norm{\gamma}_{\infty} R^{-1}\right)^2\right) \vee \mathbb{G}(T_{\text{min}}).
\]
Since $\mathbb{G}$ tends to infinity, it is possible to choose such $t_-$. Since $\mathbb{G}$ is increasing, a sufficient condition to have $t_-< T$ is $\mathbb{G}(t_-) < \mathbb{G}(T)$. Then if $\thr \in [I(t_-),I(T))$, by continuity of $I$ there exists $t_0\in [t_-,T)$ such that $I(t_0)=\thr$, and we have the result. Let us now explicit the conditions we had to assume.
\newline

We have
\begin{align*}
   & R \geq \frac{\Dlamb}{2} \\
   \Leftrightarrow \quad  &\Dlamb - 2D\Bar{B}\left(\frac{\alpha}{2}\right)\sqrt{\abs{I}} \norm{\gamma}_{\infty}\geq \frac{\Dlamb}{2}  \\
    \Leftarrow \quad  & \Bar{B}\left(\frac{\alpha}{2}\right) \leq \Box \frac{\Dlamb}{\sqrt{\abs{I}}\norm{\gamma}_\infty} \\
     \Leftrightarrow \quad & \abs{\Obj}\eta_0 \sqrt{\frac{\abs{I}}{ T_{\text{min}}} \norm{\gamma}_{\infty} \log\left(\frac{2\abs{I}\abs{J}}{\alpha}\right)}
         + \frac{\abs{I}^{3/2}}{\abs{I}_{\text{min}}} e^{-\eta_0 \delta_{\text{min}}\sqrt{M}} \leq \Box \frac{\Dlamb}{\sqrt{\abs{I}}\norm{\gamma}_\infty}
\end{align*}
A sufficient condition is that both terms of the left side of the inequality are less than half the right term, which reads as
\[
T_{\text{min}} \geq \Box \Dlamb^{-2}\abs{\Obj}^2 \eta_0^2 \abs{I}^2\norm{\gamma}_\infty^3 \log\left(\frac{\abs{I}\abs{J}}{\alpha}\right)
\]
and
\[
M\geq \Box \frac{1}{\eta_0^2 \delta_{\text{min}}^2} \log\left( \frac{\abs{I}}{\abs{I}_{\text{min}}} \right)^2 + \log\left(\frac{\norm{\gamma}_\infty}{\Dlamb}\right)^2.
\]
Besides,
\begin{align*}
   & \thr \leq I(T) \\
   \Leftrightarrow \quad & \thr \leq \left(\Bar{\lambda}^j_{\infty} - D \Bar{B}\left(\frac{\alpha}{2}\right)\sqrt{\abs{I}} \norm{\gamma}_{\infty}\right) \mathbb{G}(T) \\
  & \hspace{2cm} - D
\sqrt{\Bar{\lambda}^{j^*}_{\infty}\mathbb{G}(T)(\norm{g}_1 \vee 1) \log\left(\frac{2\abs{J}}{\alpha}\right)} - D (\norm{g}_1 \vee 1) \log\left(\frac{2\abs{J}}{\alpha}\right)
\end{align*}
where $\Bar{\lambda}^{j^*}_{\infty} \coloneqq \sum_{i\in I}w^{i\to j^*}_\infty \gamma^i_o$. Since $\Bar{\lambda}^{j^*}_{\infty} \geq \Dlamb$ and $R\geq \frac{\Dlamb}{2}$, a sufficient condition is
\begin{align*}
   \thr \leq \Box \Dlamb \mathbb{G}(T) - \Box_{\norm{\gamma}_\infty,\norm{g}_1} \left(\sqrt{ \mathbb{G}(T)\log(\abs{J}\alpha^{-1})} + \log(\abs{J}\alpha^{-1})\right) .
\end{align*}
The other condition on $\thr$ is
\begin{align*}
   & \thr \geq I(t_-) \\
  \Leftarrow \quad &  \thr \geq \norm{\gamma}_\infty \mathbb{G}(t_-) \\
  \Leftrightarrow \quad &\thr \geq \left( \Box  (\norm{g}_1 \vee 1) \log\left(\frac{2\abs{J}}{\alpha}\right) \left(
   R^{-1/2} + \norm{\gamma}_{\infty} R^{-1}\right)^2\right) \vee \norm{\gamma}_\infty\mathbb{G}(T_{\text{min}}) \\
  \Leftarrow \quad & \thr \geq \left( \Box_{\norm{g}_1} \log(\abs{J}\alpha^{-1}) \left(\frac{\norm{\gamma}_\infty}{\sqrt{\Dlamb}}\vee \frac{\norm{\gamma}_{\infty}^2}{\Dlamb}\right)\right) \vee  \norm{\gamma}_\infty \mathbb{G}(T_{\text{min}})
\end{align*}
and finally, since by definition $T> T_{\text{min}}$,
\begin{align*}
    & \mathbb{G}(t_-) < \mathbb{G}(T) \\
    \Leftarrow \quad & \Box_{\norm{g}_1} \log(\abs{J}\alpha^{-1}) \left(\frac{1}{\sqrt{\Dlamb}}\vee \frac{\norm{\gamma}_{\infty}}{\Dlamb}\right) < \mathbb{G}(T)
\end{align*}
which concludes the proof.

\subsection{Proof of Theorem~\ref{th coupling hawkes ddm}} \label{sec proof the coupling hawkes ddm}

Let $G : x \mapsto \int_0^x g(u)du$ and let $\Lambda^{j,n}_{m,t} \coloneqq \int_0^t \lambda^{j,n}_{m,s}ds $ be the compensator of the process $(N^{j,n}_{m,t})_{t\geq 0}$. We still denote $\Bar{\lambda}^j_m \coloneqq \sum_{i\in I} \gamma^i_o w^{i\to j}_m$. 

Let $x=(x_l)_{1\leq l\leq d},y = (y_l)_{1\leq l \leq d},z = (z_l)_{1\leq l \leq d}$ vectors in $\R^d$. We denote by $(x,y) \coloneqq \sum_{l=1}^d x_l y_l$ the usual scalar product between two vectors, which we extend to three vectors by $(x,y,z) \coloneqq  \sum_{l=1}^d x_l y_l z_l$. 

 We will need the following lemmas.

\begin{lemma}
\label{controlLambda}
Let $x>0$. With the constants of Lemma~\ref{original_coupling}, if $g$ has compact support on $[0,S]$, then 
there exists i.i.d. standard Brownian motions $(B^i_t)_{t\geq 0}$'s independent of everything else such that with probability larger than $1-|I| be^{-dx}$, for every $j\in J$ and $t\geq 0$,
\begin{align*}
   & |\Lambda^{j,n}_{m,t}- \E[\Lambda^{j,n}_{m,t}] - \sum_{i\in I}w^{i\to j}_m B^i_{n\gamma^i_o t}| \\
   & \leq S n \Bar{\lambda}^j_m  \|g\|_1 + \sum_{i\in I} w^{i\to j}_m |B^i_{n\gamma^i_o(t-S)}-B^i_{n\gamma^i_o t}| + \log(n\norm{\gamma}_{\infty}t\vee 1+1) (a+x),
\end{align*}
with $\norm{\gamma}_{\infty}=\max_{i\in I,o\in \Obj} \gamma^i_o$.
\end{lemma}

\begin{proof}
For all $t\geq 0$,
\begin{equation} \label{eq encadrement compens}
    n \Bar{\lambda}^j_m \|g\|_1 (t-S) \leq \E[\Lambda^{j,n}_{m,t}] \leq n \Bar{\lambda}^j_m\|g\|_1 t.
\end{equation}
Moreover
\[
\|g\|_1 \sum_{i\in I} w^{i\to j}_m \Pi^{i,n}_{m,t-S} \leq \Lambda^{j,n}_{m,t} \leq \|g\|_1 \sum_{i\in I} w^{i\to j}_m \Pi^{i,n}_{m,t}.
\]
The first part is just due to the fact that if $t-s\geq S$ then $G(t-s)=\|g\|_1$ and that $G(t-s)\leq \|g\|_1$ for all $t$ and $s$. The second part is done with exactly the same arguments.
So we get for all $t\geq 0$ that
\begin{align*}
     \sum_{i\in I} w^{i\to j}_m [\Pi^{i,n}_{m,t-S}-n\gamma^i_o (t-S)] -  S n \Bar{\lambda}^j_m \|g\|_1 &\\
      \leq \Lambda^{j,n}_{m,t}- &\E(\Lambda^{j,n}_{m,t}) \\
      & \leq \sum_{i\in I} w^{i\to j}_m [\Pi^{i,n}_{m,t}-n\gamma^i_o t] + S n \Bar{\lambda}^j_m \|g\|_1.
\end{align*}
Now we apply Lemma~\ref{original_coupling} to all $\tilde{\Pi}^i$ that are dilation of $(\Pi^{i,n}_{m,t})_{t\geq 0}$ so that they are Poisson processes of rate 1. It means we obtain i.i.d. standard Brownian motions $B^i$, only depending on $\tilde{\Pi}^i$ and auxiliary i.i.d. uniform random variables $V^i$ such that with probability larger $1-|I| b e^{-dx}$,
\[
\forall i \in I, \forall t \geq 0, \AAbs{\Pi^{i,n}_{m,t}-n\gamma^i_o t - B^i_{n\gamma^i_o t}}\leq \log(n \norm{\gamma}_{\infty}t\vee 1+1) (a+x).
\]

We apply this at time $t$ and $t-S$ in the above equation and this leads to the desired result.

\end{proof}

Now, let us choose $S \leq n^{-1/2}$.

\begin{proposition}
    \label{Precis_rebolledo}
    Let $x>0$, $m\geq 1$, $o\in \Obj$ the nature of object $m$. There are absolute positive constants $b',d'$ and a positive $a'$ depending on $\norm{\gamma}_{\infty}$ and $\norm{g}_1$ such that
    there exists independent standard Brownian motions $(B^j_t)_{t\geq0}$ such that with probability larger than $1-(|I|+|J|) b'e^{-d'x}$
    \[
    \forall j\in J, \forall t\geq 0,  \quad |N^{j,n}_{m,t}-\Lambda^{j,n}_{m,t}-B^j_{n(w^j_l, \gamma^I_o) \norm{g}_1 t}| \leq a'n^{1/4}(t\vee1)^{1/4} (\log(nt+1)(x\vee 1))^{3/4}.
    \]
\end{proposition}

\begin{proof}

For all $j \in J$, let us consider $\pi^j$ independent Poisson processes of rate 1, independent of everything else. By Lemma~\ref{original_coupling}, we can construct i.i.d. standard Brownian motions $(B^j_t)_{t\geq0}$ such that with probability larger than $1-|J| be^{-dx}$, we have that
\[
\forall j\in J, \forall t\geq 0, |\pi^j_t-t-B_t^j| \leq \log(t \vee 1+1) (a+x).
\]
Since the (random) function $\Lambda^{j,n}_{m,\cdot}$'s are non decreasing, we can always define their inverse $(\Lambda^{j,n}_{m,\cdot})^{-1}$, conditionally to the inputs $\Pi^{i,n}_{m}$'s,  so that one can construct $N^{j,n}_m$, the Hawkes processes as $\pi^j_{\Lambda^{j,n}_{m,t}}$.
Hence we have by change of time, that on the same event of probability larger than $1-|J| be^{-dx},$
\begin{equation}
    \label{start}
\forall j\in J, \forall t\geq 0, \AAbs{N^{j,n}_{m,t}-\Lambda^{j,n}_{m,t}-B_{\Lambda^{j,n}_{m,t}}^j} \leq \log(\Lambda^{j,n}_{m,t} \vee 1+1) (a+x).
\end{equation}

Now we can apply Lemma~\ref{controlLambda}, so that with probability larger than $1-(|I|+|J|) b e^{-dx}$, both the above equations and the result of Lemma~\ref{controlLambda} holds.

Now let us intersect with the event where the results of Lemma~\ref{Control_brownian_2times} hold for all $i\in I$ and $j \in J$.

By a union bound, we get therefore that on an event of probability larger than $1- (|I|+|J|) (a+c') e^{- d\wedge c" x},$
\begin{itemize}
    \item with $B^i$ at time $n\gamma^i_o t$ and $u=0$, 
    \[
    |B^i_{n\gamma^i_o t}| \leq c \sqrt{(n\gamma^i_ot+2) (\log(n\gamma^i_ot+1)+2x)} = \square_{\norm{\gamma}_\infty} \sqrt{nt(\log(nt+1)+x)}
    \]
    \item still with $B^i$ at time $n\gamma^i_o t$ and $n\gamma^i_o (t-S)$, 
    \begin{align}
         \AAbs{B^i_{n\gamma^i_o(t-S)}-B^i_{n\gamma^i_ot}} &\leq c \sqrt{(n\gamma^i_o S+2) (2\log(n\gamma^i_ot+1)+2x)}  \nonumber \\
         & = \square_{\norm{\gamma}_\infty} \sqrt{n^{1/2}(\log(nt+1)+x)} \label{eq control Bi - Bi}.
    \end{align}
\end{itemize}
Hence on the same event, 
\begin{equation*}
    \forall j\in J, \forall t \geq 0, |\Lambda^{j,n}_{m,t}-\E[\Lambda^{j,n}_{m,t}]|\leq \square_{\norm{\gamma}_\infty,\norm{g}_1} \sqrt{n(t\vee 1)\log(nt+1)(x\vee 1)}.
\end{equation*}
Besides, by \eqref{eq encadrement compens} we also have 
\[
\forall j\in J, \forall t \geq 0, \AAbs{\E[\Lambda^{j,n}_{m,t}] - n \Bar{\lambda}^j_m \norm{g}_1 t} 
\leq \Box_{\norm{\gamma}_\infty,\norm{g}_1} \sqrt{n} 
\]
so we have
\begin{equation}
    \label{diffLambdacontrol2}\forall j\in J, \forall t \geq 0, |\Lambda^{j,n}_{m,t}- n \Bar{\lambda}^j_m \norm{g}_1 t|\leq \square_{\norm{\gamma}_\infty,\norm{g}_1} \sqrt{n(t\vee 1)\log(nt+1)(x\vee 1)}.
\end{equation}

Going back to \eqref{start} and now using Lemma~\ref{Control_brownian_2times} with $j\in J$ at time $\Lambda^{j,n}_{m,t}$ and $n \Bar{\lambda}^j_m \norm{g}_1 t$, we therefore obtain that
\begin{align*}
    &\AAbs{B_{\Lambda^{j,n}_{m,t}}^j-B_{n\Bar{\lambda}^j_m \norm{g}_1 t}^j} \\
    &\leq c\sqrt{\left(|n \Bar{\lambda}^j_m \norm{g}_1 t-\Lambda^{j,n}_{m,t}|+2\right)\left(\log(n \Bar{\lambda}^j_m \norm{g}_1 t+1)+\log(\Lambda^{j,n}_{m,t}+1)+2x\right)}
\end{align*}
and because \eqref{diffLambdacontrol2} on the same event, we have $\Lambda^{j,n}_{m,t} \leq \square_{\gamma,\norm{g}_1} nt$. Moreover using \eqref{diffLambdacontrol2} again, we get
that on the same event,
\[
|B_{\Lambda^{j,n}_{m,t}}^j-B_{n \Bar{\lambda}^j_m \norm{g}_1 t}^j| \leq
 \square_{\norm{\gamma}_\infty,\norm{g}_1} (n (t\vee 1))^{1/4}(\log(nt+1)(x\vee 1))^{3/4}.
 \]
Injecting this in \eqref{start}, we obtain the result.
\end{proof}

Now we can prove the theorem. Let us work on the same event as the one of the proof of Proposition~\ref{Precis_rebolledo}, for which the conclusions of Lemma~\ref{controlLambda} also apply, which has probability more than $1-(\abs{I}+ \abs{J})b'e^{-d'x}$. On this event, according to Lemma~\ref{controlLambda} we have for every $j\in J$ and $t\geq 0$
\begin{align*}
     &\AAbs{\Lambda^{j,n}_{m,t}- \E(\Lambda^{j,n}_{m,t}) - \sum_{i\in I}w^{i\to j}_m B^i_{n \gamma^i_o t}} \\
     &\leq \norm{g}_1\norm{\gamma}_\infty n^{-1/2} + \sum_{i\in I} w^{i\to j}_m |B^i_{n\gamma^i_o(t-S)}-B^i_{n\gamma^i_ot}| + \Box_{\norm{\gamma}_\infty}\log((nt)\vee 1+1)(x\vee 1).
\end{align*}
Besides, according to \eqref{eq control Bi - Bi} on the same event we have for every $i\in I$
\[ |B^i_{n\gamma^i_o(t-S)}-B^i_{n\gamma^i_ot}| \leq  \square_{\norm{\gamma}_\infty} \sqrt{n^{1/2}(\log(nt+1)+x)}.
\]
Since $\sum_{i\in I} w^{i\to j}_m = 1$ we have then for every $j\in J$
\begin{align*}
     &\AAbs{\Lambda^{j,n}_{m,t}- \E(\Lambda^{j,n}_{m,t}) - \sum_{i\in I}w^{i\to j}_m B^i_{n \gamma^i_o t}} \leq \Box_{\norm{g}_1,\norm{\gamma}_\infty} \sqrt{n^{1/2}\log(nt +1)(x\vee 1)} .
\end{align*}
 Therefore, by combining this with the conclusions of Proposition~\ref{Precis_rebolledo}, we have that 
 on the same event, for every $j\in J$ and $t\geq 0$,
 \begin{align*}
   & \AAbs{N^{j,n}_{m,t} - \E[\Lambda^{j,n}_{m,t}] - \sum_{i\in I}w^{i\to j}_m B^i_{n \gamma^i_o t} -
 B^j_{ n \norm{g}_1 t\Bar{\lambda}^j_m}} \\
 & \hspace{3cm}\leq \Box_{\norm{g}_1,\norm{\gamma}_\infty} n^{1/4}(t\vee 1)^{1/4} \left(\log(nt+1)(x\vee 1)\right)^{3/4}.
 \end{align*}
By construction, the processes $((B^l_u)_{u\geq 0})_{l\in I\cup J}$ are mutually independent.
 We have then the following equality in distribution:
\begin{align*}
    \sum_{i\in I}w^{i\to j}_m B^i_{n \gamma^i_o t} + B^j_{n\Bar{\lambda}^j_m \norm{g}_1 t} 
    & \overset{\mathcal{D}}{=}  \sum_{i\in I}w^{i\to j}_m \sqrt{n \gamma^i_o}B^i_{ t} + 
    \sqrt{n \Bar{\lambda}^j_m\norm{g}_1}B^j_{t}
\end{align*}
Therefore, the process $W^{j,n}_{m,t} \coloneqq 
\E[\Lambda^{j,n}_{m,t}] + \sum_{i\in I}w^{i\to j}_m B^i_{n \gamma^i_o t} + 
B^j_{n \Bar{\lambda}^j_m\norm{g}_1 t}$ is a Gaussian process with mean
\[
\E[W^{j,n}_{m,t}] = \E[\Lambda^{j,n}_{m,t}]
\]
and covariance
\begin{align*}
    \cov(W^{j,n}_{m,t}, W^{j,n}_{m,s}) &= \sum_{i\in I} (w^{i\to j}_m)^2 n\gamma^i_o \cov(B^i_t, B^i_s) + n\norm{g}_1 \sum_{i\in I} w^{i\to j}_m \gamma^i_o \cov(B^j_t, B^j_s) \\
    & = n t\wedge s \left( (w^j_m, w^j_m, \gamma^I_o) + \norm{g}_1 \Bar{\lambda}^j_m\right).
\end{align*}
Furthermore, if $j'\neq j$ then
\begin{align*}
    \cov(W^{j,n}_{m,t}, W^{j',n}_{m,s}) &= \sum_{i\in I} w^{i\to j}_m w^{i\to j'}_m n\gamma^i_o \cov(B^i_t, B^i_s)   \\
    & = n (t\wedge s)(w^j_m, w^{j'}_m, \gamma^I_o).
\end{align*}

Let us choose $x$ such that $(\abs{I}+ \abs{J})b'e^{-d'x} = \alpha$, \ie $x = \frac{1}{d'}\log\left( \frac{(\abs{I}+\abs{J})b'}{\alpha}\right)$. Then for $b'\geq \Box$, we have $x\geq 1$ and with probability more than $1-\alpha$, for every $j\in J$ we have
\begin{equation} \label{eq sup N - W}
    \sup_{t> 0}\frac{ \AAbs{N^{j,n}_{m,t} - W^{j,n}_{m,t}}}{n^{1/4}(t\vee 1)^{1/4}\log(nt + 1)^{3/4}} \leq \Box_{\norm{g}_1,\norm{\gamma}_\infty} \log\left(\frac{\abs{I}+\abs{J}}{\alpha}\right)^{3/4}.
\end{equation}

Now let us study the hitting times. Let us work on the same event that we denote $\Omega_1$. Let $U\coloneqq \max_{j\in J} \frac{\thr}{\Bar{\lambda}^j_m \norm{g}_1} + 1 $. Then for every $j\in J$ and $t\geq 0$ we have
\[
\AAbs{N^{j,n}_{m,t} - \E[N^{j,n}_{m,t}]} \leq \AAbs{\Tilde{B}^{j,n}_{m,t}} + \Box_{\norm{g}_1,\norm{\gamma}_\infty} n^{1/4}(t\vee 1)^{1/4} \log(nt +1)^{3/4} 
\log\left(\frac{\abs{I}+\abs{J}}{\alpha}\right)^{3/4}
\]
where $\Tilde{B}^{j,n}_{m,t}$ =  $\sum_{i\in I}w^{i\to j}_m B^i_{n \gamma^i_o t} + B^j_{n\Bar{\lambda}^j_m \norm{g}_1 t} $. Then $(\Tilde{B}^{j,n}_{m,t})_{t\geq 0}$ is a Brownian motion with $0$ drift and scaling $\sigma_n \coloneqq \sqrt{ n \left( (w^j_m, w^j_m, \gamma^I_o) + \norm{g}_1 \Bar{\lambda}^j_m\right)} \leq \Box_{\norm{g}_1,\norm{\gamma}_\infty}\sqrt{n}$.
According to Lemma~\ref{lemma control z - mu}, by union bound there exists an event $\Omega_2$ with probability more than $1-\alpha$ such that on $\Omega_2$, 
\[
\forall j\in J, \quad \sup_{0\leq t \leq U} \AAbs{\Tilde{B}^{j,n}_{m,t}} < \Box_{\norm{\gamma}_\infty,\norm{g}_1}  \sqrt{nU\log(\abs{J}\alpha^{-1})}.
\]
Let us work on $\Omega_1 \cap \Omega_2$. We can now control $\AAbs{N^{j,n}_{m,t} - \E[N^{j,n}_{m,t}]}$. Furthermore, $\E[N^{j,n}_{m,t}] = \E[\Lambda^{j,n}_{m,t}]$ and by \eqref{eq encadrement compens} we have $\AAbs{ \E[\Lambda^{j,n}_{m,t}] - n\Bar{\lambda}^j_m \norm{g}_1 t} \leq \Box_{\norm{g}_1,\norm{\gamma}_\infty}\sqrt{n}$. By combining everything, we get that for every $j\in J$ and $t\in [0,U]$

\begin{align}
    &\AAbs{N^{j,n}_{m,t} - n\Bar{\lambda}^j_m \norm{g}_1 t} \nonumber \\
    &\leq \Box_{\norm{\gamma}_\infty,\norm{g}_1}\left(\sqrt{nU\log(\abs{J}\alpha^{-1})} 
+ n^{1/4} U^{1/4} \log(nU + 1)^{3/4}\log\left(\frac{\abs{I}+\abs{J}}{\alpha}\right)^{3/4}\right) \nonumber \\
& \leq \Box_{\norm{\gamma}_\infty,\norm{g}_1} \sqrt{nU} \log\left(\frac{\abs{I}+\abs{J}}{\alpha}\right)^{3/4}.
\label{eq sup N - E}
\end{align}

Since $W^{j,n}_{m,t} = \E[N^{j,n}_{m,t}] + \Tilde{B}^{j,n}_{m,t}$ and $\E[N^{j,n}_{m,t}] = \E[W^{j,n}_{m,t}]$, similarly we have
\begin{equation} \label{eq sup W - E}
    \forall j\in J, \forall t \in [0,U], \AAbs{W^{j,n}_{m,t} - n\Bar{\lambda}^j_m \norm{g}_1 t} \leq \Box_{\norm{\gamma}_\infty,\norm{g}_1}\sqrt{nU\log(\abs{J}\alpha^{-1})}.
\end{equation}

$\bullet$ Study of the hitting times of $W^{j,n}_{m,t}$ and $N^{j,n}_{m,t}$.
\newline

Let $\tauwbis \coloneqq \inf\{u\in[0,U], \W^{j,n}_{m,u} \geq n\theta\}$ and $\taunbis = \inf\{u\in[0,U], N^{j,n}_{m,u} \geq n\theta\}$. First, let us study under which conditions $\taunbis<\infty$ and  $\tauwbis<\infty$. According to \eqref{eq sup N - E}, we have that for all $j\in J$,
\begin{align*}
     N^{j,n}_{m,U} &\geq n\Bar{\lambda}^j_m\norm{g}_1 U - \Box_{\norm{\gamma}_\infty,\norm{g}_1} \sqrt{nU} \log\left(\frac{\abs{I}+\abs{J}}{\alpha}\right)^{3/4}.
\end{align*}
Then a sufficient condition to have $\taunbis<\infty$ is to have
\[
 n \Bar{\lambda}^j_m \norm{g}_1 U - \Box_{\norm{\gamma}_\infty,\norm{g}_1} \sqrt{nU} \log\left(\frac{\abs{I}+\abs{J}}{\alpha}\right)^{3/4} > n\thr,
\]
which is implied by
\begin{equation} \label{eq U}
    U > \frac{\thr}{\Bar{\lambda}^j_m \norm{g}_1} +  \Box_{\norm{\gamma}_\infty,\norm{g}_1, \Bar{\lambda}_{\text{min}},\thr} n^{-1/2}  \log\left(\frac{\abs{I}+\abs{J}}{\alpha}\right)^{3/4}
\end{equation}
which holds for 
\begin{equation} \label{eq large n}
    n\geq \Box_{\norm{\gamma}_\infty,\norm{g}_1, \Bar{\lambda}_{\text{min}},\thr} \log\left(\frac{\abs{I}+\abs{J}}{\alpha}\right)^{3/2}
\end{equation}
by definition of $U$. By using \eqref{eq sup W - E}, we can apply the same reasoning to $W^{j,n}_{m}$ to find that \eqref{eq U} is also a sufficient condition to have $\tauwbis<\infty$.

This means that on this event, both processes $W^{j,n}_m$ and $N^{j,n}_m$ have reached $n\thr$ before time $U$, so we have in fact $\taun = \taunbis$ and $\tauw = \tauwbis$. Therefore, for every $j\in J$ and $Z\in \{N,W\}$, we have $\tau^{Z^j}_n \leq U$, and controlling the trajectory of $Z^{j,n}_m$ on $[0,U]$ is sufficient to control it on $[0,\tau^{Z^j}_n]$. Using this fact, let us now study the convergence of $\taun$ and $\tauw$ to $\frac{\thr}{\Bar{\lambda}^j_m \norm{g}_1}$.
\newline

1. Convergence of $\taun$. Since $N^{j,n}_{m,\taun}$ is an integer and cannot make two jumps at the same time, we have $N^{j,n}_{m,\taun} = \lceil n\thr \rceil$. Therefore, for every $j\in J$ we have
\begin{align*}
    \AAbs{n\thr- n \Bar{\lambda}^j_m \norm{g}_1 \taun } &\leq \AAbs{n\thr -  \lceil n\thr \rceil} + \AAbs{N^{j,n}_{m,\taun} - n \Bar{\lambda}^j_m \norm{g}_1 \taun} \\
    & \leq 1 + \Box_{\norm{\gamma}_\infty,\norm{g}_1} \sqrt{nU} \log\left(\frac{\abs{I}+\abs{J}}{\alpha}\right)^{3/4}
\end{align*}
were the last inequality holds using \eqref{eq sup N - E}. So we have
\[
\AAbs{\taun - \frac{\thr}{\Bar{\lambda}^j_m \norm{g}_1}} \leq \Box_{\norm{\gamma}_\infty,\norm{g}_1, \Bar{\lambda}_{\text{min}},\thr}  n^{-1/2}\log\left(\frac{\abs{I}+\abs{J}}{\alpha}\right)^{3/4}.
\]

2. Convergence of $\tauw$. According to \eqref{eq sup W - E} we have
\begin{align*}
    \AAbs{n\thr- n \Bar{\lambda}^j_m \norm{g}_1 \tauw } &=  \AAbs{W^{j,n}_{m,\taun} - n \Bar{\lambda}^j_m \norm{g}_1 \tauw} \\
    & \leq  \Box_{\norm{\gamma}_\infty,\norm{g}_1}\sqrt{nU\log(\abs{J}\alpha^{-1})}
\end{align*}
so
\[
\AAbs{\tauw - \frac{\thr}{\Bar{\lambda}^j_m \norm{g}_1}} \leq\Box_{\norm{\gamma}_\infty,\norm{g}_1, \Bar{\lambda}_{\text{min}},\thr} \sqrt{\frac{1}{n}\log(\abs{J}\alpha^{-1})}
\]

3. Study of the difference $\AAbs{\taun - \tauw}$. Let us begin by bounding $\AAbs{W^{j,n}_{m,\taun} - W^{j,n}_{m,\tauw}}$. We have
\begin{align*}
    \AAbs{W^{j,n}_{m,\taun} - W^{j,n}_{m,\tauw}} &= \AAbs{W^{j,n}_{m,\taun} - n\thr}\\
    &\leq \AAbs{W^{j,n}_{m,\taun} - \lceil n\thr \rceil } + \abs{ \lceil n\thr \rceil - n\thr}\\
    &\leq  \AAbs{W^{j,n}_{m,\taun} - \Pi^{j,n}_{m,\taun}} + 1 \\
\end{align*}
so according to \eqref{eq sup N - W} we have
\begin{equation} \label{eq wtau}
     \AAbs{W^{j,n}_{m,\taun} - W^{j,n}_{m,\tauw}} \leq \Box_{\norm{\gamma}_\infty,\norm{g}_1, \Bar{\lambda}_{\text{min}},\thr} n^{1/4}\log(n)^{3/4} \log\left(\frac{\abs{I}+\abs{J}}{\alpha}\right)^{3/4}
\end{equation}

Besides,
\begin{align*}
   \AAbs{W^{j,n}_{m,\taun} - W^{j,n}_{m,\tauw}} &= \AAbs{n\Bar{\lambda}^j_m\left(\mathbb{G}(\taun) - \mathbb{G}(\tauw)\right) + \left(\Tilde{B}^{j,n}_{m,\taun}-\Tilde{B}^{j,n}_{m,\tauw}\right) }
\end{align*}
Besides, if $t_1 \geq t_2 \geq n^{-1/2}$ then
\begin{align*}
    \mathbb{G}(t_1) - \mathbb{G}(t_2) &= \int_{s = t_2}^{t_1}  \int_{u=0}^s g(u) du ds \\
    & =  \int_{s = t_2}^{t_1} \norm{g}_1 ds \\
    & = \norm{g}_1(t_1 - t_2).
\end{align*}
Since $\taun$ and $\tauw$ both converge to $\frac{\thr}{\Bar{\lambda}^j_m \norm{g}_1} >0$, so for $n$ verifying \eqref{eq large n} we have $\taun \geq n^{-1/2}$ and $\tauw \geq n^{-1/2}$ so $\abs{\mathbb{G}(\taun) - \mathbb{G}(\tauw)} = \abs{\taun - \tauw}\norm{g}_1$.  Therefore
\begin{align*}
   &\AAbs{W^{j,n}_{m,\taun} - W^{j,n}_{m,\tauw}} \geq n\Bar{\lambda}^j_m \norm{g}_1 \abs{\taun - \tauw} - \AAbs{\Tilde{B}^{j,n}_{m,\taun}- \Tilde{B}^{j,n}_{m,\tauw}}
\end{align*}
According to Lemma~\ref{Control_brownian_2times} and since $\tauw$ and $\taun$ are bounded on $\Omega_1\cap\Omega_2$, by union bound there exists an event $\Omega_3$ of probability more than $1-\alpha$ such that on $\Omega_1\cap \Omega_2\cap\Omega_3$ we have for every $j\in J$
\[
\AAbs{\Tilde{B}^{j,n}_{m,\taun} - \Tilde{B}^{j,n}_{m,\tauw}}\leq \Box_{\norm{\gamma}_\infty,\norm{g}_1, \Bar{\lambda}_{\text{min}},\thr} \sqrt{\left(\left(n\AAbs{\taun - \tauw}\right)\vee 1\right)  \log(n\abs{J}\alpha^{-1})}.
\]
If $\left(n\AAbs{\taun - \tauw}\right)\vee 1 = 1$ then $\AAbs{\taun - \tauw}\leq 1/n$ and we have the result. Otherwise we have
\begin{align*}
   &\AAbs{W^{j,n}_{m,\taun} - W^{j,n}_{m,\tauw}} \\
   &\geq n \Bar{\lambda}^j_m\abs{\taun - \tauw} - \Box_{\norm{\gamma}_\infty,\norm{g}_1, \Bar{\lambda}_{\text{min}},\thr}\sqrt{n\AAbs{\taun - \tauw}\log(n\abs{J}\alpha^{-1})} .
\end{align*}
Combining this with \eqref{eq wtau} we have
\begin{align*}
   & \Box_{\norm{\gamma}_\infty,\norm{g}_1, \Bar{\lambda}_{\text{min}},\thr} n^{1/4}\log(n)^{3/4} \log\left(\frac{\abs{I}+\abs{J}}{\alpha}\right)^{3/4} \\
   &\geq n\abs{\taun - \tauw} - \Box_{\norm{\gamma}_\infty,\norm{g}_1, \Bar{\lambda}_{\text{min}},\thr}\sqrt{n\AAbs{\taun - \tauw}\log(n\abs{J}\alpha^{-1})} .
\end{align*}
By writing $X \coloneqq \sqrt{ n\abs{\taun - \tauw}}$, $a(n)\coloneqq  \Box_{\norm{\gamma}_\infty,\norm{g}_1, \Bar{\lambda}_{\text{min}},\thr}\sqrt{\log(n\abs{J}\alpha^{-1})}$ and $b(n) \coloneqq \Box_{\norm{\gamma}_\infty,\norm{g}_1, \Bar{\lambda}_{\text{min}},\thr} n^{1/4}\log(n)^{3/4} \log\left(\frac{\abs{I}+\abs{J}}{\alpha}\right)^{3/4} $, we can rewrite it as
\begin{align*} 
    X^2 - a(n)X - b(n) \leq 0.
\end{align*}
This polynomial has two real roots, the largest one being $x_0 = \frac{a(n)+\sqrt{a(n)^2 + 4b(n)}}{2}$, so necessarily we have $X\leq x_0$, which implies that $X\leq a(n) + \sqrt{b(n)}$, \ie
\begin{multline*}
     \sqrt{ n\abs{\taun - \tauw}} \\\leq  \Box_{\norm{\gamma}_\infty,\norm{g}_1, \Bar{\lambda}_{\text{min}},\thr} \left(
     \sqrt{\log(n\abs{J}\alpha^{-1})} +  n^{1/8}\log(n)^{3/8} \log\left(\frac{\abs{I}+\abs{J}}{\alpha}\right)^{3/8}
     \right) \\
     \leq  \Box_{\norm{\gamma}_\infty,\norm{g}_1, \Bar{\lambda}_{\text{min}},\thr}  n^{1/8}\log(n)^{3/8} \log\left(\frac{\abs{I}+\abs{J}}{\alpha}\right)^{3/8}
\end{multline*}
and finally we get
\begin{align*}
    \abs{\taun - \tauw} &\leq \Box_{\norm{\gamma}_\infty,\norm{g}_1, \Bar{\lambda}_{\text{min}},\thr}   \left(\frac{\log(n)}{n}\right)^{3/4} \log\left(\frac{\abs{I}+\abs{J}}{\alpha}\right)^{3/4} .
\end{align*}
Since $\mathbb{P}(\Omega_1\cap \Omega_2 \cap \Omega_3) \geq 1-3\alpha$, we get the results of the theorem by choosing $\alpha/3$ instead of $\alpha$.

\subsection{Proof of Proposition~\ref{prop cas particulier}}

Let $j_1$ and $j_2$ the two categories, $i_1$ the feature characterizing category $j_1$ and $i_2$ the feature characterizing category $j_2$.  Let us compute the feature discrepancies of the features. Since every $o\in j_1$ has feature $i_1$, for every $o\in j_1$ we have $\gamma^{i_1}_o = \gamma$ and since there is no $o\in j_2$ having feature $i_1$, for every $o\in j_2$ we have $\gamma^{i_1}_o = 0$. Therefore
\[
d^{i_1 \to j_1} = \gamma.
\]
Similarly, for every $i\in j_2$ we have $\gamma^{i_2}_o = \gamma$ and for every $o\in j_1$ we have $\gamma^{i_2}_o = 0$ so
\[
d^{i_2 \to j_1} = - \gamma < d^{i_1 \to j_1}
\]
Let $i\in I\setminus\{i_1,i_2\}$. Then there are $4$ rockets among the $16$ having simultaneously features $i_1$ and $i$. Since $j_1$ contains $5$ objects, there is at least one object without feature $i$ in $j_1$. Therefore $\brac{\gamma^i_o}_{o\in j_1} < \gamma$ and
\[
d^{i\to j_1} < \gamma = d^{i_1 \to j_1}.
\]
Thus $I^{j_1} = \arg\max_{i\in I} d^{i\to j_1} = \{i_1\}$ and by definition of the family $w_\infty$, for $i\in I$ we have $w^{i\to j_1} = \mathbb{1}_{i=i_1}$. By symmetry, the same computations are true when exchanging index $1$ with index $2$ so for $i\in I$ we also have $I^{j_1} = \{i_2\}$ and $w^{i\to j_2} = \mathbb{1}_{i=i_2}$. Now, let us compute the $\Bar{\lambda}^j_{\infty, o}$. For $o\in j_1$,
\[
\Bar{\lambda}^{j_1}_{\infty, o} = \sum_{i\in I} w^{i\to j_1}_\infty \gamma^i_o = \gamma^{i_1}_o = \gamma
\]
and
\[
\Bar{\lambda}^{j_2}_{\infty, o} = \sum_{i\in I} w^{i\to j_2}_\infty \gamma^i_o = \gamma^{i_2}_o = 0
\]
so 
\[
\Bar{\lambda}^{j_1}_{\infty, o} > \Bar{\lambda}^{j_2}_{\infty, o} + \Dlamb
\]
where $\Dlamb$ is any constant in $(0,\gamma)$. By symmetry, for $o\in j_2$ we also have
\[
\Bar{\lambda}^{j_2}_{\infty, o} > \Bar{\lambda}^{j_1}_{\infty, o} + \Dlamb.
\]

\end{appendices}

\bibliography{biblio}%


\begin{thebibliography}{60}
\ifx \bisbn   \undefined \def \bisbn  #1{ISBN #1}\fi
\ifx \binits  \undefined \def \binits#1{#1}\fi
\ifx \bauthor  \undefined \def \bauthor#1{#1}\fi
\ifx \batitle  \undefined \def \batitle#1{#1}\fi
\ifx \bjtitle  \undefined \def \bjtitle#1{#1}\fi
\ifx \bvolume  \undefined \def \bvolume#1{\textbf{#1}}\fi
\ifx \byear  \undefined \def \byear#1{#1}\fi
\ifx \bissue  \undefined \def \bissue#1{#1}\fi
\ifx \bfpage  \undefined \def \bfpage#1{#1}\fi
\ifx \blpage  \undefined \def \blpage #1{#1}\fi
\ifx \burl  \undefined \def \burl#1{\textsf{#1}}\fi
\ifx \doiurl  \undefined \def \doiurl#1{\url{https://doi.org/#1}}\fi
\ifx \betal  \undefined \def \betal{\textit{et al.}}\fi
\ifx \binstitute  \undefined \def \binstitute#1{#1}\fi
\ifx \binstitutionaled  \undefined \def \binstitutionaled#1{#1}\fi
\ifx \bctitle  \undefined \def \bctitle#1{#1}\fi
\ifx \beditor  \undefined \def \beditor#1{#1}\fi
\ifx \bpublisher  \undefined \def \bpublisher#1{#1}\fi
\ifx \bbtitle  \undefined \def \bbtitle#1{#1}\fi
\ifx \bedition  \undefined \def \bedition#1{#1}\fi
\ifx \bseriesno  \undefined \def \bseriesno#1{#1}\fi
\ifx \blocation  \undefined \def \blocation#1{#1}\fi
\ifx \bsertitle  \undefined \def \bsertitle#1{#1}\fi
\ifx \bsnm \undefined \def \bsnm#1{#1}\fi
\ifx \bsuffix \undefined \def \bsuffix#1{#1}\fi
\ifx \bparticle \undefined \def \bparticle#1{#1}\fi
\ifx \barticle \undefined \def \barticle#1{#1}\fi
\bibcommenthead
\ifx \bconfdate \undefined \def \bconfdate #1{#1}\fi
\ifx \botherref \undefined \def \botherref #1{#1}\fi
\ifx \url \undefined \def \url#1{\textsf{#1}}\fi
\ifx \bchapter \undefined \def \bchapter#1{#1}\fi
\ifx \bbook \undefined \def \bbook#1{#1}\fi
\ifx \bcomment \undefined \def \bcomment#1{#1}\fi
\ifx \oauthor \undefined \def \oauthor#1{#1}\fi
\ifx \citeauthoryear \undefined \def \citeauthoryear#1{#1}\fi
\ifx \endbibitem  \undefined \def \endbibitem {}\fi
\ifx \bconflocation  \undefined \def \bconflocation#1{#1}\fi
\ifx \arxivurl  \undefined \def \arxivurl#1{\textsf{#1}}\fi
\csname PreBibitemsHook\endcsname

\bibitem[\protect\citeauthoryear{Brunton et~al.}{2013}]{brunton2013rats}
\begin{barticle}
\bauthor{\bsnm{Brunton}, \binits{B.W.}},
\bauthor{\bsnm{Botvinick}, \binits{M.M.}},
\bauthor{\bsnm{Brody}, \binits{C.D.}}:
\batitle{Rats and humans can optimally accumulate evidence for decision-making}.
\bjtitle{Science}
\bvolume{340}(\bissue{6128}),
\bfpage{95}--\blpage{98}
(\byear{2013})
\end{barticle}
\endbibitem

\bibitem[\protect\citeauthoryear{Besan{\c c}on et~al.}{2024}]{besanccon2024diffusive}
\begin{barticle}
\bauthor{\bsnm{Besan{\c c}on}, \binits{E.}},
\bauthor{\bsnm{Coutin}, \binits{L.}},
\bauthor{\bsnm{Decreusefond}, \binits{L.}},
\bauthor{\bsnm{Moyal}, \binits{P.}}:
\batitle{Diffusive limits of {L}ipschitz functionals of {P}oisson measures}.
\bjtitle{Ann. Appl. Probab.}
\bvolume{34}(\bissue{1A}),
\bfpage{555}--\blpage{584}
(\byear{2024})
\doiurl{10.1214/23-aap1972}
\end{barticle}
\endbibitem

\bibitem[\protect\citeauthoryear{Bacry et~al.}{2013}]{bacry2012scaling}
\begin{barticle}
\bauthor{\bsnm{Bacry}, \binits{E.}},
\bauthor{\bsnm{Delattre}, \binits{S.}},
\bauthor{\bsnm{Hoffmann}, \binits{M.}},
\bauthor{\bsnm{Muzy}, \binits{J.F.}}:
\batitle{Some limit theorems for {H}awkes processes and application to financial statistics}.
\bjtitle{Stochastic Process. Appl.}
\bvolume{123}(\bissue{7}),
\bfpage{2475}--\blpage{2499}
(\byear{2013})
\doiurl{10.1016/j.spa.2013.04.007}
\end{barticle}
\endbibitem

\bibitem[\protect\citeauthoryear{Bao et~al.}{2017}]{pmlr-v68-bao17a}
\begin{bchapter}
\bauthor{\bsnm{Bao}, \binits{Y.}},
\bauthor{\bsnm{Kuang}, \binits{Z.}},
\bauthor{\bsnm{Peissig}, \binits{P.}},
\bauthor{\bsnm{Page}, \binits{D.}},
\bauthor{\bsnm{Willett}, \binits{R.}}:
\bctitle{Hawkes {P}rocess {M}odeling of {A}dverse {D}rug {R}eactions with {L}ongitudinal {O}bservational {D}ata}.
In: \bbtitle{Proceedings of the 2nd Machine Learning for Healthcare Conference},
vol. \bseriesno{68},
pp. \bfpage{177}--\blpage{190}
(\byear{2017})
\end{bchapter}
\endbibitem

\bibitem[\protect\citeauthoryear{Boucheron et~al.}{2013}]{massartconcentration}
\begin{bbook}
\bauthor{\bsnm{Boucheron}, \binits{S.}},
\bauthor{\bsnm{Lugosi}, \binits{G.}},
\bauthor{\bsnm{Massart}, \binits{P.}}:
\bbtitle{Concentration Inequalities},
p. \bfpage{481}.
\bpublisher{Oxford University Press},
\blocation{Oxford}
(\byear{2013}).
\doiurl{10.1093/acprof:oso/9780199535255.001.0001} .
\bcomment{A nonasymptotic theory of independence, With a foreword by Michel Ledoux}.
\burl{https://doi.org/10.1093/acprof:oso/9780199535255.001.0001}
\end{bbook}
\endbibitem

\bibitem[\protect\citeauthoryear{Bretagnolle and Massart}{1989}]{bretagnolle1989hungarian}
\begin{barticle}
\bauthor{\bsnm{Bretagnolle}, \binits{J.}},
\bauthor{\bsnm{Massart}, \binits{P.}}:
\batitle{Hungarian constructions from the nonasymptotic viewpoint}.
\bjtitle{Ann. Probab.}
\bvolume{17}(\bissue{1}),
\bfpage{239}--\blpage{256}
(\byear{1989})
\end{barticle}
\endbibitem

\bibitem[\protect\citeauthoryear{Bacry et~al.}{2015}]{bacry2015hawkes}
\begin{barticle}
\bauthor{\bsnm{Bacry}, \binits{E.}},
\bauthor{\bsnm{Mastromatteo}, \binits{I.}},
\bauthor{\bsnm{Muzy}, \binits{J.-F.}}:
\batitle{Hawkes {P}rocesses in {F}inance}.
\bjtitle{Market Microstructure and Liquidity}
\bvolume{1}(\bissue{01}),
\bfpage{1550005}
(\byear{2015})
\end{barticle}
\endbibitem

\bibitem[\protect\citeauthoryear{Brody et~al.}{2003}]{brody2003basic}
\begin{barticle}
\bauthor{\bsnm{Brody}, \binits{C.D.}},
\bauthor{\bsnm{Romo}, \binits{R.}},
\bauthor{\bsnm{Kepecs}, \binits{A.}}:
\batitle{Basic mechanisms for graded persistent activity: discrete attractors, continuous attractors, and dynamic representations}.
\bjtitle{Curr. Opin. Neurobiol.}
\bvolume{13}(\bissue{2}),
\bfpage{204}--\blpage{211}
(\byear{2003})
\end{barticle}
\endbibitem

\bibitem[\protect\citeauthoryear{Cesa-Bianchi and Lugosi}{2006}]{cesa2006prediction}
\begin{bbook}
\bauthor{\bsnm{Cesa-Bianchi}, \binits{N.}},
\bauthor{\bsnm{Lugosi}, \binits{G.}}:
\bbtitle{Prediction, {L}earning, and {G}ames}.
\bpublisher{Cambridge University Press},
\blocation{Cambridge}
(\byear{2006})
\end{bbook}
\endbibitem

\bibitem[\protect\citeauthoryear{Carrillo et~al.}{2011}]{carrillo2011decision}
\begin{barticle}
\bauthor{\bsnm{Carrillo}, \binits{J.A.}},
\bauthor{\bsnm{Cordier}, \binits{S.}},
\bauthor{\bsnm{Mancini}, \binits{S.}}:
\batitle{A decision-making {F}okker-{P}lanck model in computational neuroscience}.
\bjtitle{J. Math. Biol.}
\bvolume{63}(\bissue{5}),
\bfpage{801}--\blpage{830}
(\byear{2011})
\doiurl{10.1007/s00285-010-0391-3}
\end{barticle}
\endbibitem

\bibitem[\protect\citeauthoryear{Caporale and Dan}{2008}]{caporale2008spike}
\begin{barticle}
\bauthor{\bsnm{Caporale}, \binits{N.}},
\bauthor{\bsnm{Dan}, \binits{Y.}}:
\batitle{Spike timing--dependent plasticity: a {H}ebbian learning rule}.
\bjtitle{Annu. Rev. Neurosci.}
\bvolume{31},
\bfpage{25}--\blpage{46}
(\byear{2008})
\end{barticle}
\endbibitem

\bibitem[\protect\citeauthoryear{Chevallier et~al.}{2021}]{chevallier2021diffusion}
\begin{barticle}
\bauthor{\bsnm{Chevallier}, \binits{J.}},
\bauthor{\bsnm{Melnykova}, \binits{A.}},
\bauthor{\bsnm{Tubikanec}, \binits{I.}}:
\batitle{Diffusion approximation of multi-class {H}awkes processes: theoretical and numerical analysis}.
\bjtitle{Adv. in Appl. Probab.}
\bvolume{53}(\bissue{3}),
\bfpage{716}--\blpage{756}
(\byear{2021})
\doiurl{10.1017/apr.2020.73}
\end{barticle}
\endbibitem

\bibitem[\protect\citeauthoryear{Ditterich}{2006}]{ditterich2006stochastic}
\begin{barticle}
\bauthor{\bsnm{Ditterich}, \binits{J.}}:
\batitle{Stochastic models of decisions about motion direction: behavior and physiology}.
\bjtitle{Neural networks}
\bvolume{19}(\bissue{8}),
\bfpage{981}--\blpage{1012}
(\byear{2006})
\end{barticle}
\endbibitem

\bibitem[\protect\citeauthoryear{Deco and Mart\'i}{2007}]{deco2007deterministic}
\begin{barticle}
\bauthor{\bsnm{Deco}, \binits{G.}},
\bauthor{\bsnm{Mart\'i}, \binits{D.}}:
\batitle{Deterministic analysis of stochastic bifurcations in multi-stable neurodynamical systems}.
\bjtitle{Biol. Cybernet.}
\bvolume{96}(\bissue{5}),
\bfpage{487}--\blpage{496}
(\byear{2007})
\doiurl{10.1007/s00422-007-0144-6}
\end{barticle}
\endbibitem

\bibitem[\protect\citeauthoryear{Ethier and Kurtz}{1986}]{ethier2009markov}
\begin{bbook}
\bauthor{\bsnm{Ethier}, \binits{S.N.}},
\bauthor{\bsnm{Kurtz}, \binits{T.G.}}:
\bbtitle{Markov Processes}.
\bsertitle{Wiley Series in Probability and Mathematical Statistics: Probability and Mathematical Statistics}.
\bpublisher{John Wiley \& Sons, Inc.},
\blocation{New York}
(\byear{1986}).
\doiurl{10.1002/9780470316658} .
\bcomment{Characterization and convergence}.
\burl{https://doi.org/10.1002/9780470316658}
\end{bbook}
\endbibitem

\bibitem[\protect\citeauthoryear{Erny et~al.}{2023}]{erny2023strong}
\begin{barticle}
\bauthor{\bsnm{Erny}, \binits{X.}},
\bauthor{\bsnm{L{\"o}cherbach}, \binits{E.}},
\bauthor{\bsnm{Loukianova}, \binits{D.}}:
\batitle{Strong error bounds for the convergence to its mean field limit for systems of interacting neurons in a diffusive scaling}.
\bjtitle{Ann. Appl. Probab.}
\bvolume{33}(\bissue{5}),
\bfpage{3563}--\blpage{3586}
(\byear{2023})
\end{barticle}
\endbibitem

\bibitem[\protect\citeauthoryear{Gao and Pavel}{2018}]{gao2018properties}
\begin{botherref}
\oauthor{\bsnm{Gao}, \binits{B.}},
\oauthor{\bsnm{Pavel}, \binits{L.}}:
On the {P}roperties of the {S}oftmax {F}unction with {A}pplication in {G}ame {T}heory and {R}einforcement {L}earning
(2018).
\url{https://arxiv.org/abs/1704.00805}
\end{botherref}
\endbibitem

\bibitem[\protect\citeauthoryear{Gold and Shadlen}{2001}]{gold2001neural}
\begin{barticle}
\bauthor{\bsnm{Gold}, \binits{J.I.}},
\bauthor{\bsnm{Shadlen}, \binits{M.N.}}:
\batitle{Neural computations that underlie decisions about sensory stimuli}.
\bjtitle{Trends Cogn. Sci.}
\bvolume{5}(\bissue{1}),
\bfpage{10}--\blpage{16}
(\byear{2001})
\end{barticle}
\endbibitem

\bibitem[\protect\citeauthoryear{Hawkes}{1971}]{hawkes1971spectra}
\begin{barticle}
\bauthor{\bsnm{Hawkes}, \binits{A.G.}}:
\batitle{Spectra of some self-exciting and mutually exciting point processes}.
\bjtitle{Biometrika}
\bvolume{58}(\bissue{1}),
\bfpage{83}--\blpage{90}
(\byear{1971})
\end{barticle}
\endbibitem

\bibitem[\protect\citeauthoryear{Hutcherson et~al.}{2015}]{hutcherson2015neurocomputational}
\begin{barticle}
\bauthor{\bsnm{Hutcherson}, \binits{C.A.}},
\bauthor{\bsnm{Bushong}, \binits{B.}},
\bauthor{\bsnm{Rangel}, \binits{A.}}:
\batitle{A neurocomputational model of altruistic choice and its implications}.
\bjtitle{Neuron}
\bvolume{87}(\bissue{2}),
\bfpage{451}--\blpage{462}
(\byear{2015})
\end{barticle}
\endbibitem

\bibitem[\protect\citeauthoryear{Hebb}{2005}]{hebb2005organization}
\begin{bbook}
\bauthor{\bsnm{Hebb}, \binits{D.O.}}:
\bbtitle{The Organization of Behavior: {A} Neuropsychological Theory}.
\bpublisher{Psychology press},
\blocation{New York}
(\byear{2005})
\end{bbook}
\endbibitem

\bibitem[\protect\citeauthoryear{Hansen et~al.}{2015}]{hansen2015lasso}
\begin{barticle}
\bauthor{\bsnm{Hansen}, \binits{N.R.}},
\bauthor{\bsnm{Reynaud-Bouret}, \binits{P.}},
\bauthor{\bsnm{Rivoirard}, \binits{V.}}:
\batitle{Lasso and probabilistic inequalities for multivariate point processes}.
\bjtitle{Bernoulli}
\bvolume{21}(\bissue{1}),
\bfpage{83}--\blpage{143}
(\byear{2015})
\doiurl{10.3150/13-BEJ562}
\end{barticle}
\endbibitem

\bibitem[\protect\citeauthoryear{Hubel and Wiesel}{1962}]{hubel1962receptive}
\begin{barticle}
\bauthor{\bsnm{Hubel}, \binits{D.H.}},
\bauthor{\bsnm{Wiesel}, \binits{T.N.}}:
\batitle{Receptive fields, binocular interaction and functional architecture in the cat's visual cortex}.
\bjtitle{J. Physiol.}
\bvolume{160}(\bissue{1}),
\bfpage{106}--\blpage{154}
(\byear{1962})
\end{barticle}
\endbibitem

\bibitem[\protect\citeauthoryear{Insabato et~al.}{2010}]{insabato2010confidence}
\begin{barticle}
\bauthor{\bsnm{Insabato}, \binits{A.}},
\bauthor{\bsnm{Pannunzi}, \binits{M.}},
\bauthor{\bsnm{Rolls}, \binits{E.T.}},
\bauthor{\bsnm{Deco}, \binits{G.}}:
\batitle{Confidence-related decision making}.
\bjtitle{Journal of neurophysiology}
\bvolume{104}(\bissue{1}),
\bfpage{539}--\blpage{547}
(\byear{2010})
\end{barticle}
\endbibitem

\bibitem[\protect\citeauthoryear{Jaffard et~al.}{2024}]{jaffard2024provable}
\begin{bchapter}
\bauthor{\bsnm{Jaffard}, \binits{S.}},
\bauthor{\bsnm{Vaiter}, \binits{S.}},
\bauthor{\bsnm{Muzy}, \binits{A.}},
\bauthor{\bsnm{Reynaud-Bouret}, \binits{P.}}:
\bctitle{Provable local learning rule by expert aggregation for a {H}awkes network}.
In: \bbtitle{International Conference on Artificial Intelligence and Statistics},
pp. \bfpage{1837}--\blpage{1845}
(\byear{2024}).
\bcomment{PMLR}
\end{bchapter}
\endbibitem

\bibitem[\protect\citeauthoryear{Jaffard et~al.}{2024}]{jaffard2024chani}
\begin{botherref}
\oauthor{\bsnm{Jaffard}, \binits{S.}},
\oauthor{\bsnm{Vaiter}, \binits{S.}},
\oauthor{\bsnm{Reynaud-Bouret}, \binits{P.}}:
Chani: {C}orrelation-based {H}awkes {A}ggregation of {N}eurons with bio-inspiration.
arXiv preprint arXiv:2405.18828
(2024)
\end{botherref}
\endbibitem

\bibitem[\protect\citeauthoryear{Koml\'os et~al.}{1975}]{komlos1975approximation}
\begin{barticle}
\bauthor{\bsnm{Koml\'os}, \binits{J.}},
\bauthor{\bsnm{Major}, \binits{P.}},
\bauthor{\bsnm{Tusn\'ady}, \binits{G.}}:
\batitle{An approximation of partial sums of independent {${\rm RV}$}'s and the sample {${\rm DF}$}. {I}}.
\bjtitle{Z. Wahrscheinlichkeitstheorie und Verw. Gebiete}
\bvolume{32},
\bfpage{111}--\blpage{131}
(\byear{1975})
\doiurl{10.1007/BF00533093}
\end{barticle}
\endbibitem

\bibitem[\protect\citeauthoryear{Koml\'os et~al.}{1976}]{komlos1976approximation}
\begin{barticle}
\bauthor{\bsnm{Koml\'os}, \binits{J.}},
\bauthor{\bsnm{Major}, \binits{P.}},
\bauthor{\bsnm{Tusn\'ady}, \binits{G.}}:
\batitle{An approximation of partial sums of independent {RV}'s, and the sample {DF}. {II}}.
\bjtitle{Z. Wahrscheinlichkeitstheorie und Verw. Gebiete}
\bvolume{34}(\bissue{1}),
\bfpage{33}--\blpage{58}
(\byear{1976})
\doiurl{10.1007/BF00532688}
\end{barticle}
\endbibitem

\bibitem[\protect\citeauthoryear{Karatzas and Shreve}{1998}]{karatzasbrownian}
\begin{bbook}
\bauthor{\bsnm{Karatzas}, \binits{I.}},
\bauthor{\bsnm{Shreve}, \binits{S.E.}}:
\bbtitle{Brownian Motion and Stochastic Calculus},
\bedition{2}nd edn.
\bsertitle{Graduate Texts in Mathematics},
vol. \bseriesno{113}.
\bpublisher{Springer},
\blocation{New York}
(\byear{1998}).
\doiurl{10.1007/978-1-4612-0949-2} .
\burl{https://doi.org/10.1007/978-1-4612-0949-2}
\end{bbook}
\endbibitem

\bibitem[\protect\citeauthoryear{LaBerge}{1994}]{laberge1994quantitative}
\begin{barticle}
\bauthor{\bsnm{LaBerge}, \binits{D.}}:
\batitle{Quantitative models of attention and response processes in shape identification tasks}.
\bjtitle{J. Math. Psychol.}
\bvolume{38}(\bissue{2}),
\bfpage{198}--\blpage{243}
(\byear{1994})
\end{barticle}
\endbibitem

\bibitem[\protect\citeauthoryear{LeCun et~al.}{1989}]{lecun1989handwritten}
\begin{bchapter}
\bauthor{\bsnm{LeCun}, \binits{Y.}},
\bauthor{\bsnm{Boser}, \binits{B.}},
\bauthor{\bsnm{Denker}, \binits{J.}},
\bauthor{\bsnm{Henderson}, \binits{D.}},
\bauthor{\bsnm{Howard}, \binits{R.}},
\bauthor{\bsnm{Hubbard}, \binits{W.}},
\bauthor{\bsnm{Jackel}, \binits{L.}}:
\bctitle{Handwritten digit recognition with a back-propagation network}.
In: \bbtitle{NeurIPS}
(\byear{1989})
\end{bchapter}
\endbibitem

\bibitem[\protect\citeauthoryear{Love et~al.}{2004}]{Love2004}
\begin{barticle}
\bauthor{\bsnm{Love}, \binits{B.}},
\bauthor{\bsnm{Medin}, \binits{D.}},
\bauthor{\bsnm{Gureckis}, \binits{T.}}:
\batitle{Sustain: a network model of category learning}.
\bjtitle{Psychol. Rev.}
\bvolume{111},
\bfpage{309}--\blpage{332}
(\byear{2004})
\end{barticle}
\endbibitem

\bibitem[\protect\citeauthoryear{Lovric}{2025}]{lovric2024international}
\begin{bbook}
\bauthor{\bsnm{Lovric}, \binits{M.}}:
\bbtitle{International Encyclopedia of Statistical Science}.
\bpublisher{Springer},
\blocation{Berlin, Heidelberg}
(\byear{2025})
\end{bbook}
\endbibitem

\bibitem[\protect\citeauthoryear{Machens et~al.}{2005}]{machens2005flexible}
\begin{barticle}
\bauthor{\bsnm{Machens}, \binits{C.K.}},
\bauthor{\bsnm{Romo}, \binits{R.}},
\bauthor{\bsnm{Brody}, \binits{C.D.}}:
\batitle{Flexible control of mutual inhibition: a neural model of two-interval discrimination}.
\bjtitle{Science}
\bvolume{307}(\bissue{5712}),
\bfpage{1121}--\blpage{1124}
(\byear{2005})
\end{barticle}
\endbibitem

\bibitem[\protect\citeauthoryear{Nosofsky}{1986}]{Nosofsky1986}
\begin{barticle}
\bauthor{\bsnm{Nosofsky}, \binits{R.M.}}:
\batitle{Attention, similarity, and the identification–categorization relationship}.
\bjtitle{J. Exp. Psychol. Gen.}
\bvolume{115},
\bfpage{39}--\blpage{57}
(\byear{1986})
\end{barticle}
\endbibitem

\bibitem[\protect\citeauthoryear{Ost and Reynaud-Bouret}{2023}]{ost2023neural}
\begin{barticle}
\bauthor{\bsnm{Ost}, \binits{G.}},
\bauthor{\bsnm{Reynaud-Bouret}, \binits{P.}}:
\batitle{Neural coding as a statistical testing problem}.
\bjtitle{Math. Neurosci. Appl.}
\bvolume{3},
\bfpage{4}--\blpage{33}
(\byear{2023})
\end{barticle}
\endbibitem

\bibitem[\protect\citeauthoryear{Papamakarios and Murray}{2016}]{papamakarios2016fast}
\begin{bchapter}
\bauthor{\bsnm{Papamakarios}, \binits{G.}},
\bauthor{\bsnm{Murray}, \binits{I.}}:
\bctitle{Fast $\varepsilon$-free {I}nference of {S}imulation {M}odels with {B}ayesian {C}onditional {D}ensity {E}stimation}.
In: \beditor{\bsnm{Lee}, \binits{D.}},
\beditor{\bsnm{Sugiyama}, \binits{M.}},
\beditor{\bsnm{Luxburg}, \binits{U.}},
\beditor{\bsnm{Guyon}, \binits{I.}},
\beditor{\bsnm{Garnett}, \binits{R.}} (eds.)
\bbtitle{Advances in Neural Information Processing Systems},
vol. \bseriesno{29}.
\bpublisher{Curran Associates},
\blocation{Inc.}
(\byear{2016}).
\burl{https://proceedings.neurips.cc/paper_files/paper/2016/file/6aca97005c68f1206823815f66102863-Paper.pdf}
\end{bchapter}
\endbibitem

\bibitem[\protect\citeauthoryear{Philiastides and Ratcliff}{2013}]{philiastides2013influence}
\begin{barticle}
\bauthor{\bsnm{Philiastides}, \binits{M.G.}},
\bauthor{\bsnm{Ratcliff}, \binits{R.}}:
\batitle{Influence of branding on preference-based decision making}.
\bjtitle{Psychological science}
\bvolume{24}(\bissue{7}),
\bfpage{1208}--\blpage{1215}
(\byear{2013})
\end{barticle}
\endbibitem

\bibitem[\protect\citeauthoryear{Prodhomme}{2020}]{prodhomme2020arxiv}
\begin{botherref}
\oauthor{\bsnm{Prodhomme}, \binits{A.}}:
Strong {G}aussian approximation of metastable density-dependent {M}arkov chains on large time scales
(2020).
\url{https://arxiv.org/abs/2010.06861v3}
\end{botherref}
\endbibitem

\bibitem[\protect\citeauthoryear{Prodhomme}{2023}]{prodhomme2023strong}
\begin{barticle}
\bauthor{\bsnm{Prodhomme}, \binits{A.}}:
\batitle{Strong {G}aussian approximation of metastable density-dependent {M}arkov chains on large time scales}.
\bjtitle{Stochastic Process. Appl.}
\bvolume{160},
\bfpage{218}--\blpage{264}
(\byear{2023})
\end{barticle}
\endbibitem

\bibitem[\protect\citeauthoryear{Ratcliff}{1978}]{ratcliff1978theory}
\begin{barticle}
\bauthor{\bsnm{Ratcliff}, \binits{R.}}:
\batitle{A theory of memory retrieval.}
\bjtitle{Psychol. Rev.}
\bvolume{85}(\bissue{2}),
\bfpage{59}--\blpage{108}
(\byear{1978})
\end{barticle}
\endbibitem

\bibitem[\protect\citeauthoryear{Ratcliff}{1981}]{ratcliff1981theory}
\begin{barticle}
\bauthor{\bsnm{Ratcliff}, \binits{R.}}:
\batitle{A theory of order relations in perceptual matching.}
\bjtitle{Psychol. Rev.}
\bvolume{88}(\bissue{6}),
\bfpage{552}--\blpage{572}
(\byear{1981})
\end{barticle}
\endbibitem

\bibitem[\protect\citeauthoryear{Ratcliff}{1988}]{ratcliff1988continuous}
\begin{barticle}
\bauthor{\bsnm{Ratcliff}, \binits{R.}}:
\batitle{Continuous versus discrete information processing: Modeling accumulation of partial information.}
\bjtitle{Psychol. Rev.}
\bvolume{95}(\bissue{2}),
\bfpage{238}--\blpage{255}
(\byear{1988})
\end{barticle}
\endbibitem

\bibitem[\protect\citeauthoryear{Reynaud-Bouret}{2003}]{reynaud2003adaptive}
\begin{barticle}
\bauthor{\bsnm{Reynaud-Bouret}, \binits{P.}}:
\batitle{Adaptive estimation of the intensity of inhomogeneous {P}oisson processes via concentration inequalities}.
\bjtitle{Probab. Theory Related Fields}
\bvolume{126}(\bissue{1}),
\bfpage{103}--\blpage{153}
(\byear{2003})
\end{barticle}
\endbibitem

\bibitem[\protect\citeauthoryear{Ratcliff et~al.}{2007}]{ratcliff2007dual}
\begin{barticle}
\bauthor{\bsnm{Ratcliff}, \binits{R.}},
\bauthor{\bsnm{Hasegawa}, \binits{Y.T.}},
\bauthor{\bsnm{Hasegawa}, \binits{R.P.}},
\bauthor{\bsnm{Smith}, \binits{P.L.}},
\bauthor{\bsnm{Segraves}, \binits{M.A.}}:
\batitle{Dual diffusion model for single-cell recording data from the superior colliculus in a brightness-discrimination task}.
\bjtitle{Journal of neurophysiology}
\bvolume{97}(\bissue{2}),
\bfpage{1756}--\blpage{1774}
(\byear{2007})
\end{barticle}
\endbibitem

\bibitem[\protect\citeauthoryear{Ratcliff and McKoon}{2008}]{ratcliff2008diffusion}
\begin{barticle}
\bauthor{\bsnm{Ratcliff}, \binits{R.}},
\bauthor{\bsnm{McKoon}, \binits{G.}}:
\batitle{The diffusion decision model: theory and data for two-choice decision tasks}.
\bjtitle{Neural Comput.}
\bvolume{20}(\bissue{4}),
\bfpage{873}--\blpage{922}
(\byear{2008})
\end{barticle}
\endbibitem

\bibitem[\protect\citeauthoryear{Rosenblatt}{1957}]{rosenblatt1957perceptron}
\begin{bbook}
\bauthor{\bsnm{Rosenblatt}, \binits{F.}}:
\bbtitle{The {P}erceptron, a {P}erceiving and {R}ecognizing {A}utomaton {P}roject {P}ara}.
\bpublisher{Cornell Aeronautical Laboratory},
\blocation{Ithaca, New York}
(\byear{1957})
\end{bbook}
\endbibitem

\bibitem[\protect\citeauthoryear{Roxin}{2019}]{roxin2019drift}
\begin{barticle}
\bauthor{\bsnm{Roxin}, \binits{A.}}:
\batitle{Drift-diffusion models for multiple-alternative forced-choice decision making}.
\bjtitle{J. Math. Neurosci.}
\bvolume{9},
\bfpage{5}--\blpage{23}
(\byear{2019})
\doiurl{10.1186/s13408-019-0073-4}
\end{barticle}
\endbibitem

\bibitem[\protect\citeauthoryear{Ratcliff and Rouder}{1998}]{ratcliff1998modeling}
\begin{barticle}
\bauthor{\bsnm{Ratcliff}, \binits{R.}},
\bauthor{\bsnm{Rouder}, \binits{J.N.}}:
\batitle{Modeling response times for two-choice decisions}.
\bjtitle{Psychological science}
\bvolume{9}(\bissue{5}),
\bfpage{347}--\blpage{356}
(\byear{1998})
\end{barticle}
\endbibitem

\bibitem[\protect\citeauthoryear{Robbins and Siegmund}{1970}]{robbins1970boundary}
\begin{barticle}
\bauthor{\bsnm{Robbins}, \binits{H.}},
\bauthor{\bsnm{Siegmund}, \binits{D.}}:
\batitle{Boundary crossing probabilities for the {W}iener process and sample sums}.
\bjtitle{Ann. Math. Statist.}
\bvolume{41},
\bfpage{1410}--\blpage{1429}
(\byear{1970})
\doiurl{10.1214/aoms/1177696787}
\end{barticle}
\endbibitem

\bibitem[\protect\citeauthoryear{Ratcliff and Smith}{2004}]{ratcliff2004comparison}
\begin{barticle}
\bauthor{\bsnm{Ratcliff}, \binits{R.}},
\bauthor{\bsnm{Smith}, \binits{P.L.}}:
\batitle{A {C}omparison of {S}equential {S}ampling {M}odels for {T}wo-{C}hoice {R}eaction {T}ime.}
\bjtitle{Psychol. Rev.}
\bvolume{111}(\bissue{2}),
\bfpage{333}--\blpage{367}
(\byear{2004})
\end{barticle}
\endbibitem

\bibitem[\protect\citeauthoryear{Revuz and Yor}{1999}]{revuz2013continuous}
\begin{bbook}
\bauthor{\bsnm{Revuz}, \binits{D.}},
\bauthor{\bsnm{Yor}, \binits{M.}}:
\bbtitle{Continuous Martingales and {B}rownian Motion},
\bedition{3}rd edn.
\bsertitle{Grundlehren der mathematischen Wissenschaften [Fundamental Principles of Mathematical Sciences]},
vol. \bseriesno{293},
p. \bfpage{602}.
\bpublisher{Springer},
\blocation{Berlin}
(\byear{1999}).
\doiurl{10.1007/978-3-662-06400-9} .
\burl{https://doi.org/10.1007/978-3-662-06400-9}
\end{bbook}
\endbibitem

\bibitem[\protect\citeauthoryear{Singer et~al.}{1997}]{singer1997neuronal}
\begin{barticle}
\bauthor{\bsnm{Singer}, \binits{W.}},
\bauthor{\bsnm{Engel}, \binits{A.K.}},
\bauthor{\bsnm{Kreiter}, \binits{A.K.}},
\bauthor{\bsnm{Munk}, \binits{M.H.}},
\bauthor{\bsnm{Neuenschwander}, \binits{S.}},
\bauthor{\bsnm{Roelfsema}, \binits{P.R.}}:
\batitle{Neuronal assemblies: necessity, signature and detectability}.
\bjtitle{Trends Cogn. Sci.}
\bvolume{1}(\bissue{7}),
\bfpage{252}--\blpage{261}
(\byear{1997})
\end{barticle}
\endbibitem

\bibitem[\protect\citeauthoryear{Shadlen and Kiani}{2013}]{shadlen2013decision}
\begin{barticle}
\bauthor{\bsnm{Shadlen}, \binits{M.N.}},
\bauthor{\bsnm{Kiani}, \binits{R.}}:
\batitle{Decision making as a window on cognition}.
\bjtitle{Neuron}
\bvolume{80}(\bissue{3}),
\bfpage{791}--\blpage{806}
(\byear{2013})
\end{barticle}
\endbibitem

\bibitem[\protect\citeauthoryear{Shadlen and Newsome}{1996}]{shadlen1996motion}
\begin{barticle}
\bauthor{\bsnm{Shadlen}, \binits{M.N.}},
\bauthor{\bsnm{Newsome}, \binits{W.T.}}:
\batitle{Motion perception: seeing and deciding.}
\bjtitle{Proc. Natl. Acad. Sci. USA}
\bvolume{93}(\bissue{2}),
\bfpage{628}--\blpage{633}
(\byear{1996})
\end{barticle}
\endbibitem

\bibitem[\protect\citeauthoryear{Stone}{2018}]{stone2018principles}
\begin{botherref}
\oauthor{\bsnm{Stone}, \binits{J.V.}}:
Principles of neural information theory.
Computational Neuroscience and Metabolic Efficiency
(2018)
\end{botherref}
\endbibitem

\bibitem[\protect\citeauthoryear{Smith and Van~Zandt}{2000}]{smith2000time}
\begin{barticle}
\bauthor{\bsnm{Smith}, \binits{P.L.}},
\bauthor{\bsnm{Van~Zandt}, \binits{T.}}:
\batitle{Time-dependent {P}oisson counter models of response latency in simple judgment}.
\bjtitle{British Journal of Mathematical and Statistical Psychology}
\bvolume{53}(\bissue{2}),
\bfpage{293}--\blpage{315}
(\byear{2000})
\end{barticle}
\endbibitem

\bibitem[\protect\citeauthoryear{Tavanaei et~al.}{2019}]{Tavanaei_2019}
\begin{barticle}
\bauthor{\bsnm{Tavanaei}, \binits{A.}},
\bauthor{\bsnm{Ghodrati}, \binits{M.}},
\bauthor{\bsnm{Kheradpisheh}, \binits{S.R.}},
\bauthor{\bsnm{Masquelier}, \binits{T.}},
\bauthor{\bsnm{Maida}, \binits{A.}}:
\batitle{Deep learning in spiking neural networks}.
\bjtitle{Neural Netw.}
\bvolume{111},
\bfpage{47}--\blpage{63}
(\byear{2019})
\end{barticle}
\endbibitem

\bibitem[\protect\citeauthoryear{T{\"u}rkyilmaz et~al.}{2013}]{turkyilmaz2013comparing}
\begin{barticle}
\bauthor{\bsnm{T{\"u}rkyilmaz}, \binits{K.}},
\bauthor{\bsnm{Lieshout}, \binits{M.N.M.}},
\bauthor{\bsnm{Stein}, \binits{A.}}:
\batitle{Comparing the {H}awkes and {T}rigger {P}rocess {M}odels for {A}ftershock {S}equences {F}ollowing the 2005 {K}ashmir {E}arthquake}.
\bjtitle{Math. Geosci.}
\bvolume{45},
\bfpage{149}--\blpage{164}
(\byear{2013})
\end{barticle}
\endbibitem

\bibitem[\protect\citeauthoryear{Zhou et~al.}{2013}]{zhou2013learning}
\begin{bchapter}
\bauthor{\bsnm{Zhou}, \binits{K.}},
\bauthor{\bsnm{Zha}, \binits{H.}},
\bauthor{\bsnm{Song}, \binits{L.}}:
\bctitle{Learning social {I}nfectivity in {S}parse {L}ow-rank {N}etworks {U}sing {M}ulti-dimensional {H}awkes {P}rocesses}.
In: \beditor{\bsnm{Carvalho}, \binits{C.M.}},
\beditor{\bsnm{Ravikumar}, \binits{P.}} (eds.)
\bbtitle{Proceedings of the Sixteenth International Conference on Artificial Intelligence and Statistics}.
\bsertitle{Proceedings of Machine Learning Research},
vol. \bseriesno{31},
pp. \bfpage{641}--\blpage{649}.
\bpublisher{PMLR},
\blocation{Scottsdale, Arizona, USA}
(\byear{2013}).
\burl{https://proceedings.mlr.press/v31/zhou13a.html}
\end{bchapter}
\endbibitem

\end{thebibliography}

\end{document}